\documentclass{article}

\pdfoutput=1

\usepackage{microtype}
\usepackage[pdftex]{graphicx}
\usepackage{subfigure}
\usepackage{booktabs} 

\usepackage{hyperref}

\usepackage[accepted]{icml2022}

\usepackage{amsmath}
\usepackage{amssymb}
\usepackage{mathtools}
\usepackage{amsthm}

\usepackage[capitalize,noabbrev]{cleveref}

\theoremstyle{plain}
\newtheorem{theorem}{Theorem}[section]

\newtheorem{lemma}[theorem]{Lemma}

\theoremstyle{definition}
\newtheorem{definition}[theorem]{Definition}

\theoremstyle{remark}

\usepackage[textsize=tiny]{todonotes}

\usepackage{stmaryrd}

\newtheorem*{theorem*}{Theorem}
\newtheorem*{proposition*}{Proposition}
\newtheorem*{lemma*}{Lemma}

\theoremstyle{definition}
\newtheorem*{assumption*}{Assumption}

\newcommand{\taud}{\tau_{\delta}}
\newcommand{\simplex}{\triangle_{K}}
\newcommand{\cK}{\mathcal{K}}
\newcommand{\cZ}{\mathcal{Z}}
\newcommand{\cM}{\mathcal{M}}
\newcommand{\cF}{\mathcal{F}}
\newcommand{\cL}{\mathcal{L}}
\newcommand{\cE}{\mathcal{E}}

\newcommand{\cC}{\mathcal{C}}
\newcommand{\cO}{\mathcal{O}}
\newcommand{\cT}{\mathcal{T}}
\newcommand{\cH}{\mathcal{H}}
\newcommand{\cX}{\mathcal{X}}
\newcommand{\cY}{\mathcal{Y}}
\newcommand{\cB}{\mathcal{B}}
\newcommand{\Real}{\mathbb{R}}
\newcommand{\Natural}{\mathbb{N}}
\newcommand{\bbS}{\mathbb{S}}

\newcommand{\1}{\mathbf{1}}
\newcommand{\probability}{\mathbb{P}}
\DeclareMathOperator*{\expectedvalue}{\mathbb{E}}

\newcommand{\gaussdistr}{\mathcal{N}}

\newcommand{\transpose}{^\mathsf{\scriptscriptstyle T}}
\DeclareMathOperator*{\argmin}{arg\,min}
\DeclareMathOperator*{\argmax}{arg\,max}
\DeclareMathOperator*{\add}{add}
\DeclareMathOperator*{\mul}{mul}

\newcommand{\eqdef}{\buildrel \text{def}\over =}

\begin{document}

\twocolumn[
\icmltitle{Choosing Answers in $\varepsilon$-Best-Answer Identification for Linear Bandits}



\icmlsetsymbol{equal}{*}

\begin{icmlauthorlist}
\icmlauthor{Marc Jourdan}{scool}
\icmlauthor{R\'emy Degenne}{scool}
\end{icmlauthorlist}

\icmlaffiliation{scool}{Univ. Lille, CNRS, Inria, Centrale Lille, UMR 9198-CRIStAL, F-59000 Lille, France}

\icmlcorrespondingauthor{Marc Jourdan}{marc.jourdan@inria.fr}

\icmlkeywords{Machine Learning, ICML, Bandits, Pure Exploration, Best-Arm Identification, Linear Bandits}

\vskip 0.3in
]



\printAffiliationsAndNotice{}  

\begin{abstract}
	In pure-exploration problems, information is gathered sequentially to answer a question on the stochastic environment.
While best-arm identification for linear bandits has been extensively studied in recent years, few works have been dedicated to identifying one arm that is $\varepsilon$-close to the best one (and not exactly the best one).
In this problem with several correct answers, an identification algorithm should focus on one candidate among those answers and verify that it is correct.
We demonstrate that picking the answer with highest mean does not allow an algorithm to reach asymptotic optimality in terms of expected sample complexity.
Instead, a \textit{furthest answer} should be identified.
Using that insight to choose the candidate answer carefully, we develop a simple procedure to adapt best-arm identification algorithms to tackle $\varepsilon$-best-answer identification in transductive linear stochastic bandits. Finally, we propose an asymptotically optimal algorithm for this setting, which is shown to achieve competitive empirical performance against existing modified best-arm identification algorithms.
\end{abstract}

\section{Introduction} \label{sec:section_introduction}

The multi-armed bandit (MAB) setting is an extensively studied problem in sequential decision making \citep{robbins_1952_AspectsSequentialDesign, lattimore_2020_BanditAlgorithms}. The environment is represented by a set of arms, each associated with an unknown reward distribution. The agent interacts with it by sequentially ``pulling'' arms, that is choosing an arm and observing a sample from its distribution.
We focus on the pure-exploration framework in which the objective of the agent is to answer a query as fast and reliably as possible, while disregarding the accumulated cost. We consider noisy linear observations depending on an unknown parameter $\mu$.

The pure-exploration setting for stochastic bandits was first studied in \citet{even_dar_2002_PACBoundsMultiarmed} and \citet{bubeck_2009_PureExplorationMultiArmed}. It has been studied in two major theoretical frameworks \citep{audibert_2010_BestArmIdentification,gabillon_2012_BestArmIdentification,jamieson_2014_BestarmIdentificationAlgorithms,kaufmann_2016_ComplexityBestArm}: the \textit{fixed-budget} setting and the \textit{fixed-confidence} setting. In the fixed-budget setting, the objective is to minimize the probability of misidentifying a correct answer given a fixed number of samples. We consider the fixed-confidence setting, where the agent aims at minimizing the number of pulls to identify a correct answer with confidence $1-\delta$.\looseness=-1

The most well known pure exploration setting is best-arm identification (BAI) \citep{audibert_2010_BestArmIdentification,chen_2017_InstanceOptimalBounds}, in which the goal is to return the arm with largest expected reward.
Numerous variants of BAI have been considered in recent years: linear bandits \citep{soare_2014_BestArmIdentificationLinear,zaki_2020_ExplicitBestArm, degenne_2020_GamificationPureExploration}, thresholding bandits \citep{locatelli_2016_OptimalAlgorithmThresholding,cheshire_2021_ProblemDependentView}, minimum threshold \citep{kaufmann_2018_SequentialTestLowest}, combinatorial bandits \citep{chen_2014_CombinatorialPureExploration, katz_samuels_2020_EmpiricalProcessApproach, jourdan_2021_EfficientPureExploration}, top-m identification \citep{kalyanakrishnan_2012_PACSubsetSelection, katz_samuels_2019_TopFeasibleArm, reda_2021_TopmIdentificationLinear}, matching bandits \citep{sentenac_2021_PureExplorationRegret}, logistic bandits \citep{jun_2021_ImprovedConfidenceBounds}, identifying all $\varepsilon$-optimal answers \citep{mason_2020_FindingAllEpsilon,almarjani_2022_ComplexityAllEpsilon}, etc.

When the gap between the best and the second best arm is small, BAI problems are difficult, meaning that an algorithm needs a large number of samples to be correct. To avoid wasteful queries, practitioners might be interested in the easier task of identifying one answer that is $\varepsilon$-close to the best one, but not exactly the best one.

In this work, we consider $(\varepsilon,\delta)$-PAC best-answer identification ($\varepsilon$-BAI) for transductive linear bandits. In transductive bandits \citep{fiez_2019_SequentialExperimentalDesign}, the set of arms $\cK$ that can be pulled by the agent is different from the set of answers $\cZ$ on which the identification procedure focuses. In contrast to best-answer identification, our agent aims at identifying one $\varepsilon$-optimal answer (defined below) among the existing ones \citep{mannor_2004_SampleComplexityExploration,even_dar_2006_ActionEliminationStopping,sabato_2019_EpsilonBestArmIdentificationPayPerRewarda}.
Note that the ranking and selection literature explores this question with a different approach, see \citet{hong_2021_ReviewRankingSelection} for a review. $\varepsilon$-BAI has been studied for MAB \citep{garivier_2019_NonAsymptoticSequentialTests}, spectral bandits \citep{kocak_2021_EpsilonBestArm}, on the unit sphere \citep{jedra_2020_OptimalBestarmIdentification} and as a special case of the multiple correct answer setting \citep{degenne_2019_PureExplorationMultiple}.
As we will show, the asymptotic complexity of $\varepsilon$-BAI is governed by the identification of the \textit{furthest answer}. This is the $\varepsilon$-optimal answer for which it is easiest to verify $\varepsilon$-optimality. Therefore, all $\varepsilon$-optimal answers are not equivalent.
When aiming at asymptotic optimality of the expected sample complexity, algorithms should select the candidate $\varepsilon$-optimal answer more carefully than simply using the \textit{greedy answers} defined as $z^{\star}(\mu) \eqdef \argmax_{z \in \cZ} \langle \mu, z\rangle$ (the answers for which the reward is maximal). In the special case of BAI, the assumption $|z^{\star}(\mu)| =1$ is made and algorithms use the greedy answer, which is the unique correct answer to identify.

For our algorithm, we adopt a saddle-point (or game) approach between the agent and the \textit{nature}, aiming at iteratively approximating the lower bound on the expected sample complexity. The goal is to design an asymptotically optimal algorithm with competitive empirical performance in finite-time regime.

\paragraph{Contributions} \textbf{(1)} We provide an analysis of $\varepsilon$-BAI for transductive linear bandits and highlight a phenomenon which was overlooked by previous work. The choice of the candidate $\varepsilon$-optimal answer is crucial to reach asymptotic optimality in terms of expected sample complexity and one should identify the furthest answer instead of using the greedy answers. \textbf{(2)} By carefully choosing the candidate $\varepsilon$-optimal answer and leaving the sampling rule unchanged, we develop a simple procedure to adapt BAI algorithms to be $(\varepsilon, \delta)$-PAC and empirically competitive for $\varepsilon$-BAI in transductive linear stochastic bandits. \textbf{(3)} By leveraging the concept of $\varepsilon$-optimal answer in the sampling rule, we propose an asymptotically optimal algorithm in this setting, which has competitive empirical performance.

\subsection{Problem Statement} \label{sec:subsection_problem_statement}

\paragraph{Transductive Linear Bandits} We consider the \textit{transductive linear bandits} setting, where the collection of arms $\cK \subseteq \Real^d$ and answers $\cZ \subseteq \Real^d$ are finite sets, potentially different \citep{fiez_2019_SequentialExperimentalDesign}, with cardinality $|\cK| = K$ and $|\cZ| = Z$. Taking $\cZ = \cK$ yields the \textit{linear bandits} setting.
We assume that $\cK$ spans $\Real^d$ and denote $L_{\cK} \eqdef \max_{a \in \cK} \|a\|_2$ where $\|a\|_2$ is the euclidean norm of $a \in \cK$.
The interaction with the environment goes as follows: in each round $t \geq 1$, the agent chooses an arm $a_t \in \cK$ and observes $X_{t}^{a_t} = \langle \mu, a_t \rangle + \eta_t$ where $\mu$ is an unknown \textit{mean parameter} belonging to the set $\cM \subseteq \Real^d$, which is known to the agent\footnote{Note that we use $a\in \cK$ as a superscript to denote the index of the element in $\Real^{K}$ corresponding to the vector $a$.}.
The noise $\eta_t \sim \gaussdistr(0, \sigma^2)$ is conditionally independent of the past. Without loss of generality, we consider $\sigma^2 = 1$. Prior works have lifted the Gaussian assumption by considering sub-Gaussian distributions. The focus of our work is to highlight a phenomenon which is orthogonal to the distribution, therefore restricting ourselves to Gaussians is an assumption we are willing to make.\looseness=-1

\paragraph{$(\varepsilon,\delta)$-PAC Best-Answer Identification} In $\varepsilon$-BAI ($\varepsilon \geq 0$) for transductive linear bandits, the agent aims at identifying one of the \textit{$\varepsilon$-optimal} answers by sequentially pulling arms. We address two different notions of $\varepsilon$-optimality: the \textit{additive $\varepsilon$-optimal} answers, $\cZ_{\varepsilon}^{\add}(\mu) \eqdef \left\{ z \in \cZ:  \langle \mu,  z \rangle \geq \max_{z \in \cZ} \langle \mu, z\rangle - \varepsilon \right\}$, and the \textit{multiplicative $\varepsilon$-optimal} answers, $\cZ_{\varepsilon}^{\mul}(\mu) \eqdef \left\{ z \in \cZ: \langle \mu, z \rangle \geq (1-\varepsilon)\max_{z \in \cZ} \langle \mu, z\rangle \right\}$ when $\max_{z \in \cZ} \langle \mu, z\rangle > 0$.
The notation $\cdot^{\{\add,\mul\}}$ is dropped when the statement holds for both notions of $\varepsilon$-optimality. We will deal with both notions with the same method. Previous works mostly consider the additive $\varepsilon$-optimality \citep{garivier_2019_NonAsymptoticSequentialTests,kocak_2021_EpsilonBestArm}. When $\varepsilon=0$ both notions coincide with BAI, in which there is a unique correct answer.
$\varepsilon$-BAI is often seen as a more practical objective than BAI, in cases where getting an answer close to optimal is enough: while a BAI algorithm will spend many samples distinguishing between an $\varepsilon$-optimal answer and the best answer, an $\varepsilon$-BAI algorithm will be able to stop quickly.

\paragraph{Identification Strategy} The $\sigma$-algebra $\cF_t \eqdef \sigma \left( a_1, X_{1}^{a_1}, \cdots, a_t,X_{t}^{a_t}\right)$, called \emph{history}, encompasses all the information available to the agent after $t$ rounds. In the fixed-confidence setting an \textit{identification strategy} is described by three rules: a \textit{sampling rule} $(a_t)_{t \geq 1}$ where $a_t \in \cK$ is $\cF_{t-1}$-measurable, a \textit{stopping rule} $\taud$ which is a stopping time with respect to the filtration $(\cF_t)_{t\geq1}$, also referred as the \textit{sample complexity}, and a \textit{recommendation rule} $\hat{z}$ which is $\cF_{\taud}$-measurable. While the sampling rule could depend on additional internal randomization, the algorithms proposed in this work are deterministic.

In the fixed-confidence setting, the learner is given a confidence parameter $\delta \in (0,1)$. A strategy is said to be $(\varepsilon, \delta)$-PAC if, for all $\mu \in \cM$, with probability at most $\delta$ it terminates while not recommending an $\varepsilon$-optimal answer, i.e. $\probability_{\mu} \left[ \taud < + \infty, \hat{z} \notin \cZ_{\varepsilon}(\mu)\right] \leq \delta$. Among the class of $(\varepsilon, \delta)$-PAC algorithms, our goal is to minimize the expected sample complexity $\expectedvalue_{\mu}[\taud]$.

\section{Comparing \texorpdfstring{$\varepsilon$}{}-Optimal Answers} \label{sec:section_comparing_optimal_answers}

\subsection{Lower Bound} \label{sec:subsection_lower_bound}

For any $w \in (\Real^+)^{K}$, we define the design matrix $V_{w} \eqdef \sum_{a \in \cK} w^a a a\transpose \in \Real^{d \times d}$, which is symmetric and positive semi-definite, and definite if and only if $\text{Span}(\{a \in \cK : w^a \neq 0\}) = \Real^d$. For any symmetric positive semi-definite matrix $V \in \Real^{d \times d}$, we define the semi-norm $\|x\|_{V} \eqdef \sqrt{x\transpose V x}$ for $x \in \Real^d$, which is a norm if $V$ is positive definite. The probability simplex of dimension $K-1$ is denoted by $\simplex$ for all $K \geq 2$.

\paragraph{Alternative to $z$} Given an answer $z \in \cZ$, the \textit{alternative to $z$} is defined as the set of parameters for which $z$ is not an $\varepsilon$-optimal answer, $\neg_{\varepsilon} z \eqdef \overline{\left\{ \lambda \in \cM: z \notin \cZ_{\varepsilon}(\lambda)\right\}}$ where $\overline{X}$ denotes the closure of $X$.
Rewriting it for the additive and multiplicative $\varepsilon$-optimality, we obtain: $\neg_{\varepsilon}^{\add} z = \overline{\left\{ \lambda \in \cM: \langle \lambda,  z \rangle < \max_{z \in \cZ} \langle \lambda, z\rangle - \varepsilon \right\}}$ and $\neg_{\varepsilon}^{\mul} z = \overline{\left\{ \lambda \in \cM: \langle \lambda, z \rangle < (1-\varepsilon)\max_{z \in \cZ} \langle \lambda, z\rangle \right\}}$.
Identifying an $\varepsilon$-optimal answer $z \in \cZ_{\varepsilon}(\mu)$ is equivalent to rejecting the hypothesis that the unknown mean belongs to the alternative to $z$, i.e. $\cH_{0} = \left\{\mu \in \neg_{\varepsilon} z \right\}$. Informally, if we know with enough certainty that $\mu$ does not belong to the alternative $\neg_\varepsilon z$, we can safely return the answer $z$.

\paragraph{Asymptotic Lower Bound} Theorem~\ref{thm:sample_complexity_lower_bound_epsBAI} gives an asymptotic lower bound on the expected sample complexity of any $(\varepsilon, \delta)$-PAC strategy for both additive and multiplicative $\varepsilon$-optimality. This is a corollary of Theorem 1 in \citet{degenne_2019_PureExplorationMultiple}, which holds for any multiple answer instance and sub-Gaussian distributions (Appendix~\ref{proof:sample_complexity_lower_bound_epsBAI}).

\begin{theorem}[Theorem 1 in \citet{degenne_2019_PureExplorationMultiple}] \label{thm:sample_complexity_lower_bound_epsBAI}
	For all $(\varepsilon, \delta)$-PAC strategy, for all $\mu \in \cM$,
	\begin{equation*}
		\liminf_{\delta \rightarrow 0} \frac{\expectedvalue_{\mu}[\taud]}{\ln(1/\delta)} \geq T_{\varepsilon}(\mu)
	\end{equation*}
	where the inverse of the characteristic time is
	\begin{equation} \label{eq:definition_inverse_sample_complexity}
		T_{\varepsilon}(\mu)^{-1} \eqdef \max_{z \in \cZ_{\varepsilon}(\mu)}  \max_{w \in \simplex} \inf_{\lambda \in \neg_{\varepsilon} z} \frac{1}{2} \| \mu - \lambda \|_{V_{w}}^2 \; .
	\end{equation}
\end{theorem}

An $(\varepsilon,\delta)$-PAC algorithm is said to be \textit{asymptotically optimal} if the bound is tight: for all $\mu \in \cM$, $\liminf_{\delta \rightarrow 0} \frac{\expectedvalue_{\mu}[\taud]}{\ln(1/\delta)} \leq T_{\varepsilon}(\mu)$. The first lower bound for BAI was proved in \citet{garivier_2016_OptimalBestArm}.

As noted by \citet{chernoff_1959_SequentialDesignExperiments}, the complexity $T_{\varepsilon}(\mu)^{-1}$ is the value of a zero-sum game between two players. The agent chooses an $\varepsilon$-optimal answer and a pulling proportion over arms, $(z,w) \in \cZ_{\varepsilon}(\mu) \times \simplex$. The nature plays the most confusing alternative $\lambda \in \neg_{\varepsilon} z$ with respect to a reweighted Kullback-Leibler divergence ($\|\cdot\|^2_{V_{w}}$ for Gaussians) in order to fool the agent into rejecting this answer. Our algorithm, named \hyperlink{algoLeBAI}{L$\varepsilon$BAI} (\textbf{L}inear \textbf{$\varepsilon$-BAI}), is based on this formulation. Even for known $\mu$, computing $T_{\varepsilon}(\mu)^{-1}$ is in general intractable due to the non-convexity of $\neg_{\varepsilon} z$ and the additional maximization over $\cZ_{\varepsilon}(\mu)$. When $\varepsilon$ is large enough to have $\mathcal{Z}_{\varepsilon}(\lambda) = \mathcal{Z}$ for all $\lambda \in \mathcal{M}$, then $T_{\varepsilon}(\mu) = 0$, i.e. it is \textit{so easy} that no sample is needed.

While lower bounds for BAI have been derived in the non-asymptotic regime, it remains unclear whether equivalent lower bounds hold for $\varepsilon$-BAI \citep{garivier_2019_NonAsymptoticSequentialTests}.

\paragraph{Comparison with BAI} Since $T_{\varepsilon}(\mu) \leq T_{0}(\mu)$ for all $\mu \in \cM$ (because $\neg_{\varepsilon} z \subseteq \neg_{0} z$), $\varepsilon$-BAI is easier than BAI. There exists arbitrarily hard BAI instances that can be solved if seen as an $\varepsilon$-BAI problem, e.g. when the gap between the best and the second best arm is arbitrarily small.

\subsection{Furthest Answer} \label{sec:subsection_furthest_answer}

Our contributions are linked with the concept the furthest answer: it should be leveraged in the recommendation-stopping pair (Section~\ref{sec:section_from_BAI_to_eBAI}) and in the sampling rule (Section~\ref{sec:section_lebai}). In a nutshell, to reach asymptotic optimality in terms of sample complexity one should identify that furthest answer instead of simply using the greedy answers: \textit{all $\varepsilon$-optimal answers are not equivalent}.

The furthest answer $z_{F}(\mu)$ is the $\varepsilon$-optimal answer for which it is easiest to verify that $\mu$ does not belong to its alternative (when using an optimal allocation over arms $w_{F}(\mu) \in \simplex$). Introduced in \citet{degenne_2019_PureExplorationMultiple} and \citet{garivier_2019_NonAsymptoticSequentialTests}, it is defined as
\begin{equation} \label{eq:definition_furthest_answer_and_optimal_allocation}
	(z_{F}(\mu), w_{F}(\mu)) \eqdef \argmax_{(z, w) \in \cZ_{\varepsilon}(\mu) \times \simplex} \inf_{\lambda \in \neg_{\varepsilon} z} \frac{1}{2} \| \mu - \lambda \|_{V_{w}}^2 \; .
\end{equation}

$z_F(\mu)$ belongs to the $\varepsilon$-optimal answers $\cZ_{\varepsilon}(\mu)$, as does the greedy answers $z^{\star}(\mu) = \argmax_{z \in \cZ} \langle \mu, z\rangle$, but these answers may differ. In BAI with a unique best arm the set $\cZ_{\varepsilon}(\mu)$ is a singleton, hence those two notions coincide.

We assume there is a unique furthest answer for the unknown $\mu$, i.e. $|z_{F}(\mu)|=1$.
When $|z_{F}(\mu)| > 1$, some function of $\mu$ has to have exactly the same value for all answers of the set.
This happens with probability $0$ if $\mu$ arises from an absolutely continuous distribution.
Almost all BAI algorithms make the assumption that $|z^{\star}(\mu)|=1$, which implies $|z_{F}(\mu)|=1$ in the BAI case.
Since the furthest answer is assumed unique, we abuse notation and denote by $z_{F}(\mu)$ both that answer, and the singleton containing it as in \eqref{eq:definition_furthest_answer_and_optimal_allocation}.
$z^{\star}(\mu)$ denotes a set as we don't assume $|z^{\star}(\mu)|=1$.
The dependence of $z_{F}(\mu)$ and $w_{F}(\mu)$ on $\varepsilon$ is omitted.

\paragraph{Asymptotic Sub-optimality of $z^{\star}(\mu)$}
An $(\varepsilon,\delta)$-PAC strategy is said to be \textit{asymptotically greedy} if the only $\varepsilon$-optimal answers for which the algorithm will stop asymptotically are the greedy answers $z^{\star}(\mu)$, i.e. for all $\mu \in \cM$,
\begin{equation} \label{eq:def_asym_greedy_strategy}
	\lim_{\delta \rightarrow 0}\probability_{\mu} [\taud < + \infty, \hat{z} \in \cZ_{\varepsilon}(\mu) \setminus z^{\star}(\mu) ] = 0 \: .
\end{equation}

Lemma~\ref{lem:greedy_sample_complexity_lower_bound_epsBAI} shows that any asymptotically greedy $(\varepsilon, \delta)$-PAC strategy is asymptotically sub-optimal whenever $z_{F}(\mu) \notin z^{\star}(\mu)$, i.e. it can only reach $T_{g,\varepsilon}(\mu)$ which is strictly higher than $T_{\varepsilon}(\mu)$.

\begin{lemma} \label{lem:greedy_sample_complexity_lower_bound_epsBAI}
	For all asymptotically greedy $(\varepsilon,\delta)$-PAC strategy, for all $\mu \in \cM$,
	\begin{equation*}
		\liminf_{\delta \rightarrow 0} \frac{\expectedvalue_{\mu}[\taud]}{\ln(1/\delta)} \geq T_{g,\varepsilon}(\mu)
	\end{equation*}
	where the inverse of the greedy characteristic time is
	\begin{equation} \label{eq:definition_greedy_inverse_sample_complexity}
		T_{g,\varepsilon}(\mu)^{-1} \eqdef \max_{z \in z^{\star}(\mu)} \max_{w \in \simplex} \inf_{\lambda \in \neg_{\varepsilon} z} \frac{1}{2} \| \mu - \lambda \|_{V_{w}}^2 \; ,
	\end{equation}
	and $T_{g,\varepsilon}(\mu) > T_{\varepsilon}(\mu) $ if and only if $z_{F}(\mu) \notin z^{\star}(\mu)$.
\end{lemma}

Lemma~\ref{lem:algorithms_being_greedy_asymptotically} shows that any $(\varepsilon, \delta)$-PAC strategy recommending any greedy answers $\hat{z} \in z^{\star}(\mu_{\taud})$ which succeeds in identifying $z^{\star}(\mu)$ is asymptotically greedy. Since $\mu \mapsto z^{\star}(\mu)$ is continuous and $\cZ$ is finite, it is sufficient to have a sampling rule ensuring that $\lim_{t \rightarrow + \infty} \mu_t = \mu$.

\begin{lemma} \label{lem:algorithms_being_greedy_asymptotically}
Any $(\varepsilon, \delta)$-PAC strategy recommending $\hat{z} \in z^{\star}(\mu_{\taud})$ is asymptotically greedy if the sampling rule ensures that $\lim_{\delta \rightarrow 0}\probability_{\mu} [\taud < + \infty,  z^{\star}(\mu_{\taud}) =  z^{\star}(\mu) ] = 1$.
\end{lemma}

\paragraph{Asymptotic Optimality of $z_{F}(\mu)$} The furthest answer has by definition a central role in the characteristic time. Among the oracles that first choose an answer $z \in \cZ_\varepsilon(\mu)$ and then sample according to the optimal proportions to verify that $\mu \notin \neg_\varepsilon z$, the only one achieving asymptotic optimality is the one picking $z_F(\mu)$. To be asymptotically optimal, an $\varepsilon$-BAI algorithm has to implicitly identify $z_{F}(\mu)$.

The definition of $z_F(\mu)$ comes from an asymptotic lower bound, and no finite time lower bounds are available for $\varepsilon$-BAI.
It could be that for larger $\delta$ (hence small stopping times), identifying $z_F(\mu)$ among $\cZ_\varepsilon(\mu)$ is too costly to be done before stopping.
In that regime, it could be that an algorithm cannot do better than picking any $\varepsilon$-optimal answer.
This is an interesting open question for future work.
Strong moderate confidence terms (independent of $\delta$) affecting the sample complexity have been shown in different settings \citep{katz_samuels_2020_TrueSampleComplexity,mason_2020_FindingAllEpsilon}.

\begin{figure}[ht]
	\centering
	\includegraphics[width=0.485\linewidth]{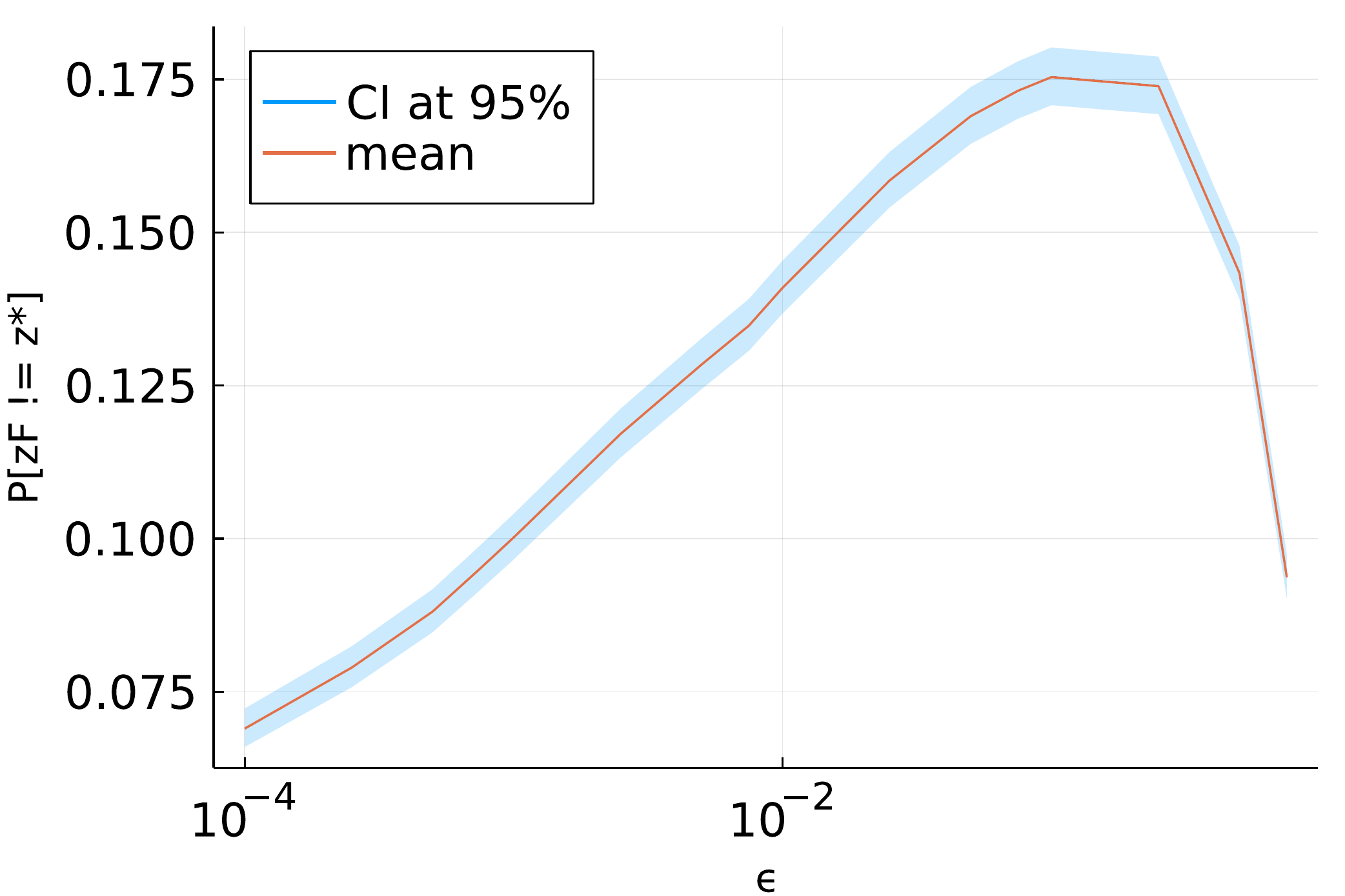}
	\includegraphics[width=0.485\linewidth]{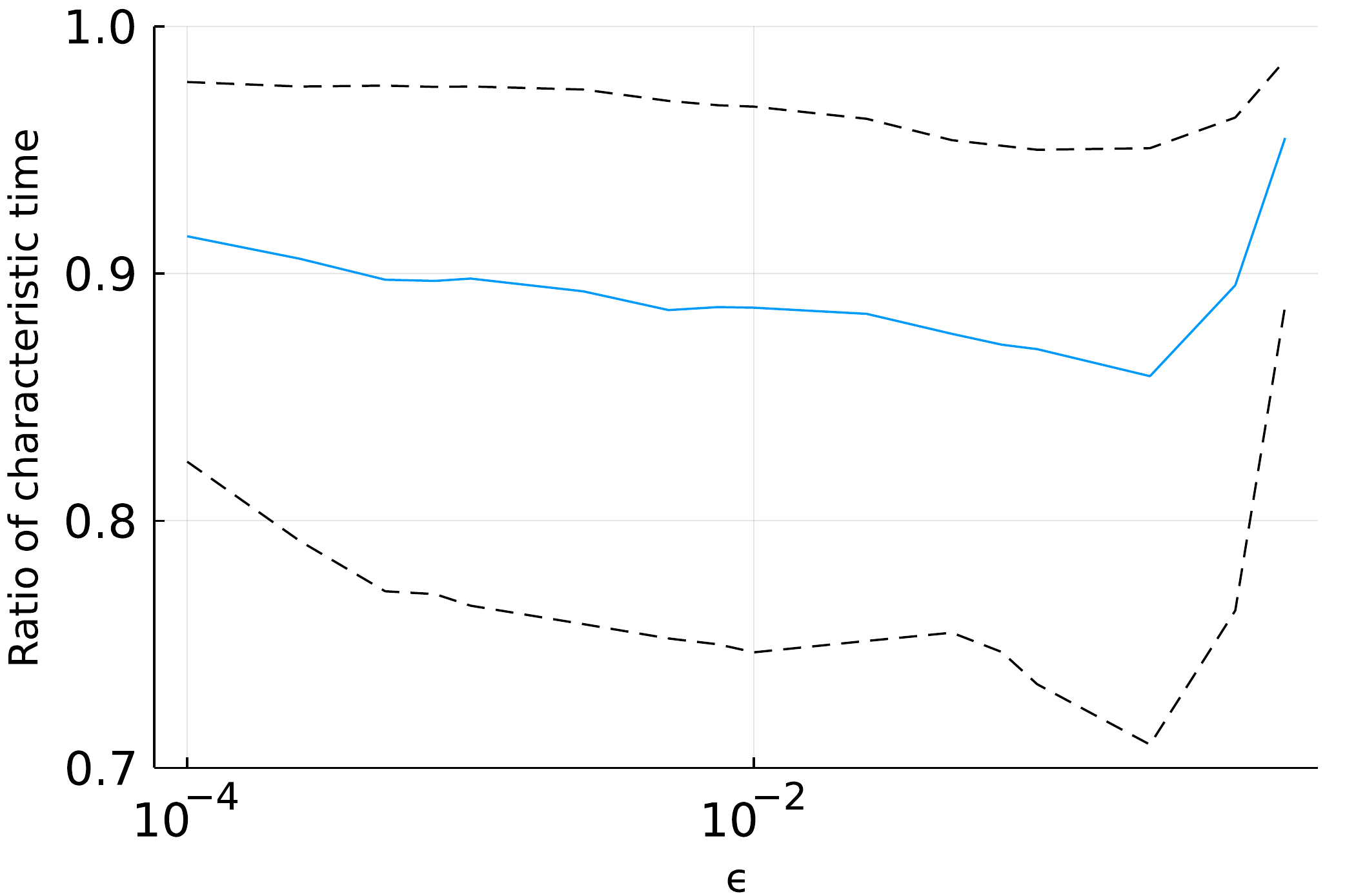}
	\caption{Influence of $\varepsilon$ on (a) the proportion of draws where $z_{F}(\mu) \notin z^{\star}(\mu)$, (b) the median (and first/third quartile) of $\frac{T_{\varepsilon}^{\mul}(\mu)}{T_{g,\varepsilon}^{\mul}(\mu)}$, when $z_{F}(\mu) \notin z^{\star}(\mu)$.}
	\label{fig:theoretical_study_greedy_vs_furthest_mul}
\end{figure}

\paragraph{Numerical Simulations} We compare the furthest and the greedy answers for the multiplicative $\varepsilon$-optimality (see Figure~\ref{fig:theoretical_study_greedy_vs_furthest_add} in Appendix~\ref{app:subsubsection_add_exp_additive} for additive $\varepsilon$-optimality).

We consider $d=2$, $\cM=\Real^2$ and $\cZ = \cK$ with $K =4$. We use $\mu = (1,0)$ and generate $25000$ random instances. In each one of them, we consider $z_1 = \mu$ and draw uniformly at random $z_2 \in \left\{(\cos(\theta), \sin(\theta)): \theta \in [-\theta_{\varepsilon}, \theta_{\varepsilon}]\right\}$ and $z_3, z_4 \in \left\{(\cos(\theta), \sin(\theta)): \theta \in (-\pi, -\theta_{\varepsilon}) \cup (\theta_{\varepsilon}, \pi]\right\}$, where $\theta_{\varepsilon} \eqdef \arccos (1-\varepsilon)$.
This yields $z_1 = z^{\star}(\mu)$, $z_2 \in \cZ_{\varepsilon}(\mu)$ and $z_3, z_4 \in  \cZ \setminus \cZ_{\varepsilon}(\mu)$. To approximate $(T_{\varepsilon}(\mu), z_{F}(\mu))$, we discretize $\triangle_{4}$ with $10000$ vectors. This is repeated for several values of $\varepsilon$. We never observed $|z^\star(\mu)| > 1$ or $|z_{F}(\mu)| > 1$.

Figure~\ref{fig:theoretical_study_greedy_vs_furthest_mul}(a) reveals that the proportion of draws where $z^{\star}(\mu) \neq z_{F}(\mu)$ is not negligible ($\approx 14 \%$). On those instances, Figure~\ref{fig:theoretical_study_greedy_vs_furthest_mul}(b) shows that $\frac{T_{\varepsilon}^{\mul}(\mu)}{T_{g,\varepsilon}^{\mul}(\mu)}$ is on average $0.9$. Therefore, when they are different, the furthest answer has a $10\%$ lower characteristic time than greedy answers.

\section{From BAI to $\varepsilon$-BAI Algorithms} \label{sec:section_from_BAI_to_eBAI}

We propose a simple procedure to convert any BAI algorithm into an $(\varepsilon, \delta)$-PAC algorithm.
While leaving the original sampling rule unchanged, the stopping-recommendation rule are carefully chosen thanks to the concept of furthest answer.

\paragraph{Structure} Since $\varepsilon$-BAI is easier than BAI, the stopping rule of BAI algorithms has to be modified for $\varepsilon$-BAI. Instead of stopping whenever a single best arm is identified, it is enough to stop when we know that an arm is $\varepsilon$-best.
In most ($\varepsilon$-)BAI algorithms, the stopping-recommendation pair and the sampling rule can be thought as two independent blocks. There exists stopping-recommendation pairs that guarantee the strategy to be $(\varepsilon, \delta)$-PAC regardless of the sampling rule (e.g. see Lemma~\ref{lem:delta_PAC_recommendation_stopping_pair}). Therefore, we can take the sampling rule from a BAI algorithm and couple it with a stopping-recommendation pair with this property.

We will now describe such a stopping-recommendation pair for $\varepsilon$-BAI in transductive linear Gaussian bandits.
Due to its generality, this procedure can be readily adapted to tackle general distributions (e.g. sub-Gaussian) and different structures (e.g. spectral bandits) by simply adapting the stopping rule and its associated threshold.

\subsection{Stopping-Recommendation Pairs} \label{sec:subsection_stopping_recommendation_pairs}

\paragraph{Estimator} Let $N_{t-1} \in (\Real^{+})^K$ denotes the counts of pulled arms at the start of round $t$, $N_{t-1}^a = \sum_{s=1}^{t-1} \1_{\{a_s = a\}}$. We denote the Ordinary Least Square (OLS) estimator by $\mu_{t-1} = V_{N_{t-1}}^{-1} \sum_{s=1}^{t-1} X_{s}^{a_s} a_s$. When $\mu_{t-1} \in \cM$, this is also the Maximum Likelihood Estimator (MLE).

\paragraph{GLR-based Stopping Rule} Given a candidate answer $z_t \in \cZ_{\varepsilon}(\mu_{t-1})$ and the history $\cF_t$, the algorithm stops as soon as the Generalized Likelihood Ratio (GLR, Appendix~\ref{app:subsection_likelihood_ratio}) exceeds a stopping threshold $\beta(t-1, \delta)$
\begin{equation} \label{eq:definition_stopping_criterion}
	\inf_{\lambda \in \neg_{\varepsilon} z_t}  \| \mu_{t-1} - \lambda \|_{V_{N_{t-1}}}^{2} > 2\beta(t-1, \delta) \: .
\end{equation}

In Lemma~\ref{lem:delta_PAC_recommendation_stopping_pair}, we show that combining a recommendation rule such that $z_t \in \cZ_{\varepsilon}(\mu_{t-1})$ and this stopping rule is sufficient to obtain a $(\varepsilon,\delta)$-PAC strategy regardless of the sampling rule. This holds even when the stopping criterion is checked only on an infinite subset of $\Natural$. The proof (Appendix~\ref{app:section_proofs_stopping_recommendation_pair}) leverages the concentration inequalities of \citet{kaufmann_2018_MixtureMartingalesRevisited}.

\begin{lemma} \label{lem:delta_PAC_recommendation_stopping_pair}
	Let $\cT \subseteq \Natural$ with $|\cT| = \infty$. Given any sampling and recommendation rules such that $z_t \in \cZ_{\varepsilon}(\mu_{t-1})$ for all $t \in \cT$, then evaluating the stopping criterion (\ref{eq:definition_stopping_criterion}) at each time $t \in \cT$ with the threshold
	\begin{equation}  \label{eq:definition_stopping_threshold}
		\beta(t, \delta) = 2 K\ln \left(4 + \ln \left( t/K \right)\right) + K \cC_{G}\left(\frac{\ln \left( 1/\delta\right)}{K}\right)
	\end{equation}
	yields an $(\varepsilon,\delta)$-PAC strategy. $\cC_{G}(x)  \approx x + \ln(x)$ as in \eqref{eq:def_C_gaussian_kaufmann_2018_MixtureMartingalesRevisited}.
\end{lemma}

Since this result holds for any sampling and recommendation rules satisfying one mild requirement, $z_t \in \cZ_{\varepsilon}(\mu_{t-1})$ for all $t \in \cT$, this leaves open the question on how to design those two rules to stop as early as possible. Algorithms that are agnostic to the choice of the candidate $\varepsilon$-optimal answer might have a higher expected sample complexity than the ones aiming at identifying the furthest answer.

\paragraph{Recommendation Rule} Taking a greedy answer $z_t \in z^{\star}(\mu_{t-1})$ is a direct choice. Thanks to its efficient implementation, using a greedy answer is the only computationally feasible recommendation rule for combinatorial or continuous answers sets. Unfortunately, when $z_{F}(\mu) \notin z^{\star}(\mu)$, this approach leads to sub-optimal algorithms in terms of asymptotic sample complexity (Lemmas~\ref{lem:greedy_sample_complexity_lower_bound_epsBAI}-\ref{lem:algorithms_being_greedy_asymptotically}).

When $Z$ is not too large or when we disregard the computational cost, a more careful choice than the greedy one alleviates this sub-optimality.
The $\varepsilon$-optimal answers for which the GLR (l.h.s of (\ref{eq:definition_stopping_criterion})) is maximized are the \textit{instantaneous furthest answers}
\begin{align*}
	z_{F}(\mu_{t-1}, N_{t-1}) \eqdef \argmax_{z \in \cZ_{\varepsilon}(\mu_{t-1})} \inf_{\lambda \in \neg_{\varepsilon} z}  \| \mu_{t-1} - \lambda \|_{V_{N_{t-1}}}^{2} \: .
\end{align*}
By definition, $z_{F}(\mu_{t-1}, N_{t-1})$ are the $\varepsilon$-optimal answers for which we have the most evidence against $\cH_{0} = \{\mu \in \neg_{\varepsilon} z\}$ at time $t$. At a lower computational cost than using a furthest answer for the current estimator $z_t \in z_{F}(\mu_{t-1})$, we will see that using an instantaneous furthest answer enjoys similar empirical performance (sample complexity).
For all the above sets of candidate answers, the ties are broken arbitrarily.
Empirically, we only observed singletons.

\paragraph{Dependence in $K$} In linear bandits, when $K$ is large, dependencies in $K$ can be replaced by $d$ \citep{lattimore_2020_BanditAlgorithms}.
The focus of our work is to highlight the importance of carefully choosing answers, therefore having $K$ instead of $d$ is a price we are willing to pay for simpler arguments.
Prior works removed the $K$ dependency in the analysis of game-based algorithms \citep{degenne_2020_GamificationPureExploration,tirinzoni_2020_AsymptoticallyOptimalPrimal,reda_2021_DealingWithMisspecification}.

\subsection{Modified BAI Algorithms} \label{ssec:subsection_modified_BAI_algorithms}

\paragraph{Modification Procedure} Given any BAI algorithm for transductive linear Gaussian bandits, we modify it to use (\ref{eq:definition_stopping_criterion}) as stopping rule while leaving the sampling rule unchanged. By Lemma~\ref{lem:delta_PAC_recommendation_stopping_pair}, the resulting algorithm is an $(\varepsilon, \delta)$-PAC strategy. For the recommendation rule, theory (Lemmas~\ref{lem:greedy_sample_complexity_lower_bound_epsBAI}-\ref{lem:algorithms_being_greedy_asymptotically}) and experiments (Figure~\ref{fig:hardinst_empirical_stop_modified_bai_add}) both suggest to use $z_t \in z_{F}(\mu_{t-1}, N_{t-1})$ instead of $z_t \in z^{\star}(\mu_{t-1})$.
We do not prove any theoretical guarantees on the sample complexity of the modified algorithms since such results depend heavily on each sampling rule.

\paragraph{BAI Benchmarks} Lots of algorithms have been designed to tackle the BAI setting and we mention below the ones used in the experiments as benchmarks. \citet{soare_2014_BestArmIdentificationLinear} proposed a static allocation design $\cX \cY$-Static and its elimination-based improvement $\cX \cY$-Adaptive, which are linked to a $G$-optimal design. In \citet{xu_2017_FullyAdaptiveAlgorithm}, LinGapE was introduced as the first gap-based BAI algorithm. All the above BAI algorithms are not shown to be asymptotically optimal and depend on $\delta$ (except $\cX \cY$-Static). Algorithm such as DKM \citep{degenne_2019_NonAsymptoticPureExploration} and LinGame \citep{degenne_2020_GamificationPureExploration} are asymptotically optimal and their sampling rule does not depend on $\delta$.

\paragraph{Other Stopping Rules} For all BAI algorithm using a GLR-based stopping rule the $\varepsilon$-BAI stopping rule (\ref{eq:definition_stopping_criterion}) is a natural modification. Some other non-GLR-based stopping rule also have a direct extension to $\varepsilon$-BAI. This is the case for the gap-based stopping rule for additive $\varepsilon$-optimality employed by LinGapE, where we can stop when the gap is smaller than $\varepsilon$ instead of stopping when it is negative.

\subsection{Experiments} \label{ssec:subsection_experiments}

We perform experiments to highlight the empirical performance of the modified BAI algorithms on additive $\varepsilon$-BAI problems. Moreover, we show that using $z_t \in z_{F}(\mu_{t-1}, N_{t-1})$ in (\ref{eq:definition_stopping_criterion}) achieves lower empirical stopping time compared to $z_t \in z^{\star}(\mu_{t-1})$, and outperforms the $\varepsilon$-gap stopping rule with $z_t \in z^{\star}(\mu_{t-1})$.
We consider linear bandits, i.e. $\cK = \cZ$, with $\cM = \Real^d$ and $(\varepsilon, \delta) = (0.05, 0.01)$, and perform $5000$ runs. The stopping-recommendation pair is updated at each time $t$.

\paragraph{Hard Instance} We adapt the usual hard instance studied in BAI for linear bandits to enforce the existence of multiple correct answers, i.e. $|\cZ_{\varepsilon}(\mu)|>1$. Taking $\mu = e_1$ with $e_{a} = (\1_{(a'=a)})_{a' \in [d]}$, the answers set is defined as $\cZ = \left\{ e_1, \cdots, e_d, a_{d+1}, a_{d+2} \right\}$ where $a_{d+1} = \cos(\phi_1)e_1 + \sin(\phi_1) e_2 \in \cZ_{\varepsilon}(\mu)$ and $a_{d+2} = \cos(\phi_2)e_1 + \sin(\phi_2) e_2 \notin \cZ_{\varepsilon}(\mu)$. Considering $d=2$, we use $\phi_1 = r_{\varepsilon} \theta_{\varepsilon}$ and $\phi_2 = (1 +r_{\varepsilon}) \theta_{\varepsilon}$ with $\theta_{\varepsilon} = \arccos(1-\varepsilon)$ and $r_{\varepsilon} = 0.1$.

On this instance, the BAI algorithms without modification require on average $545$ times more samples than compared to their modified version (Table~\ref{tab:comparison_stopping_rule_BAI_algos_add} in Appendix~\ref{app:subsection_add_exp_multiplicative}). The discrepancy is particularly striking since the hard instance for $\varepsilon$-BAI is even harder for BAI.

\begin{figure}[ht]
	\centering
	\includegraphics[width=0.98\linewidth]{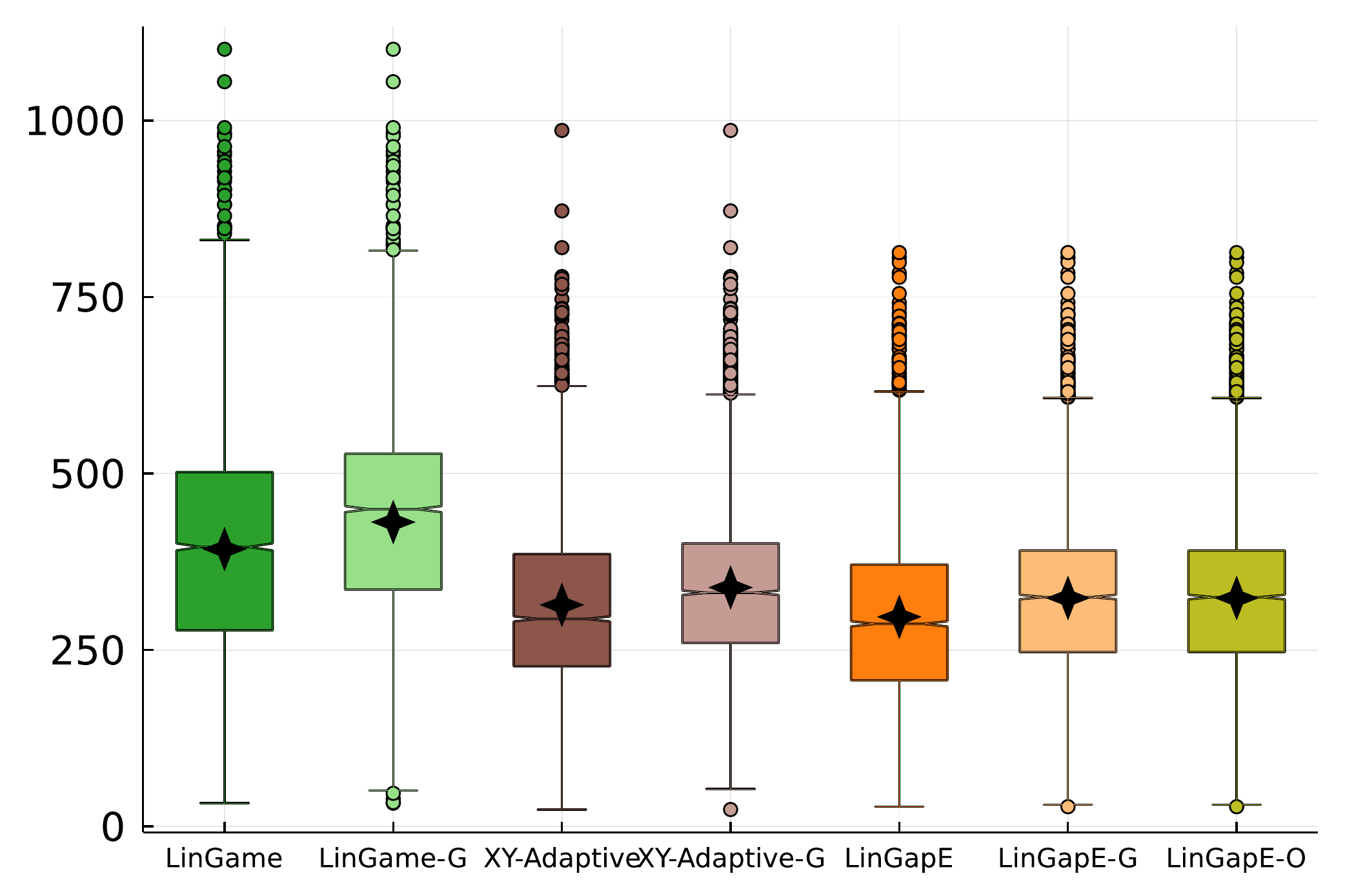}
	\caption{Empirical stopping time of the modified BAI algorithms with $z_t \in z_{F}(\mu_{t-1}, N_{t-1})$ on the hard instance (star is mean). ``-G'' denotes $z_t \in z^{\star}(\mu_{t-1})$. ``-O'' denotes the $\varepsilon$-gap stopping rule for $z_t \in z^{\star}(\mu_{t-1})$.}
	\label{fig:hardinst_empirical_stop_modified_bai_add}
\end{figure}

Figure~\ref{fig:hardinst_empirical_stop_modified_bai_add} reveals that, for all modified BAI, considering an instantaneous furthest answer instead of a greedy answer leads to lower empirical stopping time.
Their ratio is $0.92$ on average.
This matches the asymptotic observations in Figure~\ref{fig:theoretical_study_greedy_vs_furthest_add}(b) (Appendix~\ref{app:subsubsection_add_exp_additive}, equivalent of Figure~\ref{fig:theoretical_study_greedy_vs_furthest_mul}(b) for the additive setting).
The modified LinGapE using $z_t \in z_{F}(\mu_{t-1},N_{t-1})$ outperforms the $\varepsilon$-gap extension of the original stopping rule, which is equivalent to using (\ref{eq:definition_stopping_criterion}) with $z_t \in z^{\star}(\mu_{t-1})$.
While guided by the asymptotic regime, using $z_{F}(\mu_{t-1}, N_{t-1})$ instead of $z^{\star}(\mu_{t-1})$ for the stopping-recommendation pair has practical utility in the moderate confidence regime with a $10\%$ speed-up in terms of sample complexity.

\section{L$\varepsilon$BAI Algorithm} \label{sec:section_lebai}

Leveraging the concept of furthest answer in the sampling rule, we present \hyperlink{algoLeBAI}{L$\varepsilon$BAI} (\textbf{L}inear \textbf{$\varepsilon$-BAI}), an asymptotically optimal algorithm for $(\varepsilon,\delta)$-PAC best-answer identification in transductive linear bandits. It deals with both the multiplicative and the additive $\varepsilon$-optimality.
Similarly to works on linear bandits \citep{abbasi_yadkori_2011_ImprovedAlgorithmsLinear, soare_2014_BestArmIdentificationLinear}, we assume that the set of parameters is bounded, i.e. there exists $M>0$ such that for all $\mu \in \cM$, $\|\mu\|_2 \leq M$.

\begin{algorithm}[ht]
    \caption{\protect\hypertarget{algoLeBAI}{L$\varepsilon$BAI}}
    \label{algo:LeBAI}
	\begin{algorithmic}
		\STATE {\bfseries Input:} History $\cF_t$, $\cZ$-oracle $\cL^{\cZ}$ and learner $\cL^{\cK}$.
		\STATE {\bfseries Output:} Candidate $\varepsilon$-optimal answer $\hat{z}$.
        \STATE Pull once each arm $a \in \cK$, set $n_{0}=K$ and $W_{n_{0}} = 1_K$;
        \FOR{$t=n_{0} + 1, \cdots$}
        \STATE  Get $z_t \in z_{F}(\mu_{t-1}, N_{t-1})$;
        \STATE 	If (\ref{eq:definition_stopping_criterion}) holds for $z_t$ then return $z_t$;
        \STATE  Get $\left(\tilde{z}_t, w_t^{\cL^{\cK}}\right)$ from $\cL^{\cZ} \times \cL^{\cK}$;
        \STATE  Let $w_t = \frac{\1_K}{tK}  + \left( 1 - \frac{1}{t} \right) w_t^{\cL^{\cK}}$ and $W_t = W_{t-1} + w_t$;
        \STATE  Closest alternative:
        \STATE  $\quad \lambda_t \in \argmin_{\lambda \in \neg_{\varepsilon} \tilde{z}_t} \|\mu_{t-1} - \lambda\|^2_{V_{w_t}}$;
        \STATE  Optimistic gains: $\forall a \in \cK$,
        \STATE  $\quad U_t^a = \left(\|\mu_{t-1} - \lambda_t\|_{a a\transpose} + \sqrt{c_{t-1}^{a}} \right)^2$;
        \STATE  Feed $\cL^{\cK}$ with gain $g_{t}(w) = (1 - \frac{1}{t})\langle w, U_t\rangle$;
        \STATE  Pull $a_t \in \argmin_{a \in \cK} N_{t-1}^{a} - W_t^{a}$, observe $X_{t}^{a_t}$;
        \ENDFOR
	\end{algorithmic}
\end{algorithm}

\paragraph{Structure} After pulling each arm once, at each round $t \geq n_{0} + 1$, if the stopping condition (\ref{eq:definition_stopping_criterion}) for the candidate answer $z_t \in z_{F}(\mu_{t-1}, N_{t-1})$, we return $z_t$; else, the sampling rule returns an arm $a_t$ to pull. Then, the statistics are updated based on this new observation.

\paragraph{Sampling Rule} The algorithmic ingredients used in the sampling rule of \hyperlink{algoLeBAI}{L$\varepsilon$BAI} build upon the ones in LinGame \citep{degenne_2020_GamificationPureExploration}. It is a saddle-point algorithm approximating a two-player zero-sum game. At each round $t\geq n_{0}+1$, if the algorithm hasn't stopped yet, the agent chooses an $\varepsilon$-optimal answer and a pulling proportion over arms $\left(\tilde{z}_t,w_t^{\cL^{\cK}}\right) \in \cZ_{\varepsilon}(\mu_{t-1})\times \simplex$, where $\tilde{z}_t$ can be different from $z_t$. A mild logarithmic forced exploration is added, i.e. $w_t = \frac{1}{tK} \1_K + \left( 1 - \frac{1}{t} \right) w_t^{\cL^{\cK}}$. The agent will play by combining a no-regret learner on $\simplex$ (e.g. AdaHedge of \citet{derooij_2013_FollowLeaderIf}), denoted by $\cL^{\cK}$, and a $\cZ$-oracle, denoted by $\cL^{\cZ}$. While Theorem~\ref{thm:asymptotic_optimality_algorithm} was proven for $\tilde{z}_{s} \in z_{F}(\mu_{s-1})$, we obtain similar empirical performance with the heuristic $\tilde{z}_{s} \in z_{F}(\mu_{s-1}, \mu_{s-1})$ at a much lower computational cost.

Given $(\tilde{z}_t, w_t)$ from $\cL^{\cZ} \times \cL^{\cK}$, the nature plays the most confusing alternative parameter $\lambda_t \in \argmin_{\lambda \in \neg_{\varepsilon} \tilde{z}_t} \|\mu_{t-1} - \lambda\|^2_{V_{w_t}}$.
To update $\cL^{\cK}$, the agent uses gains $g_{t}(w) = (1 - \frac{1}{t})\langle w, U_t\rangle$ where the optimistic gains are defined for all $a \in \cK$ as $U_t^a = \left(\|\mu_{t-1} - \lambda_t\|_{a a\transpose} + \sqrt{c_{t-1}^{a}} \right)^2$ with $c_{t-1}^{a} = \min \left\{ 2 \beta \left(s^{2}, s^{2/3}\right)  \|a\|^2_{V_{N_{s}}^{-1}}, 4M^2L_{\cK}^2\right\}$.
Under a good event, the quantity $\langle w, U_t\rangle$ is an upper bound on the unknown $\inf_{\lambda \in \neg_{\varepsilon} z_{F}(\mu)} \|\mu - \lambda\|^2_{V_{w}}$ (Lemma~\ref{lem:optimistic_gain_is_lower_bounded_by_sample_complexity}). Finally, $a_t$ is obtained deterministically by \textit{tracking}, i.e. $a_t \in \argmin_{a \in \cK} N_{t-1}^{a} - W_t^{a}$.

The computational cost is discussed in Appendix~\ref{app:subsection_implementations_details}.
To obtain efficient implementations for combinatorial or large arms sets $K$, \hyperlink{algoLeBAI}{L$\varepsilon$BAI} should be modified by using existing improvements for game-based algorithms \citep{tirinzoni_2020_AsymptoticallyOptimalPrimal,reda_2021_DealingWithMisspecification, jourdan_2021_EfficientPureExploration}.

When $\varepsilon = 0$, \hyperlink{algoLeBAI}{L$\varepsilon$BAI} is close to LinGame, but uses one learner instead of $Z$ learners. Other differences are that LinGame uses regularization in the estimator and a stopping threshold featuring $d$.

\subsection{Upper Bound} \label{sec:subsection_sample_complexity_upper_bound}

For both the multiplicative and the additive $\varepsilon$-optimality, Theorem~\ref{thm:asymptotic_optimality_algorithm} shows that \hyperlink{algoLeBAI}{L$\varepsilon$BAI} yields an $(\varepsilon, \delta)$-PAC asymptotically optimal algorithm. The proof sketch of Theorem~\ref{thm:asymptotic_optimality_algorithm} is inspired by the one of LinGame \citep{degenne_2020_GamificationPureExploration}, hence we will only highlight the novel technical difficulties that had to be addressed.

\begin{theorem} \label{thm:asymptotic_optimality_algorithm}
Let $\cL^{\cK}$ with sub-linear regret (e.g. AdaHedge) and $\cL^{\cZ}$ returning $\tilde{z}_{t} \in z_{F}(\mu_{t-1})$.
Using (\ref{eq:definition_stopping_threshold}) as stopping threshold $\beta(t,\delta)$, \hyperlink{algoLeBAI}{L$\varepsilon$BAI} yields an $(\varepsilon, \delta)$-PAC algorithm and, for all $\mu \in \cM$ such that $|z_{F}(\mu)|=1$,
\begin{align*}
	\limsup_{\delta \rightarrow 0} \frac{\expectedvalue_{\mu} \left[ \taud \right]}{\ln \left( 1/\delta\right)} \leq T_{\varepsilon}(\mu) \: .
\end{align*}
\end{theorem}

\paragraph{Technical Difficulties} In BAI, we have $|z^{\star}(\mu)|=1$. The key property used in BAI proofs which does not hold in $\varepsilon$-BAI is that for all $z \neq z^{\star}(\mu)$, $\mu$ belongs to the alternative $\neg_{0} z$. The consequence of this is that whenever the answer used by the sampling rule $\tilde{z}_t$ is wrong, the correct parameter belongs to $\neg_0 \tilde{z}_t$, hence the algorithm will sample in order to try and exclude that true parameter, which cannot succeed and will at some point correct the mistake. In $\varepsilon$-BAI we can have $\tilde{z}_t \ne z_F(\mu)$ while having $\mu \notin \neg_{\varepsilon} \tilde{z}_t$ and there is a priori no such self-correction mechanism to enforce that $\tilde{z}_t = z_F(\mu)$ after a while.

Our analysis reveals that a similar self-correction mechanism can be obtained for \hyperlink{algoLeBAI}{L$\varepsilon$BAI}.
Let $\neg_{F} z$ be the \textit{furthest alternative} to $z$, i.e. the set of parameters for which $z$ is not the unique furthest answer.
Intuitively, as it uses $\tilde{z}_t \in z_F(\mu_{t-1})$, \hyperlink{algoLeBAI}{L$\varepsilon$BAI} samples to asymptotically exclude $\neg_F \tilde{z}_t$.
Leveraging the logarithmic forced exploration, this cannot succeed when $\tilde{z}_t \ne z_F(\mu)$.
Those two choices yield a self-correction mechanism for $\varepsilon$-BAI.
More formally, we show that, under a good concentration event, the event $\tilde{z}_s \ne z_F(\mu)$ only happens a sub-linear number of times.

\section{Related Work} \label{sec:section_related_work}

\paragraph{Track-And-Stop (TaS)} First introduced in \citet{garivier_2016_OptimalBestArm} to solve BAI in MAB, TaS computes at each time step the optimal allocation, and then tracks it (with added forced exploration). When no close form solutions are available, i.e. when additional structure is considered, TaS-based algorithms suffer from intractable computational cost. TaS-based algorithms are asymptotically optimal for BAI and efficient to compute \citep{jedra_2020_OptimalBestarmIdentification}. Building on Frank-Wolfe algorithm, the computational efficient FWS has recently been introduced \citep{wang_2021_FastPureExploration}.

\paragraph{$\varepsilon$-BAI Algorithms} Tackling $\varepsilon$-BAI in MAB for additive $\varepsilon$-optimality, $\varepsilon$-TaS \citep{garivier_2019_NonAsymptoticSequentialTests} recommends $z_{F}(\mu_t,N_t)$ and uses the associated GLRT as stopping rule. The sampling rule computes $w_{F}(\mu_{t})$ and then tracks it with added forced exploration.
Addressing additive spectral bandits, SpectralTaS \citep{kocak_2021_EpsilonBestArm} recommends $z^{\star}(\mu_t)$ and uses the GLRT associated with $z_{F}(\mu_t,N_t)$ for the stopping rule. For the sampling rule, a mirror ascent algorithm is run based on a super-gradient of a function depending on any $\varepsilon$-optimal answer. While the choice of the answer is not discussed, it is our understanding that a greedy answer is used (matching their candidate answer). When considering $\varepsilon$-BAI on the unit sphere, \citet{jedra_2020_OptimalBestarmIdentification} recommend $z^{\star}(\mu_t)$ and use the associated GLRT, however their sampling rule is uniform over a spanner.

Designed for the multiple-correct answer setting, Sticky TaS \citep{degenne_2019_PureExplorationMultiple} is a modified TaS algorithm: at round $t$, they compute $\bigcup_{\mu' \in \mathcal{C}_t}z_{F}(\mu')$ where $\mathcal{C}_t$ is a continuous confidence region around $\mu_{t}$, and stick to one of those (given an arbitrary order).
For some identification problems (e.g. Any Half-Space), it rewrites as computing $\bigcup_{\mu \in \mathcal D_t }z_{F}(\mu')$, where $\mathcal D_t$ is discrete.
There is no such rewriting for $\varepsilon$-BAI.
While Sticky TaS can be implemented for $\mathcal D_t$, it is not feasible for $\mathcal{C}_t$.
Experiments suggest that it performs on par with $\varepsilon$-TaS at a higher computational cost, i.e. solving the same optimization for each parameter in a confidence region.

\section{Experiments} \label{sec:section_experiments}

We show that \hyperlink{algoLeBAI}{L$\varepsilon$BAI} has competitive empirical performance compared to existing $\varepsilon$-BAI algorithms, which are computationally expensive, and that using an instantaneous furthest answer is efficient both in terms of computational cost and sample complexity. Moreover, \hyperlink{algoLeBAI}{L$\varepsilon$BAI} performs on par with the modified BAI algorithms, which are not asymptotically optimal, on hard and random instances.

As heuristic with lower computational cost (not supported by Theorem~\ref{thm:asymptotic_optimality_algorithm}), the $\cZ$-oracle in \hyperlink{algoLeBAI}{L$\varepsilon$BAI} returns an instantaneous furthest answer, i.e. $\tilde{z}_t \in z_{F}(\mu_{t-1}, N_{t-1})$. The experiments below are considering the multiplicative $\varepsilon$-optimality (Appendix~\ref{app:subsection_add_exp_multiplicative}-\ref{app:subsubsection_add_exp_additive} for supplementary ones). We use the same experimental setup as in Section~\ref{ssec:subsection_experiments}. On the $5000$ runs, we report the standard deviation of means by using sub-samples of $100$ runs.

\begin{table}[ht]
\caption{Empirical stopping time ($\pm$ $\sigma$) on the hard instance with $\cK=\{e_1,e_2\}$.}
\label{tab:average_empirical_stopping_time_mul}
\begin{center}
\begin{tabular}{c  c  c  c }
  \toprule
  & $z^{\star}(\mu_{t-1})$ & $z_F(\mu_{t-1})$ & $z_F(\mu_{t-1}, N_{t-1})$\\
\cmidrule(l){2-4}
L$\varepsilon$BAI & $416$ $(\pm 13)$ & $383$ $(\pm 16)$ & $381$ $(\pm 17)$ \\
$\varepsilon$-TaS & $400$ $(\pm 14)$ & $371$ $(\pm 15)$ & $371$ $(\pm 15)$ \\
Fixed           & $401$ $(\pm 14)$ & $374$ $(\pm 14)$ & $374$ $(\pm 14)$ \\
Uniform         & $492$ $(\pm 16)$ & $450$ $(\pm 17)$ & $449$ $(\pm 17)$ \\
\bottomrule
\end{tabular}
\end{center}
\end{table}

\paragraph{$\varepsilon$-BAI and Candidate Answer} Even when $K$ is small, algorithms based on solving the optimization problem $(z_{F}(\mu), w_{F}(\mu))$ are intractable, i.e. $\varepsilon$-TaS or recommending the furthest answer. We evaluate their performance empirically on the hard instance with $\cK=\{e_1,e_2\}$, and discretize uniformly $\Delta_{2}$ with $500$ vectors.

In Table~\ref{tab:average_empirical_stopping_time_mul}, we combine and compare four $\varepsilon$-BAI sampling rules with three candidate answers for the stopping rule (\ref{eq:definition_stopping_criterion}). Comparing the rows of Table~\ref{tab:average_empirical_stopping_time_mul} reveals that \hyperlink{algoLeBAI}{L$\varepsilon$BAI} performs on par with $\varepsilon$-TaS and the ``oracle'' \textit{fixed} algorithm, which tracks the unknown optimal allocation $w_{F}(\mu)$. It also consistently outperforms uniform sampling ($\approx 85 \%$).

Comparing the columns of Table~\ref{tab:average_empirical_stopping_time_mul}, we see that using a greedy answer is consistently worse than a (instantaneous) furthest answer, with a ratio of stopping time being on average $0.92$ (coherent with Figure~\ref{fig:theoretical_study_greedy_vs_furthest_mul}(b)). Moreover, it highlights that using an instantaneous furthest answer achieves similar performance as a furthest answer at a lower computational cost.
In the following experiments, the stopping-recommendation pair is (\ref{eq:definition_stopping_criterion}) combined with $z_t \in z_F(\mu_{t-1}, N_{t-1})$.

\begin{figure}[ht]
	\centering
	\includegraphics[width=0.98\linewidth]{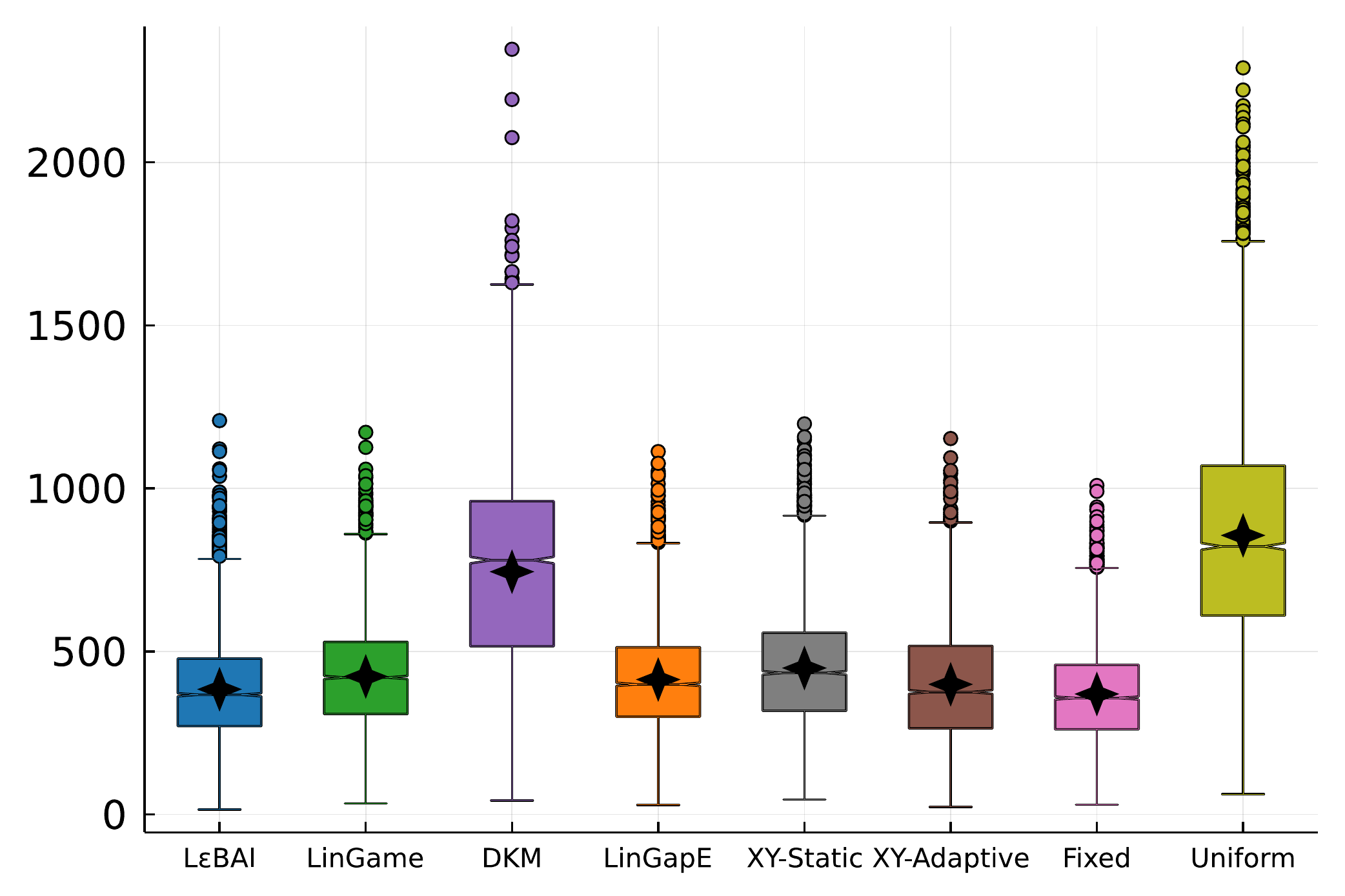}
	\caption{Empirical stopping time on the hard instance ($\cK = \cZ$). The modified BAI algorithms use (\ref{eq:definition_stopping_criterion}) with $z_t \in z_{F}(\mu_{t-1}, N_{t-1})$.}
	\label{fig:hardinst_empirical_stop_lebai_vs_modified_bai_mul}
\end{figure}

\paragraph{Modified BAI Algorithms} Figure~\ref{fig:hardinst_empirical_stop_lebai_vs_modified_bai_mul} compares \hyperlink{algoLeBAI}{L$\varepsilon$BAI} with the modified BAI algorithms, all using the same stopping-recommendation pair. We see that \hyperlink{algoLeBAI}{L$\varepsilon$BAI} slightly outperforms the modified LinGapE and $\cX\cY$-Adaptive, performs better than  the modified LinGame and $\cX\cY$-Static and is on par with the ``oracle'' \textit{fixed} algorithm. Uniform sampling and the modified DKM perform poorly.

\begin{figure*}[ht]
	\centering
	\includegraphics[width=0.24\linewidth]{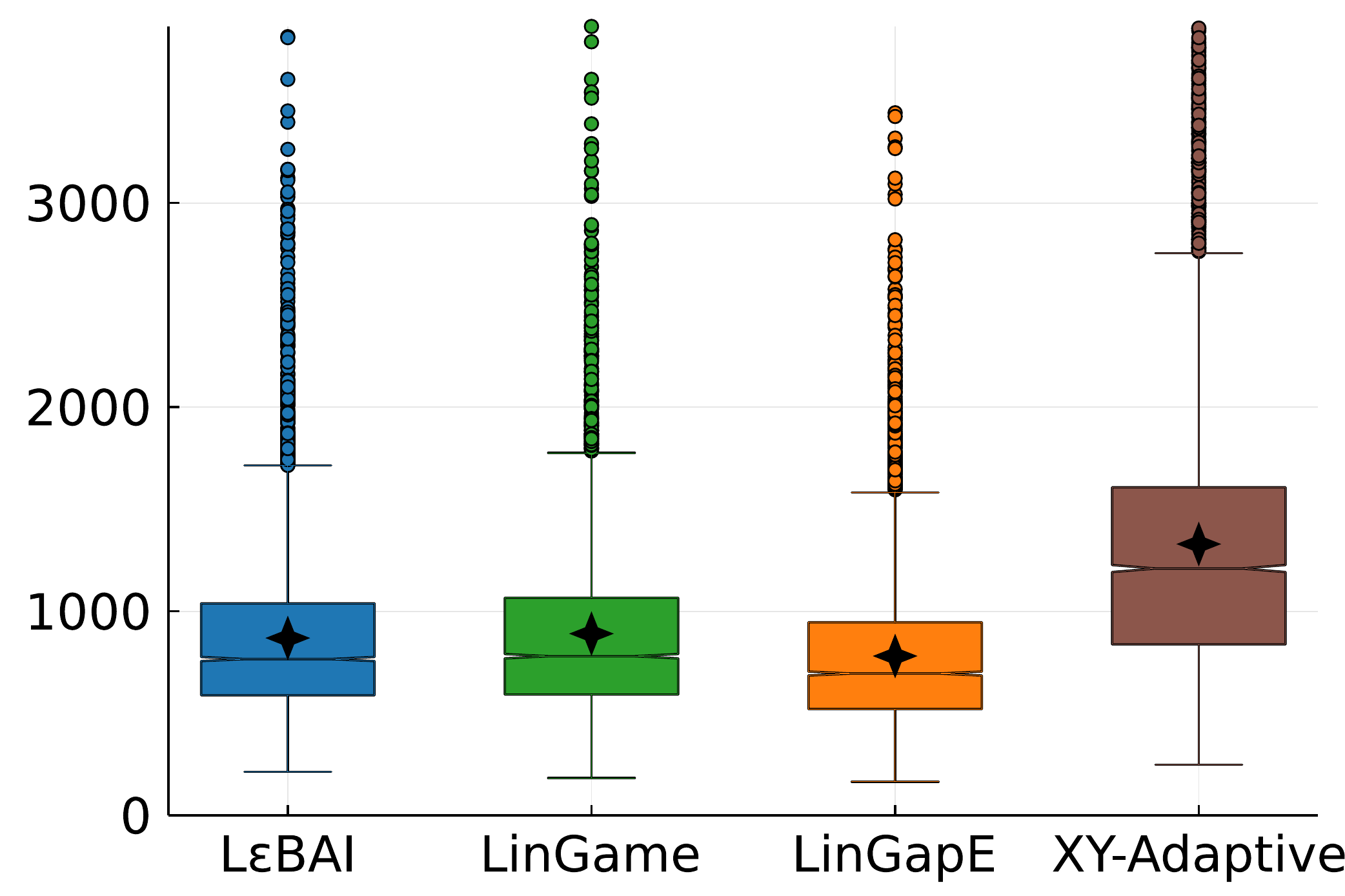}
	\includegraphics[width=0.24\linewidth]{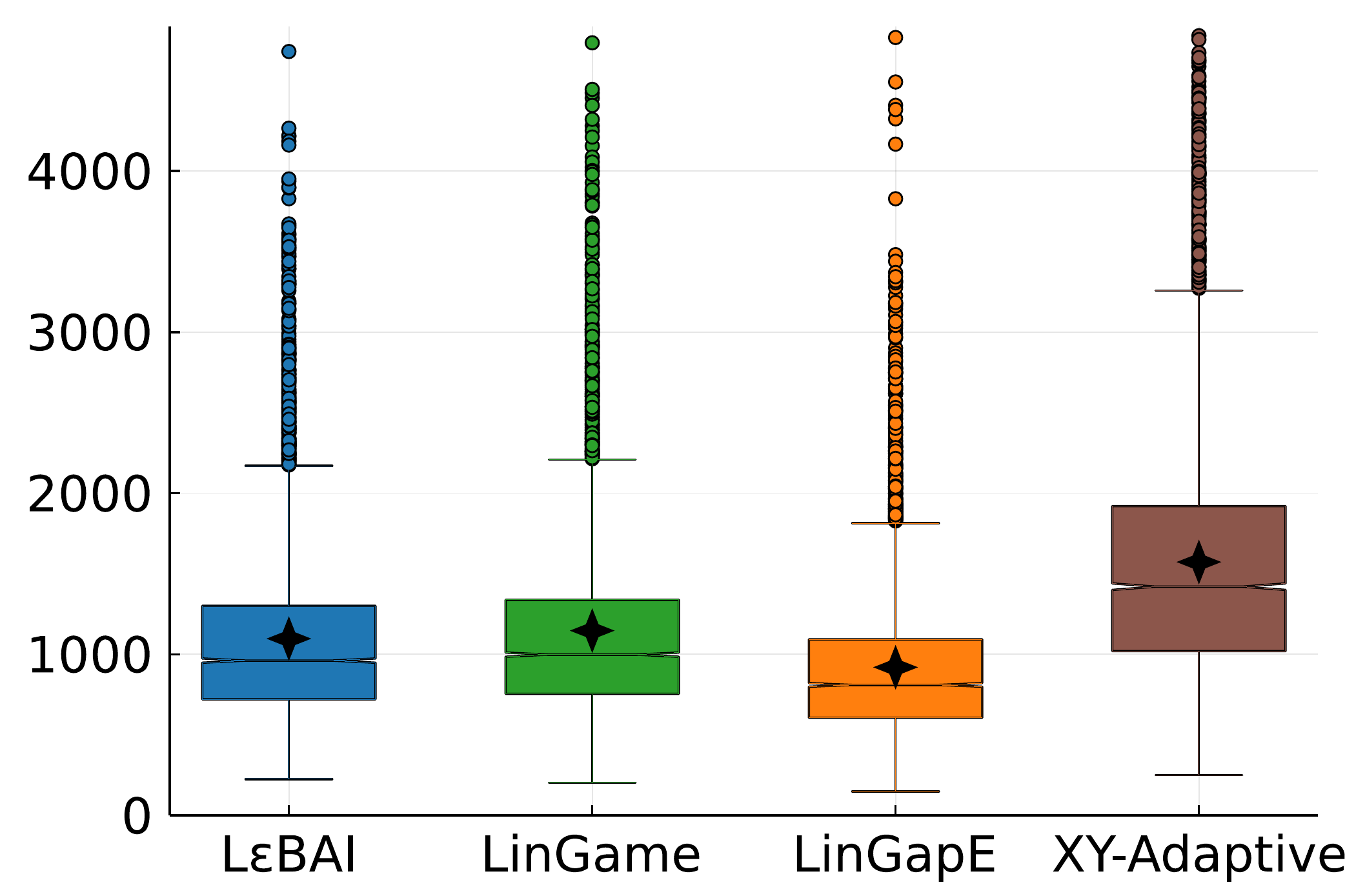}
	\includegraphics[width=0.24\linewidth]{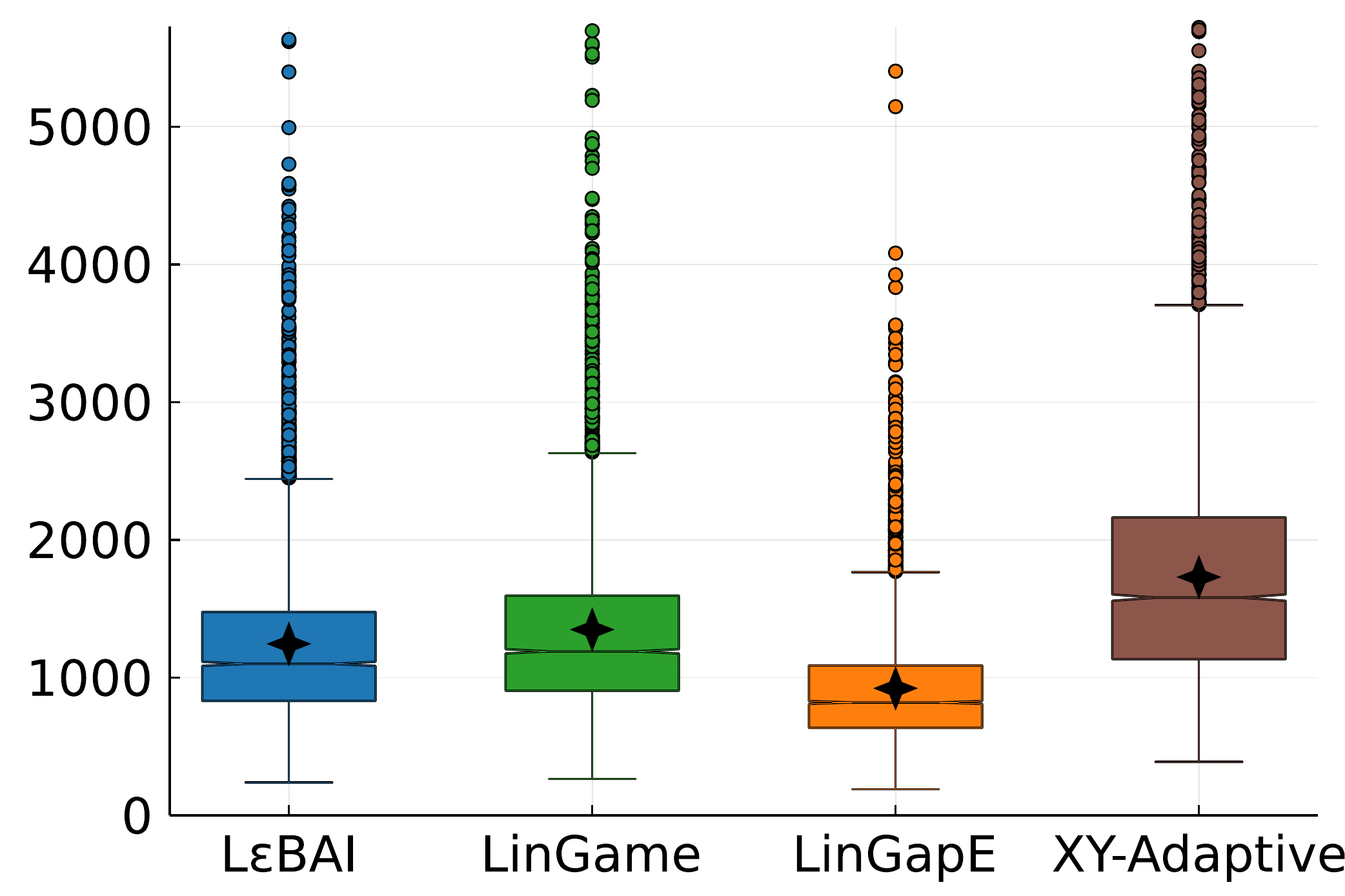}
	\includegraphics[width=0.24\linewidth]{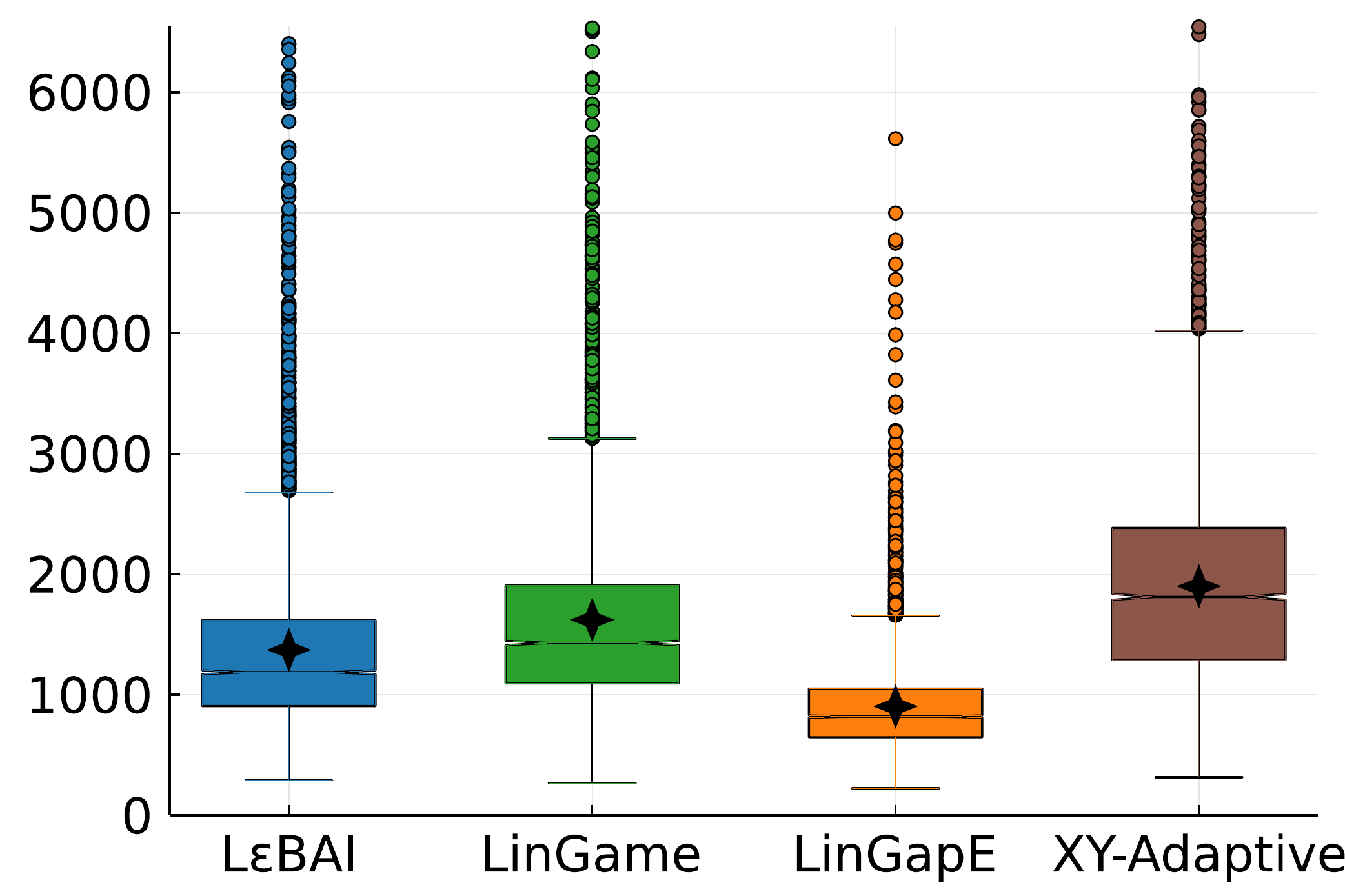}
	\caption{Empirical stopping time on random instances ($\cK=\cZ$) for $d \in \{6,8,10,12\}$ (from left to right). The modified BAI algorithms use (\ref{eq:definition_stopping_criterion}) with $z_t \in z_{F}(\mu_{t-1}, N_{t-1})$.}
	\label{fig:randinst_empirical_stop_dimensions_mul}
\end{figure*}

\paragraph{Random Instances} To assess the impact of higher dimensions, random instances are considered (one per run). For the answer set, $19$ vectors $(a_{k})_{k \in [19]}$ are uniformly drawn from $\bbS^{d-1} \eqdef \left\{a \in \Real^d: \|a\|_2 = 1\right\}$ and set $\mu = a_1$. To enforce multiple correct answers, a modification of the greedy answer is added such that $a_{20,i} = a_{1,i}$ for $i\neq i_{0}$ and $a_{20,i_0} = \frac{1-\|\mu\|^2_2 + \mu_{i_0}^2 - r_{\varepsilon}\varepsilon}{\mu_{i_0}}$ where $i_{0} = \argmin_{i \in [d]} \mu_i$ and $r_{\varepsilon} = 0.1$. Those instances are motivated by a practical BAI example where a modified/corrupted version of the unique correct answer exists. Seeing the problem as an $\varepsilon$-BAI one allows to return an $\varepsilon$-optimal answer, while avoiding wasteful queries required by BAI algorithms (Table~\ref{tab:comparison_stopping_rule_BAI_algos_mul}).

In Figure~\ref{fig:randinst_empirical_stop_dimensions_mul}, \hyperlink{algoLeBAI}{L$\varepsilon$BAI} shows similar empirical performance with modified BAI algorithms. Even though it is outperformed by the modified LinGapE, \hyperlink{algoLeBAI}{L$\varepsilon$BAI} is almost twice as fast as the modified $\cX\cY$-Adaptive and appears to be slightly more robust than the modified LinGame to increasing dimension.

\section{Conclusion} \label{sec:section_conclusion}

In $(\varepsilon, \delta)$-PAC best-answer identification for transductive linear bandits, we have shown that the choice of the candidate $\varepsilon$-optimal answer is important for the sample complexity. Using an instantaneous furthest answer as candidate answer, we proposed a simple procedure to adapt existing BAI algorithms for $\varepsilon$-BAI problems. Leveraging it in the sampling rule as well, we introduced \hyperlink{algoLeBAI}{L$\varepsilon$BAI} which is asymptotically optimal and has competitive empirical performance.

Computing the furthest answer requires solving the closest alternative sub-problem $|\cZ_{\varepsilon}(\mu_{t-1})|$ times. While that number is small (in particular much less that $Z$) in the examples we considered, that computation can become an issue if many different answers are close to each other.
If we extend the setting to continuous answers, the computation of the furthest answer by iterating becomes unfeasible. The question of finding an $\varepsilon$-optimal point of a reward function in a non-finite set is the general question of optimization, which is central to many areas of machine learning. Extending the problem-dependent approach of the bandit framework to that setting is an interesting research direction.

Studying the performance of $\varepsilon$-BAI algorithms on BAI tasks is another avenue for future work.
$\varepsilon$-BAI algorithms don't commit to the greedy answer unlike BAI methods, and we might have $z_{F}(\mu_{t-1}, N_{t-1}) = z^{\star}(\mu)$ before $z^{\star}(\mu_{t-1})= z^{\star}(\mu)$.
For this reason we conjecture that using an $\varepsilon$-BAI method with a well-tuned decreasing sequence $(\varepsilon_t)_{t \in \Natural^{\star}}$ might outperform BAI algorithms on best arm tasks.

Finally, since the existence of a tight finite-time lower bound for multiple-correct answer setting is still an open problem, it remains unclear how to assess the theoretical performance of algorithms in this regime.
We believe that, once derived, this lower bound would reveal the existence of strong moderate confidence terms (independent of $\delta$) affecting the sample complexity, which could then be used to design $\varepsilon$-BAI algorithms with theoretical guarantees in both regime.

\section*{Acknowledgements}

Experiments presented in this paper were carried out using the Grid'5000 testbed, supported by a scientific interest group hosted by Inria and including CNRS, RENATER and several Universities as well as other organizations (see \url{https://www.grid5000.fr}).


\bibliography{AlgoLeBAI}

\begin{thebibliography}{46}
\providecommand{\natexlab}[1]{#1}
\providecommand{\url}[1]{\texttt{#1}}
\expandafter\ifx\csname urlstyle\endcsname\relax
  \providecommand{\doi}[1]{doi: #1}\else
  \providecommand{\doi}{doi: \begingroup \urlstyle{rm}\Url}\fi

\bibitem[{Abbasi-yadkori} et~al.(2011){Abbasi-yadkori}, P{\'a}l, and
  Szepesv{\'a}ri]{abbasi_yadkori_2011_ImprovedAlgorithmsLinear}
{Abbasi-yadkori}, Y., P{\'a}l, D., and Szepesv{\'a}ri, C.
\newblock Improved algorithms for linear stochastic bandits.
\newblock In \emph{Advances in Neural Information Processing Systems}, 2011.

\bibitem[Al~Marjani et~al.(2022)Al~Marjani, Kocak, and
  Garivier]{almarjani_2022_ComplexityAllEpsilon}
Al~Marjani, A., Kocak, T., and Garivier, A.
\newblock {On the complexity of All $\epsilon$-Best Arms Identification}.
\newblock 2022.

\bibitem[Audibert et~al.(2010)Audibert, Bubeck, and
  Munos]{audibert_2010_BestArmIdentification}
Audibert, J., Bubeck, S., and Munos, R.
\newblock Best arm identification in multi-armed bandits.
\newblock In \emph{The 23rd Conference on Learning Theory (COLT)}, 2010.

\bibitem[Bubeck et~al.(2009)Bubeck, Munos, and
  Stoltz]{bubeck_2009_PureExplorationMultiArmed}
Bubeck, S., Munos, R., and Stoltz, G.
\newblock Pure exploration in multi-armed bandits problems.
\newblock In \emph{Algorithmic Learning Theory (ALT)}, 2009.

\bibitem[Chen et~al.(2017)Chen, Li, and Qiao]{chen_2017_InstanceOptimalBounds}
Chen, L., Li, J., and Qiao, M.
\newblock Towards instance optimal bounds for best arm identification.
\newblock In \emph{Proceedings of the 2017 Conference on Learning Theory},
  2017.

\bibitem[Chen et~al.(2014)Chen, Lin, King, Lyu, and
  Chen]{chen_2014_CombinatorialPureExploration}
Chen, S., Lin, T., King, I., Lyu, M.~R., and Chen, W.
\newblock Combinatorial {{Pure Exploration}} of {{Multi}}-{{Armed Bandits}}.
\newblock In \emph{Advances in {{Neural Information Processing Systems}} 27},
  2014.

\bibitem[Chernoff(1959)]{chernoff_1959_SequentialDesignExperiments}
Chernoff, H.
\newblock Sequential {{Design}} of {{Experiments}}.
\newblock \emph{The Annals of Mathematical Statistics}, 30\penalty0
  (3):\penalty0 755--770, 1959.

\bibitem[Cheshire et~al.(2021)Cheshire, Menard, and
  Carpentier]{cheshire_2021_ProblemDependentView}
Cheshire, J., Menard, P., and Carpentier, A.
\newblock Problem {{Dependent View}} on {{Structured Thresholding Bandit
  Problems}}.
\newblock In \emph{International {{Conference}} on {{Machine Learning}}}, 2021.

\bibitem[De~Rooij et~al.(2014)De~Rooij, Van~Erven, Gr{\"u}nwald, and
  Koolen]{derooij_2013_FollowLeaderIf}
De~Rooij, S., Van~Erven, T., Gr{\"u}nwald, P.~D., and Koolen, W.~M.
\newblock Follow the leader if you can, hedge if you must.
\newblock \emph{The Journal of Machine Learning Research}, 15\penalty0
  (1):\penalty0 1281--1316, 2014.

\bibitem[Degenne \& Koolen(2019)Degenne and
  Koolen]{degenne_2019_PureExplorationMultiple}
Degenne, R. and Koolen, W.~M.
\newblock Pure exploration with multiple correct answers.
\newblock In \emph{Advances in Neural Information Processing Systems}, 2019.

\bibitem[Degenne et~al.(2019)Degenne, Koolen, and
  M{\'e}nard]{degenne_2019_NonAsymptoticPureExploration}
Degenne, R., Koolen, W.~M., and M{\'e}nard, P.
\newblock Non-{{Asymptotic Pure Exploration}} by {{Solving Games}}.
\newblock In \emph{Advances in {{Neural Information Processing Systems}}},
  2019.

\bibitem[Degenne et~al.(2020{\natexlab{a}})Degenne, M{\'{e}}nard, Shang, and
  Valko]{degenne_2020_GamificationPureExploration}
Degenne, R., M{\'{e}}nard, P., Shang, X., and Valko, M.
\newblock Gamification of pure exploration for linear bandits.
\newblock In \emph{International Conference on Machine Learning},
  2020{\natexlab{a}}.

\bibitem[Degenne et~al.(2020{\natexlab{b}})Degenne, Shao, and
  Koolen]{degenne_2020_StructureAdaptiveAlgorithms}
Degenne, R., Shao, H., and Koolen, W.
\newblock Structure {{Adaptive Algorithms}} for {{Stochastic Bandits}}.
\newblock In \emph{International Conference on Machine Learning},
  2020{\natexlab{b}}.

\bibitem[Even{-}Dar et~al.(2002)Even{-}Dar, Mannor, and
  Mansour]{even_dar_2002_PACBoundsMultiarmed}
Even{-}Dar, E., Mannor, S., and Mansour, Y.
\newblock {PAC} bounds for multi-armed bandit and markov decision processes.
\newblock In \emph{Computational Learning Theory (COLT)}, 2002.

\bibitem[{Even-Dar} et~al.(2006){Even-Dar}, Mannor, and
  Mansour]{even_dar_2006_ActionEliminationStopping}
{Even-Dar}, E., Mannor, S., and Mansour, Y.
\newblock Action {{Elimination}} and {{Stopping Conditions}} for the
  {{Multi}}-{{Armed Bandit}} and {{Reinforcement Learning Problems}}.
\newblock \emph{Journal of Machine Learning Research}, 7:\penalty0 1079--1105,
  2006.

\bibitem[Fiez et~al.(2019)Fiez, Jain, Jamieson, and
  Ratliff]{fiez_2019_SequentialExperimentalDesign}
Fiez, T., Jain, L., Jamieson, K.~G., and Ratliff, L.~J.
\newblock Sequential experimental design for transductive linear bandits.
\newblock In \emph{Advances in Neural Information Processing Systems}, 2019.

\bibitem[Gabillon et~al.(2012)Gabillon, Ghavamzadeh, and
  Lazaric]{gabillon_2012_BestArmIdentification}
Gabillon, V., Ghavamzadeh, M., and Lazaric, A.
\newblock Best {{Arm Identification}}: {{A Unified Approach}} to {{Fixed
  Budget}} and {{Fixed Confidence}}.
\newblock In \emph{Advances in {{Neural Information Processing Systems}} 25},
  2012.

\bibitem[Garivier \& Kaufmann(2016)Garivier and
  Kaufmann]{garivier_2016_OptimalBestArm}
Garivier, A. and Kaufmann, E.
\newblock Optimal best arm identification with fixed confidence.
\newblock In \emph{Conference on Learning Theory (COLT)}, 2016.

\bibitem[Garivier \& Kaufmann(2021)Garivier and
  Kaufmann]{garivier_2019_NonAsymptoticSequentialTests}
Garivier, A. and Kaufmann, E.
\newblock Nonasymptotic sequential tests for overlapping hypotheses applied to
  near-optimal arm identification in bandit models.
\newblock \emph{Sequential Analysis}, 40\penalty0 (1):\penalty0 61--96, 2021.

\bibitem[Hong et~al.(2021)Hong, Fan, and Luo]{hong_2021_ReviewRankingSelection}
Hong, L.~J., Fan, W., and Luo, J.
\newblock Review on ranking and selection: {{A}} new perspective.
\newblock \emph{Frontiers of Engineering Management}, 8\penalty0 (3):\penalty0
  321--343, 2021.

\bibitem[Jamieson \& Nowak(2014)Jamieson and
  Nowak]{jamieson_2014_BestarmIdentificationAlgorithms}
Jamieson, K. and Nowak, R.
\newblock Best-arm identification algorithms for multi-armed bandits in the
  fixed confidence setting.
\newblock In \emph{{{Annual Conference}} on {{Information Sciences}} and
  {{Systems}} ({{CISS}})}, 2014.

\bibitem[Jedra \& Prouti{\`{e}}re(2020)Jedra and
  Prouti{\`{e}}re]{jedra_2020_OptimalBestarmIdentification}
Jedra, Y. and Prouti{\`{e}}re, A.
\newblock Optimal best-arm identification in linear bandits.
\newblock In \emph{Advances in Neural Information Processing Systems}, 2020.

\bibitem[Jourdan et~al.(2021)Jourdan, Mutn{\`y}, Kirschner, and
  Krause]{jourdan_2021_EfficientPureExploration}
Jourdan, M., Mutn{\`y}, M., Kirschner, J., and Krause, A.
\newblock Efficient pure exploration for combinatorial bandits with semi-bandit
  feedback.
\newblock In \emph{Algorithmic Learning Theory (ALT)}, 2021.

\bibitem[Jun et~al.(2021)Jun, Jain, Nassif, and
  Mason]{jun_2021_ImprovedConfidenceBounds}
Jun, K.-S., Jain, L., Nassif, H., and Mason, B.
\newblock Improved {{Confidence Bounds}} for the {{Linear Logistic Model}} and
  {{Applications}} to {{Bandits}}.
\newblock In \emph{International {{Conference}} on {{Machine Learning}}}, 2021.

\bibitem[Kalyanakrishnan et~al.(2012)Kalyanakrishnan, Tewari, Auer, and
  Stone]{kalyanakrishnan_2012_PACSubsetSelection}
Kalyanakrishnan, S., Tewari, A., Auer, P., and Stone, P.
\newblock {PAC} subset selection in stochastic multi-armed bandits.
\newblock In \emph{International Conference on Machine Learning (ICML)}, 2012.

\bibitem[Katz-Samuels \& Jamieson(2020)Katz-Samuels and
  Jamieson]{katz_samuels_2020_TrueSampleComplexity}
Katz-Samuels, J. and Jamieson, K.
\newblock The true sample complexity of identifying good arms.
\newblock In \emph{International Conference on Artificial Intelligence and
  Statistics}, 2020.

\bibitem[{Katz-Samuels} \& Scott(2019){Katz-Samuels} and
  Scott]{katz_samuels_2019_TopFeasibleArm}
{Katz-Samuels}, J. and Scott, C.
\newblock Top {{Feasible Arm Identification}}.
\newblock In \emph{{{International Conference}} on {{Artificial Intelligence}}
  and {{Statistics}}}, 2019.

\bibitem[Katz-Samuels et~al.(2020)Katz-Samuels, Jain, Jamieson,
  et~al.]{katz_samuels_2020_EmpiricalProcessApproach}
Katz-Samuels, J., Jain, L., Jamieson, K.~G., et~al.
\newblock An empirical process approach to the union bound: Practical
  algorithms for combinatorial and linear bandits.
\newblock \emph{Advances in Neural Information Processing Systems}, 2020.

\bibitem[Kaufmann \& Koolen(2021)Kaufmann and
  Koolen]{kaufmann_2018_MixtureMartingalesRevisited}
Kaufmann, E. and Koolen, W.~M.
\newblock Mixture martingales revisited with applications to sequential tests
  and confidence intervals.
\newblock \emph{Journal of Machine Learning Research}, 22\penalty0
  (246):\penalty0 1--44, 2021.

\bibitem[Kaufmann et~al.(2016)Kaufmann, Capp{\'e}, and
  Garivier]{kaufmann_2016_ComplexityBestArm}
Kaufmann, E., Capp{\'e}, O., and Garivier, A.
\newblock On the complexity of best-arm identification in multi-armed bandit
  models.
\newblock \emph{The Journal of Machine Learning Research}, 17\penalty0
  (1):\penalty0 1--42, 2016.

\bibitem[Kaufmann et~al.(2018)Kaufmann, Koolen, and
  Garivier]{kaufmann_2018_SequentialTestLowest}
Kaufmann, E., Koolen, W.~M., and Garivier, A.
\newblock Sequential test for the lowest mean: From thompson to murphy
  sampling.
\newblock In \emph{Advances in Neural Information Processing Systems}, 2018.

\bibitem[Koc\'ak \& Garivier(2021)Koc\'ak and
  Garivier]{kocak_2021_EpsilonBestArm}
Koc\'ak, T. and Garivier, A.
\newblock Epsilon best arm identification in spectral bandits.
\newblock In \emph{International Joint Conference on Artificial Intelligence
  (IJCAI)}, 2021.

\bibitem[Lattimore \& Szepesv{\'a}ri(2020)Lattimore and
  Szepesv{\'a}ri]{lattimore_2020_BanditAlgorithms}
Lattimore, T. and Szepesv{\'a}ri, C.
\newblock \emph{Bandit {{Algorithms}}}.
\newblock {Cambridge University Press}, 2020.

\bibitem[Locatelli et~al.(2016)Locatelli, Gutzeit, and
  Carpentier]{locatelli_2016_OptimalAlgorithmThresholding}
Locatelli, A., Gutzeit, M., and Carpentier, A.
\newblock An optimal algorithm for the thresholding bandit problem.
\newblock In \emph{International Conference on Machine Learning (ICML)}, 2016.

\bibitem[Mannor \& Tsitsiklis(2004)Mannor and
  Tsitsiklis]{mannor_2004_SampleComplexityExploration}
Mannor, S. and Tsitsiklis, J.~N.
\newblock The {{Sample Complexity}} of {{Exploration}} in the {{Multi}}-{{Armed
  Bandit Problem}}.
\newblock \emph{The Journal of Machine Learning Research}, 5:\penalty0
  623--648, 2004.

\bibitem[Mason et~al.(2020)Mason, Jain, Tripathy, and
  Nowak]{mason_2020_FindingAllEpsilon}
Mason, B., Jain, L., Tripathy, A., and Nowak, R.
\newblock Finding all {$\epsilon$}-good arms in stochastic bandits.
\newblock \emph{Advances in Neural Information Processing Systems}, 2020.

\bibitem[R{\'{e}}da et~al.(2021)R{\'{e}}da, Kaufmann, and
  Delahaye{-}Duriez]{reda_2021_TopmIdentificationLinear}
R{\'{e}}da, C., Kaufmann, E., and Delahaye{-}Duriez, A.
\newblock Top-m identification for linear bandits.
\newblock In \emph{International Conference on Artificial Intelligence and
  Statistics (AISTATS)}, 2021.

\bibitem[R{\'e}da et~al.(2021)R{\'e}da, Tirinzoni, and
  Degenne]{reda_2021_DealingWithMisspecification}
R{\'e}da, C., Tirinzoni, A., and Degenne, R.
\newblock Dealing with misspecification in fixed-confidence linear top-m
  identification.
\newblock \emph{Advances in Neural Information Processing Systems}, 2021.

\bibitem[Robbins(1952)]{robbins_1952_AspectsSequentialDesign}
Robbins, H.
\newblock Some aspects of the sequential design of experiments.
\newblock \emph{Bulletin of the American Mathematical Society}, 58:\penalty0
  527--535, 1952.

\bibitem[Sabato(2019)]{sabato_2019_EpsilonBestArmIdentificationPayPerRewarda}
Sabato, S.
\newblock Epsilon-{{Best}}-{{Arm Identification}} in {{Pay}}-{{Per}}-{{Reward
  Multi}}-{{Armed Bandits}}.
\newblock In \emph{Advances in {{Neural Information Processing Systems}}},
  2019.

\bibitem[Sentenac et~al.(2021)Sentenac, Yi, Calauzenes, Perchet, and
  Vojnovic]{sentenac_2021_PureExplorationRegret}
Sentenac, F., Yi, J., Calauzenes, C., Perchet, V., and Vojnovic, M.
\newblock Pure {{Exploration}} and {{Regret Minimization}} in {{Matching
  Bandits}}.
\newblock In \emph{International {{Conference}} on {{Machine Learning}}}, 2021.

\bibitem[Soare et~al.(2014)Soare, Lazaric, and
  Munos]{soare_2014_BestArmIdentificationLinear}
Soare, M., Lazaric, A., and Munos, R.
\newblock Best-arm identification in linear bandits.
\newblock In \emph{Advances in Neural Information Processing Systems}, 2014.

\bibitem[Tirinzoni et~al.(2020)Tirinzoni, Pirotta, Restelli, and
  Lazaric]{tirinzoni_2020_AsymptoticallyOptimalPrimal}
Tirinzoni, A., Pirotta, M., Restelli, M., and Lazaric, A.
\newblock An asymptotically optimal primal-dual incremental algorithm for
  contextual linear bandits.
\newblock \emph{Advances in Neural Information Processing Systems}, 2020.

\bibitem[Wang et~al.(2021)Wang, Tzeng, and
  Proutiere]{wang_2021_FastPureExploration}
Wang, P.-A., Tzeng, R.-C., and Proutiere, A.
\newblock Fast pure exploration via frank-wolfe.
\newblock \emph{Advances in Neural Information Processing Systems}, 2021.

\bibitem[Xu et~al.(2018)Xu, Honda, and
  Sugiyama]{xu_2017_FullyAdaptiveAlgorithm}
Xu, L., Honda, J., and Sugiyama, M.
\newblock A fully adaptive algorithm for pure exploration in linear bandits.
\newblock In \emph{International Conference on Artificial Intelligence and
  Statistics}, 2018.

\bibitem[Zaki et~al.(2020)Zaki, Mohan, and Gopalan]{zaki_2020_ExplicitBestArm}
Zaki, M., Mohan, A., and Gopalan, A.
\newblock Explicit {{Best Arm Identification}} in {{Linear Bandits Using
  No}}-{{Regret Learners}}.
\newblock \emph{arXiv:2006.07562 [cs, stat]}, June 2020.

\end{thebibliography}
\bibliographystyle{icml2022}

\newpage
\appendix
\onecolumn

\section{Outline} \label{app:section_outline}

The appendices are organized as follows:
\begin{itemize}
	\item Notations are summarized in Appendix~\ref{app:section_notation}.
	\item In Appendix~\ref{app:section_proofs_related_characteristic_time}, we prove all the results presented in Section~\ref{sec:section_comparing_optimal_answers} on the (greedy) characteristic time.
	\item The proof of Lemma~\ref{lem:delta_PAC_recommendation_stopping_pair} on the stopping-recommendation pairs is given in Appendix~\ref{app:section_proofs_stopping_recommendation_pair}.
	\item The full proof of Theorem~\ref{thm:asymptotic_optimality_algorithm} is given in Appendix~\ref{app:section_proof_thm_asymptotic_optimality_algorithm}.
	\item The implementation details and additional experiments are presented in Appendix~\ref{app:section_implementations_and_experiments}.
\end{itemize}

\begin{table}[H]
	\caption{Notation for the setting}
    \label{tab:notation_problem}
	\begin{center}
     \begin{tabular}{c c l}
	\toprule
     Notation & Type & Description \\
	\midrule
     $\cM \subseteq \Real^d$ &   & Set of possible parameters subset of $\Real^d$, $\max_{\lambda \in \cM} \|\lambda\|_2 \leq M$ \\
     $\mu$ & $\cM$  & Bandit mean parameter, $(\mu_i)_{i \in [d]}$  \\
     $\cK \subseteq \Real^d$ &   & Set of arms of cardinality $K$, an arm is denoted by $a$ and $L_{\cK} \eqdef \max_{a \in \cK} \|a\|_2$ \\
     $\mu^a$ &  $\Real$ & Mean reward of arm $a$, $\mu^a \eqdef \langle \mu , a \rangle$  \\
     $\cZ \subseteq \Real^d$ &   & Set of answers of cardinality $Z$, an arm is denoted by $z$ \\
     $\varepsilon$ &  $\Real^{\star}_{+}$ & Approximation error \\
     $\cZ_{\varepsilon}(\mu) \subseteq \cZ$ &   & Set of $\varepsilon$-optimal answers, with $\cdot^{\{\add,\mul\}}$\\
     $\neg_{\varepsilon} z \subseteq \cM$ &   & Set of alternative parameters to answer $z$, with $\cdot^{\{\add,\mul\}}$\\
     $T_{\varepsilon}(\mu)$ & $\Real^{\star}_{+}$  & Characteristic time for $\mu$, with $\cdot^{\{\add,\mul\}}$\\
     $z^{\star}(\mu) \subseteq \cZ$ &   & Greedy answers for $\mu$ \\
     $z_F(\mu)$ & $\cZ$  & Furthest answer for $\mu$, with $\cdot^{\{\add,\mul\}}$ \\
     $z_F(\mu, N) \subseteq \cZ$ &   & Instantaneous furthest answers for $\mu$ and allocation $N$, with $\cdot^{\{\add,\mul\}}$ \\
	\bottomrule
    \end{tabular}
\end{center}
\end{table}

\begin{table}[H]
	\caption{Notation for algorithms}
    \label{tab:notation_algorithms}
	\begin{center}
     \begin{tabular}{c c l}
	\toprule
     Notation & Type & Description \\
	\midrule
     $\delta$ &  $(0,1)$ & Confidence parameter \\
     $a_t$ &  $\cK$ & Arm sampled at time $t$  \\
     $X_{t}^{a_t}$ &  $\Real$ & Observation at time $t$, $X_{t}^{a_t} \sim \gaussdistr(\mu^{a_t}, \sigma^2)$ where $\sigma^2=1$ \\
     $\cF_t$ &   & History up to time $t$, $\sigma(a_1, X_{1}^{a_1}, \cdots, a_{t}, X_{t}^{a_t})$ \\
     $\mu_t$ & $\Real^d$  & Maximum likelihood estimator, $\mu_{t} = V_{N_{t}}^{-1} \sum_{s=1}^{t} X_{s}^{a_s} a_s$ \\
     $z_t$ & $\cZ$  & Candidate answer at time $t$  \\
     $\tau_{\delta}$ &  $\Natural$ & Stopping time for confidence $\delta$\\
     $N_t$ & $(\Real_{+})^{K}$  & Empirical count of sampled arms at time $t$\\
     $w_t$, $W_t$ &  $(\Real_{+})^{K}$ & Pulling distribution over arms and its cumulative sum at time $t$ \\
     $\beta(t, \delta)$ &  $\Natural \times (0,1) \to \Real^{\star}_{+}$ & Stopping threshold at time $t$ for confidence $\delta$ \\
     $f(t)$ &  $\Natural \to \Real^{\star}_{+}$  & Exploration bonus at time $t$ \\
     $\lambda_t$ &  $\cM$ & Most confusing alternative parameter player by nature at time $t$ \\
     $U_t$ & $(\Real_{+})^{K}$  & Optimistic gain at time $t$ \\
	\bottomrule
    \end{tabular}
\end{center}
\end{table}

\section{Notation} \label{app:section_notation}

We recall some commonly used notations: the set of integers $[n] \eqdef \{1, \cdots, n\}$, the interval of integers $\left\llbracket a, b \right\rrbracket$, the euclidean inner-product $\langle x, y \rangle \eqdef \sum_{i \in [d]} x_i y_i$, the design matrix $V_{w} \eqdef \sum_{a \in \cK} w^a a a\transpose$ for an allocation over arms $w$, the norm  $\|x\|_{V} \eqdef \sqrt{x\transpose V x}$ for a positive definite matrix $V$, the closure $\overline{X}$ of a set $X$, the cartesian product $\prod_{i \in [n]} X_i$ between sets $(X_i)_{i \in [n]}$, the set of probability distributions $\mathcal{P}(X)$ over $X$ and the $n$-dimensional probability simplex $\triangle_{n} \eqdef \left\{x \in \Real^{n}\mid x \geq 0 , \langle 1_{n}, x \rangle = 1 \right\}$. The Landau's notations $o$, $\cO$ and $\Theta$ are used.
In Table~\ref{tab:notation_problem}, we summarize problem-specific notation. Table~\ref{tab:notation_algorithms} gathers notation for the algorithms.

\section{Proofs Related to the Characteristic Time} \label{app:section_proofs_related_characteristic_time}

In Appendix~\ref{app:subsection_likelihood_ratio}, we recall the formula of the generalized log-likelihood ratio for Gaussian bandits. The asymptotic lower bound on the expected sample complexity (Theorem~\ref{thm:sample_complexity_lower_bound_epsBAI}) is stated in Appendix~\ref{proof:sample_complexity_lower_bound_epsBAI}. In Appendix~\ref{app:subsection_explicit_and_hard}, when $\overline{\cM} = \Real^d$, we derive explicit formulas on $T_{\varepsilon}(\mu)^{-1}$ (Lemma~\ref{lem:explicit_sample_complexity_addtive_multiplicative_optimality}) and exhibit hard BAI instances that can be solved as an $\varepsilon$-BAI problem (Lemma~\ref{lem:impossible_BAI_problems}). In Appendix~\ref{app:subsection_greedy_characteristic_time}, we prove the results related to the greedy characteristic time (Lemma~\ref{lem:greedy_sample_complexity_lower_bound_epsBAI}-\ref{lem:algorithms_being_greedy_asymptotically}).

\subsection{Likelihood Ratio for Gaussian Bandits} \label{app:subsection_likelihood_ratio}

For Gaussian bandits such that $\sigma^2 =1$, the Kullback-Leibler (KL) divergence between two bandits, with mean parameters $\mu$ and $\lambda$, on the arm $a \in \cK$ is:
\begin{equation*}
	d_{\text{KL}}(\mu^a, \lambda^a) = \frac{1}{2}(\mu^a - \lambda^a)^2 = \frac{1}{2}\left(a \transpose (\mu - \lambda)\right)^2
\end{equation*}

The generalized log-likelihood ratio (GLR) between the whole model space $\cM$ and a subset $\Lambda \subseteq \cM$ is
\begin{equation*}
	\text{GLR}^{\cM}_{t}(\Lambda) \eqdef \ln \frac{\sup_{\tilde{\mu} \in \cM} \cL_{\tilde{\mu}}(X_{1}^{a_1}, \cdots, X_{t}^{a_t})}{\sup_{\lambda \in \Lambda} \cL_{\lambda}(X_{1}^{a_1}, \cdots, X_{t}^{a_t})}
\end{equation*}
where $\cL_{\lambda}(X_{1}^{a_1}, \cdots, X_{t}^{a_t})$ denotes the likelihood of the observations $X_{1}^{a_1}, \cdots, X_{t}^{a_t}$ for a bandit with parameter $\lambda$.

Given two mean parameters $(\theta, \lambda) \in \cM^2$, we have
\begin{align*}
	\ln \frac{\cL_{\theta}(X_{1}^{a_1}, \cdots, X_{t}^{a_t})}{\cL_{\lambda}(X_{1}^{a_1}, \cdots, X_{t}^{a_t})} = \sum_{a \in \cK} N_{t}^a \left(d_{\text{KL}}(\mu_t^a, \lambda^a) - d_{\text{KL}}(\mu_t^a, \theta^a) \right)
\end{align*}
where $\mu_t \eqdef V_{N_{t-1}}^{-1} \sum_{s=1}^{t-1} X_{s}^{a_s} a_s$ is a sufficient statistic of our observations. Note that, when $\mu_t \in \cM$, $\mu_t$ coincide with the MLE $\tilde{\mu}_t$ defined as:
\begin{align*}
	\tilde{\mu}_t = \argmin_{\lambda \in \cM} \sum_{a \in \cK} N_{t}^a d_{\text{KL}}(\mu_t^a, \lambda^a) = \argmin_{\lambda \in \cM} \frac{1}{2} \| \mu_t - \lambda \|_{V_{N_t}}^{2}
\end{align*}

The GLR for the set $\Lambda$ is
\begin{align*}
	\text{GLR}^{\cM}_{t}(\Lambda) &= \min_{\lambda \in \Lambda} \sum_{a \in \cK} N_{t}^a d_{\text{KL}}(\mu_t^a, \lambda^a) - \sum_{a \in \cK} N_{t}^a d_{\text{KL}}(\mu_t^a, \tilde{\mu}_t^a) \\
	&= \min_{\lambda \in \Lambda} \frac{1}{2} \| \mu_t - \lambda \|_{V_{N_t}}^{2}  - \frac{1}{2} \| \mu_t - \tilde{\mu}_t \|_{V_{N_t}}^{2}
\end{align*}

\subsection{Proof of Theorem~\ref{thm:sample_complexity_lower_bound_epsBAI}} \label{proof:sample_complexity_lower_bound_epsBAI}

\begin{theorem*}[Theorem~\ref{thm:sample_complexity_lower_bound_epsBAI}, Corollary of Theorem 1 in \citep{degenne_2019_PureExplorationMultiple}]
	For all $(\varepsilon, \delta)$-PAC strategy, for all $\mu \in \cM$,
	\begin{equation*}
		\liminf_{\delta \rightarrow 0} \frac{\expectedvalue_{\mu}[\taud]}{\ln(1/\delta)} \geq T_{\varepsilon}(\mu)
	\end{equation*}
	where the inverse of the characteristic time is
	\begin{equation*}
		T_{\varepsilon}(\mu)^{-1} \eqdef \max_{z \in \cZ_{\varepsilon}(\mu)}  \max_{w \in \simplex} \inf_{\lambda \in \neg_{\varepsilon} z} \frac{1}{2} \| \mu - \lambda \|_{V_{w}}^2
	\end{equation*}
\end{theorem*}
\begin{proof}
Since $\varepsilon$-BAI for Gaussian distributions is a special case of identification in bandits with multiple correct answers and sub-Gaussian distributions, Theorem 1 in \citet{degenne_2019_PureExplorationMultiple} applies to our setting. Therefore, we obtain that: for all $(\varepsilon, \delta)$-PAC strategy, for all $\mu \in \cM$,
\begin{equation*}
	\liminf_{\delta \rightarrow 0} \frac{\expectedvalue_{\mu}[\taud]}{\ln(1/\delta)} \geq T^{\star}(\mu)
\end{equation*}
where
\begin{equation*}
	T^{\star}(\mu)^{-1} \eqdef  \max_{i \in i^{\star}(\mu)} \max_{w \in \simplex} \inf_{\lambda \in \neg i} \sum_{a \in \cK} w^a d_{\text{KL}}(\mu^a, \lambda^a)
\end{equation*}
with $i^{\star}(\mu)$ is the set of correct answers for $\mu$ and $\neg i \eqdef \{\mu \in \overline{\cM} : i \notin i^{\star}(\mu)\}$.

To conclude, we need to specify to our setting and some simple manipulations. First, under the Gaussian assumptions with $\sigma^2=1$, notice that $\mu^{a} = \langle a, \mu \rangle$ and $d_{\text{KL}}(\mu^a, \lambda^a) = \frac{1}{2} \langle a, \mu - \lambda \rangle^2 = \frac{1}{2}\|\mu - \lambda \|^2_{a a \transpose}$ yield that $\sum_{a \in \cK} w^a d_{\text{KL}}(\mu^a, \lambda^a) = \frac{1}{2} \| \mu - \lambda \|_{V_{w}}^2$.

When considering additive $\varepsilon$-optimality, the set of correct answer $i^{\star}(\mu)$ is $\cZ_{\varepsilon}^{\add}(\mu) = \left\{ z \in \cZ : \langle \mu,  z \rangle \geq \max_{z \in \cZ} \langle \mu, z\rangle - \varepsilon \right\}$, and $\neg i$ corresponds to $\neg_{\varepsilon}^{\add} z$.
For the multiplicative $\varepsilon$-optimality, the set of correct answer $i^{\star}(\mu)$ is $\cZ_{\varepsilon}^{\mul}(\mu) = \left\{ z \in \cZ:  \langle \mu,  z \rangle \geq (1-\varepsilon)\max_{z \in \cZ} \langle \mu, z\rangle \right\}$, and $\neg i$ corresponds to $\neg_{\varepsilon}^{\mul} z$.
Note that for the multiplicative setting, we assume $\max_{z \in \cZ} \langle \mu, z\rangle > 0$ and $\varepsilon \geq 0$ (or $\max_{z \in \cZ} \langle \mu, z\rangle < 0$ and $\varepsilon \leq 0$), while for the additive setting, we can consider $\varepsilon \geq 0$ and $\max_{z \in \cZ} \langle \mu, z\rangle \in \Real$. Using the correspondences between sets for each notion of $\varepsilon$-optimality, Theorem 1 in \citet{degenne_2019_PureExplorationMultiple} rewrites as: for all $(\varepsilon,\delta)$-PAC strategy, for all $\mu \in \cM$,
\begin{equation*}
	\liminf_{\delta \rightarrow 0} \frac{\expectedvalue_{\mu}[\taud]}{\ln(1/\delta)} \geq T_{\varepsilon}(\mu)
\end{equation*}
where
\begin{align*}
	T_{\varepsilon}^{\add}(\mu)^{-1} &= \max_{z \in \cZ_{\varepsilon}^{\add}(\mu)}  \max_{w \in \simplex} \inf_{\lambda \in \neg_{\varepsilon}^{\add} z} \frac{1}{2} \| \mu - \lambda \|_{V_{w}}^2 \\
	T_{\varepsilon}^{\mul}(\mu)^{-1} &= \max_{z \in \cZ_{\varepsilon}^{\mul}(\mu)}  \max_{w \in \simplex} \inf_{\lambda \in \neg_{\varepsilon}^{\mul} z} \frac{1}{2} \| \mu - \lambda \|_{V_{w}}^2
\end{align*}
\end{proof}

\subsection{Explicit Formulas and Hard BAI Instances} \label{app:subsection_explicit_and_hard}

In Appendix~\ref{proof:explicit_sample_complexity_addtive_multiplicative_optimality}, we derive explicit formulas on $T_{\varepsilon}(\mu)^{-1}$ (Lemma~\ref{lem:explicit_sample_complexity_addtive_multiplicative_optimality}), while Appendix~\ref{proof:impossible_BAI_problems} proved that there exists hard BAI instances that can be solved as an $\varepsilon$-BAI problem (Lemma~\ref{lem:impossible_BAI_problems}).

\subsubsection{Explicit Formulas} \label{proof:explicit_sample_complexity_addtive_multiplicative_optimality}

Lemma~\ref{lem:explicit_sample_complexity_addtive_multiplicative_optimality} shows more explicit formulas for $T_{\varepsilon}(\mu)$ when $\overline{\cM} = \Real^d$.

\begin{lemma} \label{lem:explicit_sample_complexity_addtive_multiplicative_optimality}
Assume $\overline{\cM} = \Real^d$,
\begin{align*}
	\frac{2}{T_{\varepsilon}^{\add}(\mu)} &=  \max_{z \in \cZ_{\varepsilon}^{\add}(\mu)} \max_{w \in \simplex}  \min_{x \in \cZ \setminus \{z\}}  \frac{\left(\varepsilon + \langle \mu, z-x \rangle\right)^2}{\|z-x\|^2_{V_{w}^{\dagger}}}  \\
	\frac{2}{T_{\varepsilon}^{\mul}(\mu)} &=  \max_{z \in \cZ_{\varepsilon}^{\mul}(\mu)} \max_{w \in \simplex}  \min_{x \in \cZ \setminus \{z\}} \frac{ \langle \mu, z - (1{-}\varepsilon)x \rangle^2}{\|z - (1{-}\varepsilon)x\|^2_{V_{w}^{\dagger}}}
\end{align*}
where $V_w^\dagger$ is the Moore-Penrose pseudo-inverse of $V_w$. The associated alternative parameters are
\begin{align*}
 \lambda_{\varepsilon}^{\add}(\mu, z, w, x) &= \mu - \frac{\varepsilon + \langle \mu, z-x \rangle}{\|z - x\|^2_{V_{w}^{\dagger}}} V_{w}^{\dagger} (z - x) \: , \\
 \lambda_{\varepsilon}^{\mul}(\mu, z, w, x) &= \mu - \frac{\langle \mu, z - (1-\varepsilon)x \rangle}{\|z - (1-\varepsilon)x\|^2_{V_{w}^{\dagger}}} V_{w}^{\dagger} (z - (1-\varepsilon)x)  \: .
\end{align*}
\end{lemma}

For the multiplicative $\varepsilon$-optimality, we recognize a term similar to the characteristic time $T_{0}(\mu, \cZ)$ of BAI for transductive bandits with answer set $\cZ$. Defining $\cZ_{\varepsilon}^{z} \eqdef \{z\} \cup \{(1- \varepsilon)x : x \in \cZ \setminus \{z\}\}$, we see that $T_{\varepsilon}^{\mul}(\mu) =  \min_{z \in \cZ_{\varepsilon}^{\mul}(\mu)} T_{0}\left(\mu, \cZ_{\varepsilon}^{z}\right)$. Therefore, the complexity of the multiplicative setting is equals to the easiest BAI for modified transductive bandits, in which $\cK$ is unchanged and $\cZ_{\varepsilon}^{z}$ is the set of answers. Since $T_{\varepsilon}^{\mul}(\mu) = T_{0}\left(\mu, \cZ_{\varepsilon}^{z^{\mul}_{F}(\mu)}\right)$, the unique correct answer that has to be identified in this modified BAI is the furthest answer $z^{\mul}_{F}(\mu)$.

The proof of Lemma~\ref{lem:explicit_sample_complexity_addtive_multiplicative_optimality} uses Lemma~\ref{lem:lemma_5_degenne_2020_GamificationPureExploration}, a known result in the literature which we prove for completeness, and the fact that $\neg_{\varepsilon} z$ is a union over a finite number of half-spaces.

\begin{lemma} \label{lem:lemma_5_degenne_2020_GamificationPureExploration}
For all $\mu, \lambda \in \Real^d$, $y \in \Real^d$, $x \in \Real$, we have
	\begin{align*}
		\inf_{\lambda \in \Real^d:  x - \langle \lambda, y \rangle \leq 0} \|\mu - \lambda\|^2_{V_{w}}	&= \begin{cases}
				0 &\text{if } x - \langle \mu, y \rangle \leq 0 \text{ or } y \notin \text{Im}(V_w)\\
				\left( \frac{x - \langle \mu, y \rangle}{\|y\|_{V_{w}^\dagger}} \right)^2 & \text{else}
				   \end{cases}
	\end{align*}
	where $V_w^\dagger$ is the Moore-Penrose pseudo-inverse of $V_w$. It is achieved at
	\begin{align*}
		\lambda &= \begin{cases}
						\mu &\text{if } x - \langle \mu, y \rangle \leq 0 \\
						\mu + \alpha u &\text{if } y \notin \text{Im}(V_w)\text{, s.t. } (u, \alpha) \in \text{Ker}(V_w)\times \Real, x - \langle\mu + \alpha u, y\rangle \le 0\\
						\mu + \frac{x - \langle \mu, y \rangle}{\|y\|^2_{V_{w}^{\dagger}}} V_{w}^{\dagger} y &\text{else}
				   \end{cases}
	\end{align*}
\end{lemma}
\begin{proof}
Since $V_w$ is self-adjoint, $\text{Ker}(V_w)^\bot = \text{Im}(V_w)$.
If $y \notin \text{Im}(V_w)$, then there exists $u \in \text{Ker}(V_w)$ with $\langle u, y\rangle \ne 0$. There exists also $\alpha \in \Real$ such that $x - \langle\mu + \alpha u, y\rangle \le 0$ and $\Vert \mu - (\mu + \alpha u) \Vert_{V_w} = 0$, hence the value of the objective is 0.

Otherwise, $y \in \text{Im}(V_w)$. We can restrict the problem to an infimum over $\lambda -\mu \in \text{Im}(V_w)$ since for any $\lambda$ satisfying the constraint, its projection on that space also satisfies the constraint and has lower objective value.

Let $V_w^\dagger$ be the Moore-Penrose pseudo-inverse of $V_w$. Restricted to $\text{Im}(V_w)$, this is a true inverse.

By using the Lagrangian of the problem, we obtain the following:
\begin{align*}
	\inf _{\lambda -\mu \in \text{Im}(V_w): x - \langle \lambda, y \rangle \leq 0} \|\mu-\lambda\|_{V_{w}}^{2}
		&=\sup _{\alpha \geq 0} \inf _{\lambda - \mu\in \text{Im}(V_w)} \|\mu-\lambda\|_{V_{w}}^{2}+\alpha(x-\langle\lambda, y\rangle) \\
		&=\sup _{\alpha \geq 0} \alpha(x-\langle\mu, y\rangle)-\alpha^{2} \frac{\|y\|_{V_{w}^\dagger}^{2}}{4} \\
		&= \begin{cases}0 & \text { if } x - \langle \mu, y \rangle \leq 0 \\
\left( \frac{x - \langle \mu, y \rangle}{\|y\|_{V_{w}^\dagger}} \right)^2 & \text { else }
\end{cases}
\end{align*}
where the infimum in the first equality is reached at $\lambda = \mu + \frac{1}{2}\alpha V_w^\dagger y$ and the supremum in the last equality is reached at $\alpha = 2\frac{x - \langle \mu, y \rangle}{\|y\|^2_{V_{w}^\dagger}} $ if $x - \langle \mu, y \rangle \geq 0$ and at $\alpha = 0$ else.
\end{proof}

We are now ready to prove Lemma~\ref{lem:explicit_sample_complexity_addtive_multiplicative_optimality}.

\begin{proof}
Let $w \in \simplex$. We will conduct the proof of both results in parallel. When $\overline{\cM} = \Real^d$, by definitions of $\neg_{\varepsilon}^{\add} z$ and $\neg_{\varepsilon}^{\mul} z$ which involve the closure of a set that can be written as a union over a finite number of half-spaces, we can rewrite:
\begin{align*}
	&\neg_{\varepsilon}^{\add} z = \bigcup_{x \neq z} \left\{ \lambda \in \Real^d: \langle \lambda,  z - x \rangle + \varepsilon \leq 0 \right\} \quad \quad \text{and} \quad \quad \inf_{\lambda \in \neg_{\varepsilon}^{\add} z} \frac{1}{2} \| \mu - \lambda \|_{V_{w}}^2 = \min_{x \neq z} \inf_{\lambda \in \Real^d: \langle \lambda,  z - x \rangle + \varepsilon \leq 0}  \| \mu - \lambda \|_{V_{w}}^2\\
	& \neg_{\varepsilon}^{\mul} z = \bigcup_{x \neq z} \left\{ \lambda \in \Real^d: \langle \lambda,  z - (1-\varepsilon)x \rangle  \leq 0 \right\} \quad \quad \text{and} \quad \quad \inf_{\lambda \in \neg_{\varepsilon}^{\mul} z} \frac{1}{2} \| \mu - \lambda \|_{V_{w}}^2 = \min_{x \neq z} \inf_{\lambda \in \Real^d: \langle \lambda,  z - (1-\varepsilon)x \rangle  \leq 0}  \| \mu - \lambda \|_{V_{w}}^2
\end{align*}

Since $z^{\star}(\mu) = \argmax_{z \in \cZ}\langle \mu,  z \rangle$, for all $z \in \cZ^{\add}_{\varepsilon}(\mu)$ (resp. $z \in \cZ^{\mul}_{\varepsilon}(\mu)$) and all $x \neq z$, we have by definition $\langle \mu,  z - x \rangle + \varepsilon \geq \langle \mu,  z - z^{\star}(\mu) \rangle + \varepsilon \geq 0$ (resp. $\langle \mu,  z - (1-\varepsilon)x \rangle \geq \langle \mu,  z - (1-\varepsilon)x \rangle  \geq 0 $). Using Lemma~\ref{lem:lemma_5_degenne_2020_GamificationPureExploration} with $\tilde{y} = x-z$ and $\tilde{x}=\varepsilon$ (resp. $\tilde{y} = (1-\varepsilon)x-z$ and $\tilde{x} = 0$), we obtain directly:
\begin{align*}
	2T_{\varepsilon}^{\add}(\mu)^{-1} &=  \max_{z \in \cZ_{\varepsilon}^{\add}(\mu)} \max_{w \in \simplex}  \min_{x \neq z} \left( \frac{\varepsilon + \langle \mu, z-x \rangle}{\|z-x\|_{V_{w}^{\dagger}}} \right)^2 \\
	2T_{\varepsilon}^{\mul}(\mu)^{-1} &=  \max_{z \in \cZ_{\varepsilon}^{\mul}(\mu)} \max_{w \in \simplex}  \min_{x \neq z} \left( \frac{ \langle \mu, z - (1-\varepsilon)x \rangle}{\|z - (1-\varepsilon)x\|_{V_{w}^{\dagger}}} \right)^2
\end{align*}
The associated alternative parameters are:
\begin{align*}
 \lambda_{\varepsilon}^{\add}(\mu, z, w, x) &= \mu - \frac{\varepsilon + \langle \mu, z-x \rangle}{\|z - x\|^2_{V_{w}^{\dagger}}} V_{w}^{\dagger} (z - x) \\
 \lambda_{\varepsilon}^{\mul}(\mu, z, w, x) &= \mu - \frac{\langle \mu, z - (1-\varepsilon)x \rangle}{\|z - (1-\varepsilon)x\|^2_{V_{w}^{\dagger}}} V_{w}^{\dagger} (z - (1-\varepsilon)x)
\end{align*}
\end{proof}

\paragraph{Link with BAI} First, note that $\neg_{\varepsilon} z \subseteq \neg_{0} z$ for all $z \in \cZ$, yields $T_{\varepsilon}(\mu) \leq T_{0}(\mu)$ for all $\mu \in \cM$. When considering the BAI for transductive linear bandits with Gaussian distribution, the characteristic time that lower bounds the expected sample complexity has the following form when $\overline{\cM} = \Real^d$ \citep{degenne_2020_GamificationPureExploration},
\begin{equation} \label{eq:definition_characteristic_time_BAI}
	2T_{0}(\mu)^{-1} =  \max_{w \in \simplex}  \min_{x \neq z^{\star}(\mu)} \left( \frac{\langle \mu, z^{\star}(\mu)-x \rangle}{\|z^{\star}(\mu)-x\|_{V_{w}^{\dagger}}} \right)^2
\end{equation}

This formula has striking similarities with the one obtained for multiplicative $\varepsilon$-optimality. Defining $\cZ_{\varepsilon}^{z} \eqdef \{z\} \cup \{(1 - \varepsilon)x\}_{x \neq z}$, we obtain that
\begin{align*}
	T_{\varepsilon}^{\mul}(\mu) =  \min_{z \in \cZ_{\varepsilon}^{\mul}(\mu)} T_{0}\left(\mu, \cZ_{\varepsilon}^{z}\right)
\end{align*}
where $T_{0}\left(\mu, \cZ_{\varepsilon}^{z}\right)$ is the characteristic time for a BAI problem where the set of answers is $\cZ_{\varepsilon}^{z}$, as defined in (\ref{eq:definition_characteristic_time_BAI}). Therefore, the complexity of the multiplicative setting is equal to the complexity of a BAI instance in which $\cK$ is unchanged, but the set of answers is changed into $\cZ_{\varepsilon}^{z}$.

\subsubsection{Hard BAI Instances} \label{proof:impossible_BAI_problems}

Lemma~\ref{lem:impossible_BAI_problems} shows that there exists arbitrarily hard BAI instances that can be solved if seen as an $\varepsilon$-BAI problem.

\begin{lemma} \label{lem:impossible_BAI_problems}
	Let $\varepsilon>0$, $\overline{\cM} = \Real^d$ and $T_{\varepsilon}(\mu, \cZ) = T_{\varepsilon}(\mu)$. For all $\mu \in \cM$, there exists a set of arms $\cK$ and a sequence of answers sets $(\cZ_{t})_{t\in \Natural}$, such that
	\begin{align*}
		\lim_{t \rightarrow + \infty} T_{0}(\mu, \cZ_t) = + \infty \quad \text{and} \quad \lim_{t \rightarrow + \infty} T_{\varepsilon}(\mu, \cZ_t) < + \infty
	\end{align*}
\end{lemma}

The proof of Lemma~\ref{lem:impossible_BAI_problems} is obtained directly by using Lemma~\ref{lem:impossible_BAI_problems_explicit_instances}, which provides explicit hard instances.

\begin{lemma} \label{lem:impossible_BAI_problems_explicit_instances}
	Let $\varepsilon>0$, $\overline{\cM} = \Real^d$ and $\mu \in \cM$. Let $\{z_{i}\}_{i \in [d-1]}$ such that $\{\frac{\mu}{\|\mu\|_2}\} \cup \{z_{i}\}_{i \in [d-1]}$ is an orthonormal basis of $\Real^d$. Define $\cK = \{\frac{\mu}{\|\mu\|_2}\} \cup \{z_{i}\}_{i \in [d-1]}$. Let $(\theta_t) \in [0,\frac{\pi}{2})^{\Natural}$, a decreasing sequence such that $\theta_t \rightarrow_{+ \infty} 0$. Define the sequence of answers sets $\cZ_t = \{\frac{\mu}{\|\mu\|_2}, \cos(\theta_t) \frac{\mu}{\|\mu\|_2} + \sin(\theta_t)  z_{d-1}\} \cup \{z_i\}_{i \in [d-1]}$ for all $t \in \Natural$ and
	\begin{align*}
		T_{\varepsilon}(\mu, \cZ)^{-1} \eqdef \max_{z \in \cZ_{\varepsilon}(\mu)}  \max_{w \in \simplex} \inf_{\lambda \in \neg_{\varepsilon} z} \frac{1}{2} \| \mu - \lambda \|_{V_{w}}^2 \: .
	\end{align*}
	Then, we have $T_{0}(\mu, \cZ_t) \rightarrow_{t \rightarrow + \infty} + \infty$ and
	\begin{align*}
	&T_{\varepsilon}^{\add}(\mu, \cZ_t) \rightarrow_{t \rightarrow + \infty} \frac{1}{(1+ \varepsilon)^2} \left(\frac{1}{w^{\star}(d-1)} + \frac{d-1}{1-w^{\star}(d-1)} \right) < + \infty \\
	&T_{\varepsilon}^{\mul}(\mu, \cZ_t) \rightarrow_{t \rightarrow + \infty}  \frac{1}{w^{\star}((1-\varepsilon)^2(d-1))} + (1-\varepsilon)^2\frac{d-1}{1-w^{\star}((1-\varepsilon)^2(d-1))}  < + \infty
	\end{align*}
	where
	\begin{align*}
	w^{\star}(a) = \begin{cases}
			\frac{\sqrt{1+(1-a)^2}-1}{a-1} & \text{if } a> 1 \\
			\frac{1}{2} & \text{if } a = 1 \\
			\frac{\sqrt{1+(1-a)^2}+ 1}{1-a} & \text{if } a \in (0,1)
		\end{cases}
	\end{align*}
\end{lemma}
\begin{proof}
Using the explicit formulas of Lemma~\ref{lem:lemma_5_degenne_2020_GamificationPureExploration}, we will perform the computations for both the multiplicative and the additive notions of $\varepsilon$-optimality. Define $\cK = \{\frac{\mu}{\|\mu\|_2}\} \cup \{z_{i}\}_{i \in [d-1]}$, where the $d$-th component is associated with $\frac{\mu}{\|\mu\|}$. Define $\cZ_t = \{\frac{\mu}{\|\mu\|_2}, \cos(\theta_t)\frac{\mu}{\|\mu\|_2} + \sin(\theta_t) z_{d-1}\} \cup \{z_i\}_{i \in [d-1]}$ for all $t \geq 1$, where $\{\frac{\mu}{\|\mu\|}\} \cup \{z_{i}\}_{i \in [d-1]}$ is an orthonormal basis of $\Real^d$. Note that $\|\cos(\theta_t)\frac{\mu}{\|\mu\|_2} + \sin(\theta_t) z_{d-1}\|_2 = 1$ and $\langle \mu, z \rangle = 0$ if $z \in \{z_{i}\}_{i \in [d-1]}$, else its value is $1$ and $\cos(\theta_t)$ for $\frac{\mu}{\|\mu\|_2}$ and $\cos(\theta_t)\frac{\mu}{\|\mu\|_2} + \sin(\theta_t) z_{d-1}$ respectively. Therefore, we have directly that $z^{\star}(\mu) = \frac{\mu}{\|\mu\|_2}$ and
\begin{align*}
	\forall t \geq t_0, \quad \cZ_{\varepsilon}(\mu) = \left\{\frac{\mu}{\|\mu\|_2}, \cos(\theta_t)\frac{\mu}{\|\mu\|_2} + \sin(\theta_t) z_{d-1}\right\}
\end{align*}
where $t_0 = \inf \{t \in \Natural^{\star} \mid \theta_t \leq \arccos(1-\varepsilon)\}$.

In the following, we consider $t \geq t_{0}$. Let $z \in \left\{\frac{\mu}{\|\mu\|_2}, \cos(\theta_t)\frac{\mu}{\|\mu\|_2} + \sin(\theta_t) z_{d-1}\right\} $ and $x \in \cZ_{t} \setminus \{z\}$, we have
\begin{align*}
	 \langle \mu, z-x \rangle &= \begin{cases}
			1- \cos(\theta_t) & \text{if } (z,x) = (\frac{\mu}{\|\mu\|_2}, \cos(\theta_t)\frac{\mu}{\|\mu\|_2} + \sin(\theta_t) z_{d-1})\\
			\cos(\theta_t)-1 & \text{if } (z,x) = ( \cos(\theta_t)\frac{\mu}{\|\mu\|_2} + \sin(\theta_t) z_{d-1}, \frac{\mu}{\|\mu\|_2}) \\
			1 & \text{if } x \in \{z_i\}_{i \in [d-1]}
		\end{cases} \\
	\langle \mu, z-(1-\varepsilon)x \rangle &= \begin{cases}
			1-(1-\varepsilon)\cos(\theta_t) & \text{if } (z,x) = (\frac{\mu}{\|\mu\|_2}, \cos(\theta_t)\frac{\mu}{\|\mu\|_2} + \sin(\theta_t) z_{d-1})\\
			\cos(\theta_t)-(1-\varepsilon) & \text{if } (z,x) = ( \cos(\theta_t)\frac{\mu}{\|\mu\|_2} + \sin(\theta_t) z_{d-1}, \frac{\mu}{\|\mu\|_2}) \\
			1 & \text{if } x \in \{z_i\}_{i \in [d-1]}
		\end{cases}
\end{align*}
Using that $\cK = \{\frac{\mu}{\|\mu\|_2}\} \cup \{z_{i}\}_{i \in [d-1]}$ is an orthogonal basis of $\Real^d$. Let $w \in \mathring{\simplex}$
\begin{align*}
	&\|z-x\|_{V_{w}^\dagger}^2 \\
	&=  \begin{cases}
			 (1- \cos(\theta_t))^2 \frac{1}{w_d} +  \sin(\theta_t)^2 \frac{1}{w_{d-1}} & \text{if } (z,x) \in \left\{ \frac{\mu}{\|\mu\|_2}, \cos(\theta_t)\frac{\mu}{\|\mu\|_2} + \sin(\theta_t) z_{d-1} \right\}^2 \\
			\frac{1}{w_{d}} + \frac{1}{w_i} & \text{if } z= \frac{\mu}{\|\mu\|_2} \text{ and } x \in \{z_i\}_{i \in [d-1]} \\
			\cos(\theta_t)^2 \frac{1}{w_{d}} + \sin(\theta_t)^2 \frac{1}{w_{d-1}} + \frac{1}{w_i} & \text{if } z= \cos(\theta_t)\frac{\mu}{\|\mu\|_2} + \sin(\theta_t) z_{d-1} \text{ and } x \in \{z_i\}_{i \in [d-2]} \\
			\cos(\theta_t)^2 \frac{1}{w_{d}} + (1-\sin(\theta_t))^2 \frac{1}{w_{d-1}} & \text{if } z= \cos(\theta_t)\frac{\mu}{\|\mu\|_2} + \sin(\theta_t) z_{d-1} \text{ and } x = z_{d-1}
		\end{cases} \\
	&\|z-(1-\varepsilon)x\|_{V_{w}^\dagger}^2 =\\
	&\begin{cases}
			 (1- (1-\varepsilon)\cos(\theta_t))^2 \frac{1}{w_d} +  (1-\varepsilon)^2 \sin(\theta_t)^2 \frac{1}{w_{d-1}} & \text{if } (z,x) = (\frac{\mu}{\|\mu\|_2}, \cos(\theta_t)\frac{\mu}{\|\mu\|_2} + \sin(\theta_t) z_{d-1})  \\
			(\cos(\theta_t)- (1-\varepsilon))^2 \frac{1}{w_d} +  \sin(\theta_t)^2 \frac{1}{w_{d-1}} & \text{if } (z,x) = ( \cos(\theta_t)\frac{\mu}{\|\mu\|_2} + \sin(\theta_t) z_{d-1}, \frac{\mu}{\|\mu\|_2}) \\
			\frac{1}{w_{d}} + \frac{(1-\varepsilon)^2}{w_i} & \text{if } z= \frac{\mu}{\|\mu\|_2} \text{ and } x \in \{z_i\}_{i \in [d-1]} \\
			\cos(\theta_t)^2 \frac{1}{w_{d}} + \sin(\theta_t)^2 \frac{1}{w_{d-1}} +\frac{(1-\varepsilon)^2}{w_i} & \text{if } z= \cos(\theta_t)\frac{\mu}{\|\mu\|_2} + \sin(\theta_t) z_{d-1}, \: x \in \{z_i\}_{i \in [d-2]} \\
			\cos(\theta_t)^2 \frac{1}{w_{d}} + (1-\varepsilon-\sin(\theta_t))^2 \frac{1}{w_{d-1}} & \text{if } z= \cos(\theta_t)\frac{\mu}{\|\mu\|_2} + \sin(\theta_t) z_{d-1}, \: x = z_{d-1}
		\end{cases}
\end{align*}

Since $\cos(\theta_t) \rightarrow_{t \rightarrow + \infty} 1$ and $\sin(\theta_t) \rightarrow_{t \rightarrow + \infty} 0$, we obtain
\begin{align*}
	\frac{\left(\varepsilon + \langle \mu, z-x \rangle\right)^2}{\|z-x\|^2_{V_{w}^{\dagger}}} \rightarrow_{t \rightarrow + \infty} \begin{cases}
			 + \infty & \text{if } (z,x) \in \left\{ \frac{\mu}{\|\mu\|_2}, \cos(\theta_t)\frac{\mu}{\|\mu\|_2} + \sin(\theta_t) z_{d-1} \right\}^2  \\
			\frac{(1+ \varepsilon)^2}{\frac{1}{w_{d}} + \frac{1}{w_i}} & \text{if } x \in \{z_i\}_{i \in [d-1]}
		\end{cases} \\
	\frac{\langle \mu, z-(1-\varepsilon)x \rangle^2}{\|z-(1-\varepsilon)x\|^2_{V_{w}^{\dagger}}} \rightarrow_{t \rightarrow + \infty} \begin{cases}
			 w_d & \text{if } (z,x) \in \left\{ \frac{\mu}{\|\mu\|_2}, \cos(\theta_t)\frac{\mu}{\|\mu\|_2} + \sin(\theta_t) z_{d-1} \right\}^2  \\
			\frac{1}{\frac{1}{w_{d}} + \frac{(1-\varepsilon)^2}{w_i}} & \text{if } x \in \{z_i\}_{i \in [d-1]}
		\end{cases}
\end{align*}

Since $|\cZ_t| = d+1$, the minimum over a $d$ values $g_{z}(w,\cZ_t) = \min_{x \in \cZ_{t}} g_{z}(w,x)$ is a continuous function of $(w,\cZ_t)$, where $g^{\add}_{z}(w,x) = \frac{\left(\varepsilon + \langle \mu, z-x \rangle\right)^2}{\|z-x\|^2_{V_{w}^{\dagger}}}$ and $g^{\mul}(w,x) =  \frac{\langle \mu, z-(1-\varepsilon)x \rangle^2}{\|z-(1-\varepsilon)x\|^2_{V_{w}^{\dagger}}}$.
Therefore, we have for all $z \in \left\{\frac{\mu}{\|\mu\|_2}, \cos(\theta_t)\frac{\mu}{\|\mu\|_2} + \sin(\theta_t) z_{d-1}\right\}$,
\begin{align*}
	\min_{x \in \cZ_{t} \setminus \{z\}} \frac{\left(\varepsilon + \langle \mu, z-x \rangle\right)^2}{\|z-x\|^2_{V_{w}^{\dagger}}} \rightarrow_{t \rightarrow + \infty} \min_{i \in [d-1]}\frac{(1+ \varepsilon)^2}{\frac{1}{w_{d}} + \frac{1}{w_i}} \\
	\min_{x \in \cZ_{t} \setminus \{z\}}  \frac{\langle \mu, z-(1-\varepsilon)x \rangle^2}{\|z-(1-\varepsilon)x\|^2_{V_{w}^{\dagger}}} \rightarrow_{t \rightarrow + \infty}  \min_{i \in [d-1]} \frac{1}{\frac{1}{w_{d}} + \frac{(1-\varepsilon)^2}{w_i}}
\end{align*}
where we used that $w_d > \frac{1}{\frac{1}{w_{d}} + \frac{(1-\varepsilon)^2}{w_i}}$. Since $g_{z}(w,\cZ_t)$ is a continuous function and $\simplex$ is compact, Berge's theorem yields that $g_{z}(\cZ_t) = \max_{w \in \simplex} g_{z}(w,\cZ_t)$ is continuous function of $\cZ_{t}$. Therefore, we have for all $z \in \left\{\frac{\mu}{\|\mu\|_2}, \cos(\theta_t)\frac{\mu}{\|\mu\|_2} + \sin(\theta_t) z_{d-1}\right\}$,
\begin{align*}
	\max_{w \in \simplex} \min_{x \in \cZ_{t} \setminus \{z\}} \frac{\left(\varepsilon + \langle \mu, z-x \rangle\right)^2}{\|z-x\|^2_{V_{w}^{\dagger}}} \rightarrow_{t \rightarrow + \infty} (1+ \varepsilon)^2 \max_{w \in \simplex} \min_{i \in [d-1]}\frac{1}{\frac{1}{w_{d}} + \frac{1}{w_i}} \\
	\max_{w \in \simplex} \min_{x \in \cZ_{t} \setminus \{z\}}  \frac{\langle \mu, z-(1-\varepsilon)x \rangle^2}{\|z-(1-\varepsilon)x\|^2_{V_{w}^{\dagger}}} \rightarrow_{t \rightarrow + \infty}  \max_{w \in \simplex} \min_{i \in [d-1]} \frac{1}{\frac{1}{w_{d}} + \frac{(1-\varepsilon)^2}{w_i}}
\end{align*}

Since $|\cZ_{\varepsilon}(\mu)|= |\left\{\frac{\mu}{\|\mu\|_2}, \cos(\theta_t)\frac{\mu}{\|\mu\|_2} + \sin(\theta_t) z_{d-1}\right\}| =2$, the minimum over $2$ values is continuous. Since the inverse is also a continuous function in $\Real^{\star}$ and $\cZ_{\varepsilon}(\mu) \rightarrow_{t \rightarrow + \infty} \{\frac{\mu}{\|\mu\|_2}\}$, hence
\begin{align*}
	T_{\varepsilon}^{\add}(\mu, \cZ_t) \rightarrow_{t \rightarrow + \infty} \frac{1}{(1+ \varepsilon)^2} \min_{w \in \simplex} \max_{i \in [d-1]} \frac{1}{w_{d}} + \frac{1}{w_i}  \\
	T_{\varepsilon}^{\mul}(\mu, \cZ_t) \rightarrow_{t \rightarrow + \infty}  \min_{w \in \simplex} \max_{i \in [d-1]}  \frac{1}{w_{d}} + \frac{(1-\varepsilon)^2}{w_i}
\end{align*}

Rewriting those optimization problems, we obtain
\begin{align*}
	&\min_{w \in \simplex} \max_{i \in [d-1]} \frac{1}{w_{d}} + \frac{1}{w_i} = \min_{w \in (0,1)} \left( \frac{1}{w} + \frac{1}{1-w} \left(\max_{\tilde{w} \in \triangle_{d-1}}\min_{i \in [d-1]} \tilde{w}_i\right)^{-1} \right) = \min_{w \in (0,1)} \left( \frac{1}{w} + \frac{d-1}{1-w} \right) \\
	&\min_{w \in \simplex} \max_{i \in [d-1]}  \frac{1}{w_{d}} + \frac{(1-\varepsilon)^2}{w_i} = \min_{w \in (0,1)} \left( \frac{1}{w} + \frac{(1-\varepsilon)^2}{1-w} \left(\max_{\tilde{w} \in \triangle_{d-1}}\min_{i \in [d-1]} \tilde{w}_i\right)^{-1} \right) = \min_{w \in (0,1)} \left( \frac{1}{w} + \frac{(1-\varepsilon)^2(d-1)}{1-w}\right)
\end{align*}
where the first equality uses a rewriting of the simplex. The second inequalities are obtained since $\max_{\tilde{w} \in \triangle_{d-1}}\min_{i \in [d-1]} \tilde{w}_i = \frac{1}{d-1}$, achieved for $\tilde{w} = \frac{1}{d-1} \1_{d-1}$ and would lead to smaller values if there exists $i_{0} \in [d-1]$ such that $\tilde{w}_{i_0} <\frac{1}{d-1}$, namely $\tilde{w}_{i_0}$.

Considering the function $f_{a}(w) = \frac{1}{w} + \frac{a}{1-w}$, we have $f_{a}'(w) = \frac{(a-1)w^2 + 2w - 1}{w^2(1-w)^2}$. Let $w^{\star}(a) = \argmin_{w \in (0,1)} f_{a}(w)$. Therefore, by solving the second order equation with discriminant $\Delta = 4 + 4(1-a)^2$ and keeping the solution belonging to $(0,1)$, we obtain
\begin{align*}
	w^{\star}(a) = \begin{cases}
			\frac{\sqrt{1+(1-a)^2}-1}{a-1} & \text{if } a> 1 \\
			\frac{1}{2} & \text{if } a = 1 \\
			\frac{\sqrt{1+(1-a)^2}+ 1}{1-a} & \text{if } a \in (0,1)
		\end{cases}
\end{align*}

Therefore, we have shown
\begin{align*}
	&T_{\varepsilon}^{\add}(\mu, \cZ_t) \rightarrow_{t \rightarrow + \infty} \frac{1}{(1+ \varepsilon)^2} \left(\frac{1}{w^{\star}(d-1)} + \frac{d-1}{1-w^{\star}(d-1)} \right)  \\
	&T_{\varepsilon}^{\mul}(\mu, \cZ_t) \rightarrow_{t \rightarrow + \infty}  \frac{1}{w^{\star}((1-\varepsilon)^2(d-1))} + (1-\varepsilon)^2\frac{d-1}{1-w^{\star}((1-\varepsilon)^2(d-1))}
\end{align*}

Now, lets show that the corresponding BAI problem diverges. Using the same arguments of continuity as before and taking the limit when $t \rightarrow + \infty$, we have
\begin{align*}
\frac{\langle \mu, z-x \rangle^2}{\|z-x\|^2_{V_{w}^{\dagger}}} &\rightarrow_{t \rightarrow + \infty} \frac{1}{\frac{1}{w_{d}} + \frac{1}{w_i}} && \text{if } x \in \{z_i\}_{i \in [d-1]} \\
		&\sim_{+\infty} \frac{1}{\frac{1}{w_{d}} + \frac{\sin(\theta_t)^2}{(1-\cos(\theta_t))^2}\frac{1}{w_i}} \rightarrow_{t \rightarrow + \infty} 0  && \text{if } (z,x) \in \left\{ \frac{\mu}{\|\mu\|_2}, \cos(\theta_t)\frac{\mu}{\|\mu\|_2} + \sin(\theta_t) z_{d-1} \right\}^2
\end{align*}
where the last part is obtained since $\frac{\sin(\theta_t)^2}{(1-\cos(\theta_t))^2} \rightarrow + \infty$. Therefore, using the same continuity arguments, we have
\begin{align*}
	\min_{x \in \cZ_{t} \setminus \{\frac{\mu}{\|\mu\|}\}}  \frac{\langle \mu, z-x \rangle^2}{\|z-x\|^2_{V_{w}^{\dagger}}} \rightarrow_{t \rightarrow + \infty} 0 \quad \text{ and } \quad  T_{0}(\mu, \cZ_t)^{-1} \rightarrow_{t \rightarrow + \infty} 0 \: .
\end{align*}

This yields that $T_{0}(\mu, \cZ_t) \rightarrow_{t \rightarrow + \infty} + \infty$.
\end{proof}

\subsection{Greedy Characteristic Time} \label{app:subsection_greedy_characteristic_time}

In Appendix~\ref{proof:greedy_sample_complexity_lower_bound_epsBAI}, we prove the lower bound for asymptotically greedy algorithm (Lemma~\ref{lem:greedy_sample_complexity_lower_bound_epsBAI}). In Appendix~\ref{proof:algorithms_being_greedy_asymptotically}, we show that using $z_t \in z^{\star}(\mu_{t})$ yields an asymptotically greedy strategy (Lemma~\ref{lem:algorithms_being_greedy_asymptotically}) .

\subsubsection{Proof of Lemma~\ref{lem:greedy_sample_complexity_lower_bound_epsBAI}} \label{proof:greedy_sample_complexity_lower_bound_epsBAI}

Lemma 3 in \citet{garivier_2019_NonAsymptoticSequentialTests} is a useful change of measures result in the low-level form (involving probabilities and not expectation). This Lemma is key to derive a lower bound on the sample complexity of $\varepsilon$-BAI, as done in \citet{garivier_2019_NonAsymptoticSequentialTests} and \citet{degenne_2019_PureExplorationMultiple}.

\begin{lemma}[Lemma 3 in \citet{garivier_2019_NonAsymptoticSequentialTests}] \label{lem:lemma_3_garivier_2019_NonAsymptoticSequentialTests}
Consider two distributions $\mathbb{P}$ and $\mathbb{Q}$. Let us denote the log-likelihood ratio after $t$ rounds by $\cL_{t}=\ln \frac{\mathrm{d} \mathbb{P}}{\mathrm{d} \mathbb{Q}}$. Then for any measurable event $A \in \mathcal{F}_{t}$ and threshold $\gamma \in \Real$,
$$
\mathbb{P}[A] \leq e^{\gamma} \mathbb{Q}[A]+\mathbb{P}\left\{\cL_{t}>\gamma\right\}
$$
\end{lemma}

We rewrite Lemma 2 in \citet{degenne_2019_PureExplorationMultiple} in the setting of $(\varepsilon,\delta)$-PAC BAI for transductive linear bandits with Gaussian distribution.

\begin{lemma}[Lemma 2 in \citet{degenne_2019_PureExplorationMultiple}] \label{lem:lemma_2_degenne_2019_PureExplorationMultiple}
For any answer $z \in \cZ$, the divergence from $\mu$ to $\neg_{\varepsilon} z$ equals
\begin{equation} \label{eq:def_characteristic_time_answer_z}
	T_{\varepsilon}(\mu, z)^{-1} = \sup_{w \in \simplex} \inf_{\lambda \in \neg_{\varepsilon} z} \|\mu-\lambda\|^2_{V_{w}} = \inf_{\probability} \max _{a \in \cK} \mathbb{E}_{\lambda \sim \probability}\left[\|\mu - \lambda\|^2_{a a\transpose}\right]
\end{equation}
where the infimum ranges over probability distributions on $\neg_{\varepsilon} z$ supported on (at most) $K$ points.
\end{lemma}

We rewrite Lemma 19 in \citet{degenne_2019_PureExplorationMultiple} in the setting of $(\varepsilon,\delta)$-PAC BAI for transductive linear bandits with Gaussian distribution. In words, if $A$ is likely under $\mu$, it must also be likely under at least one $\lambda^{k}$ for sample sizes $t \ll T_{\varepsilon}(\mu,z)$. It was proven using Lemma~\ref{lem:lemma_3_garivier_2019_NonAsymptoticSequentialTests}.

\begin{lemma}[Lemma 19 in \citet{degenne_2019_PureExplorationMultiple}] \label{lem:lemma_19_degenne_2019_PureExplorationMultiple}
Let $z \in \cZ$. Let $w$ and $\lambda_{1}, \cdots, \lambda_{K}$ be a minimax witness from Lemma~\ref{lem:lemma_2_degenne_2019_PureExplorationMultiple}, and let us introduce the abbreviation $\alpha_{a}= \| \mu -\sum_{k \in [K]} w_{k} \lambda_{k}\|_{a a\transpose}^2$ for all $a \in \cK$. Fix a sample size $t$, and consider any event $A \in \mathcal{F}_{t}$. Then, for any $\beta>0$
$$
\max _{k \in[K]} \probability_{\lambda_{k}}\{A\} \geq e^{-t T_{\varepsilon}(\mu,z)^{-1}-\beta}\left(\probability_{\mu}\{A\}-\exp \left(\frac{-\beta^{2}}{2 t \max_{a \in \cK} \alpha_{a}}\right)\right) .
$$
where $T_{\varepsilon}(\mu, z)^{-1} = \sup_{w \in \simplex} \inf_{\lambda \in \neg_{\varepsilon} z} \|\mu-\lambda\|^2_{V_{w}}$.
\end{lemma}

We are now ready to prove Lemma~\ref{lem:greedy_sample_complexity_lower_bound_epsBAI}, which we recall below.

\begin{lemma*}[Lemma~\ref{lem:greedy_sample_complexity_lower_bound_epsBAI}]
	For all asymptotically greedy $(\varepsilon,\delta)$-PAC strategy, for all $\mu \in \cM$,
	\begin{equation*}
		\liminf_{\delta \rightarrow 0} \frac{\expectedvalue_{\mu}[\taud]}{\ln(1/\delta)} \geq T_{g,\varepsilon}(\mu)
	\end{equation*}
	where the inverse of the greedy characteristic time is
	\begin{equation*}
		T_{g,\varepsilon}(\mu)^{-1} \eqdef \max_{z \in z^{\star}(\mu)} \max_{w \in \simplex} \inf_{\lambda \in \neg_{\varepsilon} z} \frac{1}{2} \| \mu - \lambda \|_{V_{w}}^2 \; ,
	\end{equation*}
	and $T_{g,\varepsilon}(\mu) > T_{\varepsilon}(\mu) $ if and only if $z_{F}(\mu) \notin z^{\star}(\mu)$.
\end{lemma*}
\begin{proof}
We will bound the expectation of the stopping time $\tau_{\delta}$ through Markov's inequality. For $T>0$,
$$
\mathbb{E}_{\mu}\left[\taud\right] \geq T\left(1-\mathbb{P}_{\mu}\left[\taud \leq T\right]\right) \: .
$$
The event $\left\{\tau_{\delta} \leq T\right\}$ can be partitioned depending on the answer whether the answer is $\varepsilon$-optimal or not, and then whether it's $z^{\star}(\mu)$ or not. By hypothesis, $\mathbb{P}_{\mu}\left[\tau_{\delta} \leq T, \hat{z} \notin \cZ_{\varepsilon}(\mu)\right] \leq \mathbb{P}_{\mu}\left[\tau_{\delta} < +\infty, \hat{z} \notin \cZ_{\varepsilon}(\mu)\right] \leq \delta$ and $0\leq \lim_{\delta \rightarrow 0 }\mathbb{P}_{\mu}\left[\tau_{\delta} \leq T, \hat{z} \in \cZ_{\varepsilon}(\mu) \setminus z^{\star}(\mu)\right] \leq \lim_{\delta \rightarrow 0 } \mathbb{P}_{\mu}\left[\tau_{\delta} <  +\infty, \hat{z} \in \cZ_{\varepsilon}(\mu) \setminus z^{\star}(\mu) \right] = 0$. This yields
\begin{align*}
	\probability_{\mu}\left[\tau_{\delta} \leq T\right] &= \probability_{\mu}\left[\tau_{\delta} \leq T, \hat{z} \notin \cZ_{\varepsilon}(\mu)\right] + \probability_{\mu}\left[\tau_{\delta} \leq T, \hat{z} \in \cZ_{\varepsilon}(\mu) \setminus z^{\star}(\mu) \right]  + \sum_{z \in z^{\star}(\mu)} \probability_{\mu}\left[\tau_{\delta} \leq T, \hat{z}  = z \right] \: , \\
	\lim_{\delta \rightarrow 0 } \probability_{\mu}\left[\tau_{\delta} \leq T\right] &\leq \sum_{z \in z^{\star}(\mu)} \lim_{\delta \rightarrow 0}  \probability_{\mu}\left[\tau_{\delta} \leq T, \hat{z}   = z \right]
\end{align*}

Let $z \in z^{\star}(\mu)$, $w$ and $\lambda_{1}, \cdots, \lambda_{K}$ be a minimax witness from Lemma~\ref{lem:lemma_2_degenne_2019_PureExplorationMultiple}. Then by Lemma~\ref{lem:lemma_19_degenne_2019_PureExplorationMultiple}, for any $\beta>0$
\begin{align*}
	\probability_{\mu}\left[\tau_{\delta} \leq T, \hat{z} = z \right] & \leq \exp \left(\frac{T}{T_{\varepsilon}(\mu, z)}+\beta\right) \max_{k \in [K]} \probability_{\lambda_{k}}\left[\tau_{\delta} \leq T, \hat{z} = z\right]+\exp \left(\frac{-\beta^{2}}{2 T \max_{a \in \cK} \alpha_{a}}\right) \\
& \leq \delta \exp \left(\frac{T}{T_{\varepsilon}(\mu, z)}+\beta\right)+\exp \left(\frac{-\beta^{2}}{2 T \max_{a \in \cK} \alpha_{a}}\right)
\end{align*}
where the second inequality uses that $\lambda_{k} \in \neg_{\varepsilon} z$ for all $k \in [K]$, hence $z \in z^{\star}(\mu)  \subseteq \cZ \setminus \cZ_{\varepsilon}(\lambda_{k})$ and that the strategy satisfies $\mathbb{P}_{\lambda}\left[\tau_{\delta} < +\infty, \hat{z} \notin \cZ_{\varepsilon}(\lambda)\right] \leq \delta$ for all $\lambda \in \cM$.

Let $\alpha=\max_{a \in \cK} \alpha_{a}$. For $\eta \in(0,1)$, $T  = (1-\eta) \min_{z \in z^{\star}(\mu)} T_{\varepsilon}(\mu,z) \ln(1 / \delta)$, $\beta=\frac{\eta}{2 \sqrt{1-\eta}} \sqrt{\frac{T}{\min_{z \in z^{\star}(\mu)} T_{\varepsilon}(\mu,z)} \log (1 / \delta)}$, and all $z \in z^{\star}(\mu)$,
\begin{align*}
	\probability_{\mu}\left[\tau_{\delta} \leq T, \hat{z} = z\right] & \leq \delta \exp \left(\frac{T}{T_{\varepsilon}(\mu,z)}+\frac{\eta}{2 \sqrt{1-\eta}} \sqrt{\frac{T\log (1 / \delta)}{\min_{z \in z^{\star}(\mu)} T_{\varepsilon}(\mu,z)} }\right)+\exp \left(\frac{-\eta^{2} \log (1 / \delta)}{8(1-\eta) \min_{z \in z^{\star}(\mu)} T_{\varepsilon}(\mu,z) \alpha}\right)\\
& \leq \delta \exp \left((1-\eta / 2) \log \frac{1}{\delta}\right)+\exp \left(\frac{-\eta^{2} \log (1 / \delta)}{8(1-\eta) \min_{z \in z^{\star}(\mu)} T_{\varepsilon}(\mu,z) \alpha}\right) \\
&=\delta^{\eta / 2}+\delta^{\eta^{2} /\left(8(1-\eta) \min_{z \in z^{\star}(\mu)} T_{\varepsilon}(\mu,z) \alpha \right)} \rightarrow_{\delta \rightarrow 0} 0  \: ,
\end{align*}
where we used that $\min_{z \in z^{\star}(\mu)} T_{\varepsilon}(\mu,z) \le T_{\varepsilon}(\mu,z)$.

Since we have just shown $\lim_{\delta \rightarrow 0} \probability_{\mu}\left[\tau_{\delta} \leq T\right]$ for $T  = (1-\eta)\min_{z \in z^{\star}(\mu)} T_{\varepsilon}(\mu,z)\ln(1 / \delta)$, we obtain
\begin{align*}
	\lim_{\delta \rightarrow 0} \frac{\mathbb{E}_{\mu}\left[\taud\right]}{\ln (1 / \delta)} \geq \lim_{\delta \rightarrow 0} \frac{T}{\ln(1/\delta)}\left(1-\mathbb{P}_{\mu}\left[\taud \leq T\right]\right) &\geq  (1-\eta)\min_{z \in z^{\star}(\mu)} T_{\varepsilon}(\mu,z) \left(1-\lim_{\delta \rightarrow 0}\mathbb{P}_{\mu}\left[\taud \leq T\right]\right) \\
	&= (1-\eta)\min_{z \in z^{\star}(\mu)} T_{\varepsilon}(\mu,z) \: .
\end{align*}
Letting $\eta$ go to zero, we obtain that
$$
\liminf _{\delta \rightarrow 0} \frac{\mathbb{E}_{\mu}\left[\tau_{\delta}\right]}{\log (1 / \delta)} \geq \min_{z \in z^{\star}(\mu)} T_{\varepsilon}(\mu,z) = T_{g,\varepsilon}(\mu) \: .
$$

The fact that $T_{g,\varepsilon}(\mu) > T_{\varepsilon}(\mu)$ if and only if $z_{F}(\mu) \notin z^{\star}(\mu) $ is a direct consequence of the definition of $T_{\varepsilon}(\mu)$, $T_{g,\varepsilon}(\mu)$ and $z_{F}(\mu)$.
\end{proof}

\subsubsection{Proof of Lemma~\ref{lem:algorithms_being_greedy_asymptotically}} \label{proof:algorithms_being_greedy_asymptotically}

\begin{lemma*}[Lemma~\ref{lem:algorithms_being_greedy_asymptotically}]
Any $(\varepsilon, \delta)$-PAC strategy recommending $\hat{z} \in z^{\star}(\mu_{\taud})$ is asymptotically greedy if the sampling rule ensures that $\lim_{\delta \rightarrow 0}\probability_{\mu} [\taud < + \infty,  z^{\star}(\mu_{\taud}) =  z^{\star}(\mu) ] = 1$.
\end{lemma*}
\begin{proof}
By definition of $(\varepsilon, \delta)$-PAC, we have $\probability_{\mu} \left[ \taud < + \infty, \hat{z} \notin \cZ_{\varepsilon}(\mu)\right] \leq \delta$.
Since $\hat{z} \in z^{\star}(\mu_{\taud})$, we have $\{\taud < + \infty,  z^{\star}(\mu_{\taud}) =  z^{\star}(\mu)\} \subseteq \{\taud < + \infty,  \hat{z}  \in z^{\star}(\mu)\}$.
Therefore, the assumption yields $\lim_{\delta \rightarrow 0}\probability_{\mu} [\taud < + \infty,  \hat{z}  \in z^{\star}(\mu)] = 1$.
Partitioning the event $\{\taud < + \infty, \hat{z} \in \cZ\}$ (which obviously holds)
\begin{align*}
	\delta \ge \probability_{\mu} \left[ \taud < + \infty,  \hat{z} \notin \cZ_{\varepsilon}(\mu)\right] &= 1 - \probability_{\mu} \left[ \taud < + \infty,  \hat{z}  \in z^{\star}(\mu)\right] - \probability_{\mu} \left[ \taud < + \infty,  \hat{z} \in \cZ_{\varepsilon}(\mu) \setminus z^{\star}(\mu) \right] \ge 0 \: .
\end{align*}
Taking the limit $\delta \rightarrow 0$ yields
\begin{align*}
	\lim_{\delta \rightarrow 0} \probability_{\mu} \left[ \taud < + \infty,  \hat{z} \in \cZ_{\varepsilon}(\mu) \setminus z^{\star}(\mu) \right] = 0 \: ,
\end{align*}
i.e. the strategy is asymptotically greedy.
\end{proof}

\section{Proof of Lemma~\ref{lem:delta_PAC_recommendation_stopping_pair}} \label{app:section_proofs_stopping_recommendation_pair}

The proof leverages the concentration inequalities in the Corollary 10 of \citet{kaufmann_2018_MixtureMartingalesRevisited}, which we restate below.
\begin{lemma} \label{lem:corollary_10_kaufmann_2018_MixtureMartingalesRevisited}
	Let $\nu$ a Gaussian bandit with mean $\mu$. Let $S \subseteq \cK$ and $ x > 0$.
	\begin{align*}
		\probability_{\nu} \left[ \exists t \in \Natural: \sum_{a \in S} N_{t}^{a} d_{\text{KL}}(\mu_t^{a} , \mu^a) >  \sum_{a \in S} 2 \ln \left(4 + \ln \left( N_t^a\right)\right) + |S| \cC_{G}\left(\frac{x}{|S|}\right) \right] \leq e^{-x}
	\end{align*}
	where $\cC_{G}$ is defined in \citet{kaufmann_2018_MixtureMartingalesRevisited} by $\cC_{G}(x) = \max_{\lambda \in ]1/2,1]} \frac{g_G(\lambda) + x}{\lambda}$ and
	\begin{equation} \label{eq:def_C_gaussian_kaufmann_2018_MixtureMartingalesRevisited}
		g_G (\lambda) = 2\lambda - 2\lambda \ln (4 \lambda) + \ln \zeta(2\lambda) - \frac{1}{2} \ln(1-\lambda) \: ,
	\end{equation}
	where $\zeta$ is the Riemann $\zeta$ function and $\cC_{G}(x) \approx x + \ln(x)$.
\end{lemma}

Let $\cT \subseteq \Natural$ be a set of times with $|\cT| = \infty$. For both notions of $\varepsilon$-optimality (additive or multiplicative), the crucial element is the stopping criterion, which performs a GLRT at each time $t \in \cT$ for an arbitrary candidate $\varepsilon$-optimal answer $z_t \in \cZ_{\varepsilon}(\mu_t)$,
\begin{equation*}
	\inf_{\lambda \in \neg_{\varepsilon} z_t}  \| \mu_{t-1} - \lambda \|_{V_{N_{t-1}}}^{2} > 2\beta(t-1, \delta)
\end{equation*}

Lemma~\ref{lem:delta_PAC_recommendation_stopping_pair} holds for both notions of $\varepsilon$-optimality (additive or multiplicative) and is agnostic to the sampling rule. Moreover, it holds when the stopping criterion is evaluated only on some predefined indices of time $\cT \subseteq \Natural$ with $|\cT| = \infty$, for example on a geometric grid, and for any recommendation rule verifying $z_t \in \cZ_{\varepsilon}(\mu_t)$ for all $t \in \cT$.

\begin{lemma*}[Lemma~\ref{lem:delta_PAC_recommendation_stopping_pair}]
	Let $\cT \subseteq \Natural$ be a set of times with $|\cT| = \infty$. Regardless of the sampling rule and for any recommendation rule such that $z_t \in \cZ_{\varepsilon}(\mu_{t-1})$ for all $t \in \cT$, then evaluating the stopping criterion (\ref{eq:definition_stopping_criterion}) at each time $t \in \cT$ with the threshold
	\begin{equation*}
		\beta(t, \delta) = 2 K\ln \left(4 + \ln \left( \frac{t}{K} \right)\right) + K \cC_{G}\left(\frac{\ln \left( \frac{1}{\delta}\right)}{K}\right)
	\end{equation*}
	yields an $(\varepsilon, \delta)$-PAC strategy. $\cC_{G}$ is defined in \eqref{eq:def_C_gaussian_kaufmann_2018_MixtureMartingalesRevisited} and $\cC_{G}(x) \approx x + \ln(x)$.
\end{lemma*}
\begin{proof}
For simplicity, let's first prove the result when $\cT = \Natural$. Considering a notion of $\varepsilon$-optimality (additive or multiplicative), an arbitrary sampling rule and a recommendation rule such that for all $t$, $z_t \in \cZ_{\varepsilon}(\mu_{t-1})$. By evaluating the stopping criterion (\ref{eq:definition_stopping_criterion}) at each time $t \in \Natural$, we obtain the following inequalities:
\begin{align*}
	\probability_{\mu} \left[ \taud < + \infty , \hat{z} \notin \cZ_{\varepsilon}(\mu)\right] &\leq \probability_{\mu} \left[ \exists t \in \Natural, z_{t} \notin \cZ_{\varepsilon}(\mu) , \inf_{\lambda \in \neg_{\varepsilon} z_{t}} \|\mu_{t-1} - \lambda \|^2_{V_{N_{t-1}}} > 2 \beta(t-1, \delta) \right] \\
	&\leq \probability_{\mu} \left[ \exists t \in \Natural, z_{t} \notin \cZ_{\varepsilon}(\mu) , \|\mu_{t-1} - \mu \|^2_{V_{N_{t-1}}} > 2 \beta(t-1, \delta) \right] \\
	&\leq \probability_{\mu} \left[ \exists t \in \Natural, \|\mu_{t-1} - \mu \|^2_{V_{N_{t-1}}} > 2 \beta(t-1, \delta) \right] \\
	&= \probability_{\mu} \left[ \exists t \in \Natural, \sum_{a \in \cK} N_{t-1}^{a} d_{\text{KL}}(\mu_{t-1}^{a} , \mu^a) >  \beta(t-1, \delta) \right]
\end{align*}
where the first inequality is obtained by definition of the stopping criterion and using that $\probability \left[\exists t > n_{0}, A_t \right] \leq \probability \left[ \exists t \in \Natural, A_t  \right] $. Crucially, since $\beta(t-1, \delta) > 0$, it implicitly requires that $z_t \in \cZ_{\varepsilon}(\mu_t)$, otherwise $\inf_{\lambda \in \neg_{\varepsilon} z_{t}} \|\mu_{t-1} - \lambda \|^2_{V_{N_{t-1}}} = 0$. The second inequality is obtained since $\mu \in \neg_{\varepsilon} z_t$ (as $z_{t} \notin \cZ_{\varepsilon}(\mu)$) and the third inequality uses $\probability \left[ A , B \right] \leq \probability \left[ A \right]$. The equality is a direct consequence of our notations and the Gaussian assumption with $\sigma^2 =1$, i.e. $V_{N} = \sum_{a \in \cK} N^a a a \transpose$ and $d_{\text{KL}}(\mu^a, \lambda^a) = \frac{1}{2}\|\mu - \lambda \|^2_{a a \transpose}$.

Using the concavity of $x \mapsto 2 \ln \left(4 + \ln \left(x\right)\right)$ and the fact that $\sum_{a \in \cK} N_t^a = t$, we obtain $\sum_{a \in \cK} 2 \ln \left(4 + \ln \left( N_t^a\right)\right) \leq 2 K\ln \left(4 + \ln \left(\frac{t}{K}\right)\right)$. Since for $u \leq v$, $\probability[X > v] \leq \probability[X > u]$, by defining the stopping threshold as in (\ref{eq:definition_stopping_threshold}) and using Lemma~\ref{lem:corollary_10_kaufmann_2018_MixtureMartingalesRevisited} for $S=\cK$, we obtain:
\begin{align*}
	\probability_{\mu} \left[ \taud < + \infty , \hat{z} \notin \cZ_{\varepsilon}(\mu)\right] &\leq \probability_{\mu} \left[ \exists t \in \Natural, \sum_{a \in \cK} N_{t-1}^{a} d_{\text{KL}}(\mu_{t-1}^{a} , \mu^a) >  \beta(t-1, \delta) \right] \\
	&\leq \probability_{\mu} \left[ \exists t \in \Natural: \sum_{a \in \cK} N_{t-1}^{a} d_{\text{KL}}(\mu_{t-1}^{a} , \mu^a) >  \sum_{a \in \cK} 2 \ln \left(4 + \ln \left( N_{t-1}^a\right)\right) + K \cC_{G}\left(\frac{\ln \left( \frac{1}{\delta} \right)}{K}\right) \right] \\
	&\leq \delta
\end{align*}

This concludes the proof for stopping-recommendation pair evaluating the stopping criterion at each time $t$. When the stopping criterion is evaluated only on some predefined indices of time $\cT \subseteq \Natural$ with $|\cT| = \infty$, for example on a geometric grid, the proof is identical. Since $\probability \left[\exists t \in \cT, A_t \right] \leq \probability \left[ \exists t \in \Natural, A_t  \right] $, we perform the same manipulations with $\cT$ instead of $\Natural$ and conclude
\begin{align*}
	\probability_{\mu} \left[ \exists t \in \cT, \sum_{a \in \cK} N_{t-1}^{a} d_{\text{KL}}(\mu_{t-1}^{a} , \mu^a) >  \beta(t-1, \delta) \right] &\leq \probability_{\mu} \left[ \exists t \in \Natural, \sum_{a \in \cK} N_{t-1}^{a} d_{\text{KL}}(\mu_{t-1}^{a} , \mu^a) >  \beta(t-1, \delta) \right]
\end{align*}
\end{proof}

Note that Lemma~\ref{lem:delta_PAC_recommendation_stopping_pair} covers all the stopping-recommendation pairs considered in Section~\ref{sec:subsection_stopping_recommendation_pairs}: the different candidate answers (greedy, instantaneous furthest and furthest answers) and the computational relaxations (each time $t$, sticky and lazy updates).

\section{Proof of Theorem~\ref{thm:asymptotic_optimality_algorithm}} \label{app:section_proof_thm_asymptotic_optimality_algorithm}

In this appendix, we provide the full proof of Theorem~\ref{thm:asymptotic_optimality_algorithm}, which is recalled below. It is organized as follows:
\begin{itemize}
	\item In Appendix~\ref{app:subsection_modular_proof_scheme}, the proof scheme is first sketched, then detailed and proved.
	\item In Appendix~\ref{app:subsection_stopping_recommendation_pairs}, we show that combining an instantaneous furthest answer $z_t \in z_{F}(\mu_{t-1}, N_{t-1})$ as recommendation rule and the stopping criterion (\ref{eq:definition_stopping_criterion}) yield the desired property (\ref{eq:desired_condition_stopping_recommendation_pair}) to obtain asymptotic optimality. Moreover, we introduce the \textit{lazy} and \textit{sticky} computational relaxations for the stopping-recommendation pair, which both satisfy the property (\ref{eq:desired_condition_stopping_recommendation_pair}).
	\item In Appendix~\ref{app:subsection_sampling_rules}, the sampling rule used in \hyperlink{algoLeBAI}{L$\varepsilon$BAI} is shown to satisfy the desired property (\ref{eq:desired_condition_sampling_rule}) on the sampling rule.
	\item We assemble the different blocks of the proof in Appendix~\ref{app:subsection_summarized_proof}, yielding the Theorem~\ref{thm:asymptotic_optimality_algorithm}.
	\item Technical results are stated and proved in Appendix~\ref{app:subsection_technical_arguments}.
\end{itemize}

We shall assume that $\|\mu_{t}\|_{2} \leq M$ holds for all $t$. This can be guaranteed by projecting the estimates onto the set of realizable models. We will abuse the Landau's notations $o$, $\cO$ and $\Theta$ to cope for $t$ and $\ln\left(\frac{1}{\delta}\right)$ simultaneously, while discarding crossed terms. The notations $\tilde{o}$, $\tilde{\cO}$ and $\tilde{\Theta}$ are used similarly with the simplification of dropping the poly-logarithmic multiplicative factors.

For the sake of generality, the proofs will be conducted with parameters $b>0$ and $\alpha>1$ in the definitions of the exploration bonus $f(t) \eqdef 2 \beta \left(t, t^{1/\alpha}\right)$ and the slacks $(c_{s}^{a})_{(s,a) \in \left\llbracket n_{0}, t-1\right\rrbracket \times \cK}$ defined below in (\ref{eq:optimistic_slack_definition}). In the main content, we directly used the values chosen thanks to the analysis $(b,\alpha)=(1,3)$ as discussed in Appendix~\ref{app:subsection_summarized_proof}.

\subsection{Proof Scheme} \label{app:subsection_modular_proof_scheme}

The proof scheme sketched below is inspired by recent works using a game approach \citep{degenne_2020_GamificationPureExploration}. We present below a sketch of the proof, see Appendix~\ref{app:subsection_modular_proof_scheme_details} for a more detail proof scheme.

In order to obtain an asymptotic upper bound on $\expectedvalue_{\mu} \left[ \taud \right]$, we derive a non-asymptotic one and take the limit $\delta \rightarrow 0$. Having multiple $\varepsilon$-optimal answers is a key difficulty in several arguments. Our main contribution with respect to this proof lies in overcoming this hurdle.

Using Lemma~\ref{lem:lemma_13_degenne_2020_GamificationPureExploration} (Lemma 13 in \citet{degenne_2020_GamificationPureExploration}), the proof boils down to proving the existence and an upper bound on $T_{1}(\delta) \in \Natural$, such that for all $t \geq T_{1}(\delta)$ if a concentration event $\cE_t$ holds then the algorithm has already stopped, $\cE_{t} \subseteq \{\taud \leq t\}$. To obtain the asymptotic optimality of the identification strategy, the upper bound should satisfy $T_1(\delta) \leq T_{\varepsilon}(\mu)\ln\left(\frac{1}{\delta}\right) + o\left( \ln\left(\frac{1}{\delta}\right) \right)$, where the dependency in $t$ vanishes when $\delta \rightarrow 0$.

\begin{lemma}[Lemma 13 in \citet{degenne_2020_GamificationPureExploration}] \label{lem:lemma_13_degenne_2020_GamificationPureExploration}
	Let $(\cE_t)_{t\geq 1}$ be a sequence of concentration events, such that for all $t \geq 1$, $\probability_{\mu}\left[ \cE_{t}^{\complement} \right] \leq \frac{1}{t^{\alpha}}$ and an identification strategy such that for all $\delta \in (0,1)$, there exists $T_{1}(\delta) \in \Natural$ such that for $t \geq T_{1}(\delta)$, $\cE_{t} \subseteq \{\taud \leq t\}$. Then, $\expectedvalue_{\mu} \left[ \taud \right] \leq T_{1}(\delta) + \frac{1}{\alpha-1}$.
\end{lemma}
\begin{proof}
Since $\alpha > 1$, using an integral-sum comparison we obtain that: $\sum_{t=1}^{+\infty} \frac{1}{t^{\alpha}} \leq \int_{x=1}^{+ \infty} \frac{1}{x^{\alpha}} \,dx = \frac{1}{\alpha-1} $. Since $\taud$ is a positive random variable, we have:
	\begin{align*}
		\expectedvalue_{\mu} \left[ \taud \right] = \sum_{t \in \Natural} \probability_{\mu}\left[ \taud > t \right] \leq T_{1}(\delta) + \sum_{t \geq T_{1}(\delta)} \probability_{\mu}\left[ \cE_{t}^{\complement} \right] \leq T_{1}(\delta) + \sum_{t \geq T_{1}(\delta)} \frac{1}{t^{\alpha}} \leq  T_{1}(\delta) + \frac{1}{\alpha-1}
	\end{align*}
where we split the sum in two terms and used that for $t \geq T_{1}(\delta)$, $\{\taud > t\} \subseteq \cE_{t}^{\complement}$.
\end{proof}

Considering the sequence of concentration events $(\cE_t)_{t\geq 1}$, where for all $t \geq 1$
\begin{equation} \label{eq:concentration_events_definition}
	\cE_t \eqdef \left\{ \forall s \leq t: \|\mu_{s} - \mu\|^2_{V_{N_s}} \leq f(t)\right\}
\end{equation}
ensures that $\probability_{\mu}\left[ \cE_{t}^{\complement} \right] \leq \frac{1}{t^{\alpha}}$ (Lemma~\ref{lem:concentration_event_holds_with_high_proba}), where $f(t) \eqdef 2 \beta\left(t, t^{1/\alpha}\right)$. To show the existence of $T_1(\delta)$ and an upper bound leading to asymptotic optimality, it is sufficient to show that under $\cE_t$, if the algorithm does not stop at time $t+1$, then
\begin{align*}
	t T_{\varepsilon}(\mu)^{-1} \leq \ln \left( \frac{1}{\delta} \right) + o\left( t + \ln \left( \frac{1}{\delta} \right)\right)
\end{align*}

To derive the above inequality, the analysis distinguishes between two independent components. Under $\cE_t$, if the algorithm does not stop at time $t+1$, the stopping-recommendation pair should satisfy
\begin{align*}
	2 \beta(t, \delta) \geq \max_{z \in \cZ}\inf_{\lambda \in \neg_{\varepsilon} z}  \|\mu - \lambda \|^2_{V_{N_t}} - o\left( t + \ln \left( \frac{1}{\delta} \right)\right)
\end{align*}
while the sampling rule has to verify
\begin{align*}
	\max_{z \in \cZ}\inf_{\lambda \in \neg_{\varepsilon} z}  \|\mu - \lambda \|^2_{V_{N_t}} \geq 2t T_{\varepsilon}(\mu)^{-1} - o\left( t + \ln \left( \frac{1}{\delta} \right)\right)
\end{align*}

The expression $\max_{z \in \cZ}\inf_{\lambda \in \neg_{\varepsilon} z}  \|\mu - \lambda \|^2_{V_{N_t}}$ only feature the empirical counts. Therefore, our proof scheme allows to combine any stopping-recommendation pair and any sampling rule, provided they satisfy the corresponding inequality. The fact that our algorithms are $(\varepsilon, \delta)$-PAC is a direct consequence of Lemma~\ref{lem:delta_PAC_recommendation_stopping_pair}.

\subsubsection{Details and Proofs} \label{app:subsection_modular_proof_scheme_details}

As stated in Subsection~\ref{sec:subsection_sample_complexity_upper_bound}, the proof
scheme is inspired by recent works using a game approach \citep{degenne_2019_NonAsymptoticPureExploration,degenne_2020_GamificationPureExploration}. To obtain an asymptotic upper bound on the expected sample complexity $\expectedvalue_{\mu} \left[ \taud \right]$, we first derive a non-asymptotic upper bound, which might be loose as a function of $t$. Then, we consider the limit $\delta \rightarrow 0$.

The first step towards the non-asymptotic upper-bound on $\expectedvalue_{\mu} \left[ \taud \right]$ is Lemma~\ref{lem:lemma_13_degenne_2020_GamificationPureExploration}, whose proof is inspired by Lemma 13 in \citet{degenne_2020_GamificationPureExploration}. This approach is due to \citet{garivier_2016_OptimalBestArm}.

In Lemma~\ref{lem:concentration_event_holds_with_high_proba}, it is shown that the sequence of concentration events $(\cE_t)_{t\geq 1}$ defined in (\ref{eq:concentration_events_definition}), for all $t \geq 1$
\begin{equation*}
	\cE_t \eqdef \left\{ \forall s \leq t: \|\mu_{s} - \mu\|^2_{V_{N_s}} \leq f(t)\right\}
\end{equation*}
satisfies the first condition of Lemma~\ref{lem:lemma_13_degenne_2020_GamificationPureExploration}. Recall that $\beta(t,\delta) = \Theta\left(\ln\left(\frac{1}{\delta}\right) + \ln(\ln(t))\right)$. Since $f(t) \eqdef 2 \beta\left(t, t^{1/\alpha}\right)$, $f$ is a logarithmic function of $t$, i.e. $f(t) = \Theta(\ln(t))$.

\begin{lemma} \label{lem:concentration_event_holds_with_high_proba}
Considering $(\cE_t)_{t\geq 1}$ in (\ref{eq:concentration_events_definition}), we have: for all $t \geq 1$, $	\probability_{\mu}\left[ \cE_{t}^{\complement} \right] \leq \frac{1}{t^{\alpha}}$
\end{lemma}
\begin{proof}
Similarly to the proof of Lemma~\ref{lem:delta_PAC_recommendation_stopping_pair}, we will apply Corollary 10 in \citet{kaufmann_2018_MixtureMartingalesRevisited}:
\begin{align*}
	\probability_{\mu} \left[ \exists s \in \Natural: \sum_{a \in \cK} N_{s}^{a} d_{\text{KL}}(\mu_{s}^{a} , \mu^a) >  \sum_{a \in \cK} 2 \ln \left(4 + \ln \left( N_{s}^a\right)\right) + K \cC_{G}\left(\frac{\ln \left( \frac{1}{\delta} \right)}{K}\right) \right] \leq \delta
\end{align*}

By concavity of $x \mapsto\ln(4 + \ln(x))$ (which is also an increasing function) and $\sum_{a \in \cK} N_{s}^{a} = s$, we have: for all $s \in [t]$,
\begin{align*}
	\sum_{a \in \cK} 2 \ln(4 + \ln(N_{s}^{a})) \leq 2 K \ln \left(4 + \ln\left(\frac{s}{K}\right)\right) \leq 2 K  \ln \left(4 + \ln\left(\frac{t}{K}\right)\right)
\end{align*}

Recall that for $u \leq v$, $\probability[X > v] \leq \probability[X > u]$ and $\probability_{\mu}\left[\bigcup_{s \in [t]} A_s\right] \leq \probability_{\mu}\left[\bigcup_{s \in \Natural} A_s\right]$.
\begin{align*}
	\probability_{\mu}\left[ \cE_{t}^{\complement} \right] &= \probability_{\mu}\left[ \exists s \in [t], \|\mu_{s} - \mu\|^2_{V_{N_s}} > 2 \beta \left(t, t^{-\alpha}\right) \right] \\
	& \leq \probability_{\mu}\left[ \exists s \in [t], \sum_{a \in \cK} N_{s}^{a} d_{\text{KL}}(\mu_{s}^{a} , \mu^a) >  \sum_{a \in \cK} 2 \ln \left(4 + \ln \left( N_{s}^a\right)\right) + K \cC_{G}\left(\frac{\ln \left( t^{\alpha} \right)}{K}\right) \right] \\
	& \leq \probability_{\mu} \left[ \exists s \in \Natural: \sum_{a \in \cK} N_{s}^{a} d_{\text{KL}}(\mu_{s}^{a} , \mu^a) >  \sum_{a \in \cK} 2 \ln \left(4 + \ln \left( N_{s}^a\right)\right) + K \cC_{G}\left(\frac{\ln \left( t^{\alpha} \right)}{K}\right) \right] \leq \frac{1}{t^{\alpha}}
\end{align*}
\end{proof}

Lemma~\ref{lem:good_strategy_for_concentration_event_implies_we_stopped} gives a sufficient condition on the identification strategy to ensure there exists $T_1(\delta) \in \Natural$ satisfying the second condition of Lemma~\ref{lem:lemma_13_degenne_2020_GamificationPureExploration}.

\begin{lemma} \label{lem:good_strategy_for_concentration_event_implies_we_stopped}
Let $(\cE_t)_{t\geq 1}$ be a sequence of concentration events. Assume that the identification strategy verifies the following property: there exists $\beta_1, \beta_2 \in \left(0, 1 \right]$ such that if the algorithm does not stop at time $t+1$ and $\cE_t$ holds then
	\begin{equation} \label{eq:definition_sample_efficient_identification_strategy}
		H_{\mu}(t, \delta) \geq t T_{\varepsilon}(\mu)^{-1}  \quad \quad \text{where  } H_{\mu}(t, \delta) = \ln \left( \frac{1}{\delta} \right) + \tilde{\cO}\left( t^{1-\beta_1} + \ln \left(1/\delta\right)^{1-\beta_2}\right)
	\end{equation}
Then, for all $\delta \in (0,1)$, there exists $T_{1}(\delta) \in \Natural$ such that for $t \geq T_{1}(\delta)$, $\cE_{t} \subseteq \{\taud \leq t\}$.
\end{lemma}
\begin{proof}
Assume there exists an identification strategy satisfying the above condition. Let $T(\delta)$ be the maximum of the $t \in \Natural$ such that
\begin{equation} \label{eq:t_delta_definition}
	H_{\mu}(t, \delta) \geq t T_{\varepsilon}(\mu)^{-1}
\end{equation}
where $T(\delta)$ always exists since $H_{\mu}(t, \delta) = \cO(t^{1-\beta_1}) = o(t)$. Let $\delta_{\min}$ be the largest $\delta \in (0,1)$ such that
\begin{equation} \label{eq:delta_min_definition}
	H_{\mu}\left(\ln(1/\delta)^{(1- \beta_1)^{-1}}, \delta\right) < \ln(1/\delta)^{(1- \beta_1)^{-1}} T_{\varepsilon}(\mu)^{-1}
\end{equation}
where $\delta_{\min}$ always exists since $H_{\mu}\left(\ln(1/\delta)^{(1- \beta_1)^{-1}}, \delta\right) = \cO\left( \ln(1/\delta) \right) = o\left( \ln(1/\delta)^{(1- \beta_1)^{-1}} \right) $. It depends only on the parameters of the problem: $T_{\varepsilon}(\mu)$, $L$, $M$, $d$ and $\alpha$.

For $\delta \leq \delta_{\min}$, we have $H_{\mu}\left(\ln(1/\delta)^{(1- \beta_1)^{-1}}, \delta\right) < \ln(1/\delta)^{(1- \beta_1)^{-1}} T_{\varepsilon}(\mu)^{-1} $, hence $T(\delta) < \ln(1/\delta)^{(1- \beta_1)^{-1}}$. Since $t \mapsto H_{\mu}(t, \delta)$ is increasing, plugging $\ln(1/\delta)^{(1- \beta_1)^{-1}} $ in the left-hand side of (\ref{eq:t_delta_definition}), we obtain $T(\delta) < T_{0}(\delta)$ where
\begin{equation} \label{eq:t_0_delta_definition}
	T_{0}(\delta) \eqdef T_{\varepsilon}(\mu) H_{\mu}\left(\ln(1/\delta)^{(1- \beta_1)^{-1}}, \delta\right)
\end{equation}
Moreover, under $\cE_t$ and if we do not stop, it implies that $t \leq T(\delta)$. To sum up, we have shown: (1) for $\delta \leq \delta_{\min}$, for $t \geq T_{0}(\delta)$, we know that $\taud \leq t$ and (2) for $\delta > \delta_{\min}$, for $t \geq T(\delta)+1$, we know that $\taud \leq t$. In both cases, this means that $\cE_{t} \subseteq \{\taud \leq t\}$. Therefore, there exists $T_{1}(\delta)$ defined as
\begin{equation} \label{eq:t_1_delta_definition}
	T_{1}(\delta) \eqdef \begin{cases}
						T_{0}(\delta) &\text{if } \delta \leq \delta_{\min} \\
						T(\delta)+1 &\text{else.}
						\end{cases}
\end{equation}
such that for $t \geq T_{1}(\delta)$, $\cE_{t} \subseteq \{\taud \leq t\}$.
\end{proof}

By definition of $T_{1}(\delta)$ in (\ref{eq:t_1_delta_definition}), we have $T_{1}(\delta) = T_{\varepsilon}(\mu) \ln \left( \frac{1}{\delta}\right) + o\left(\ln \left( 1/\delta\right) \right)$. This fact is crucial to obtain the upper bound with the right constant when considering $\delta \rightarrow 0$, hence it is necessary to prove asymptotic optimality.

The last and hardest component of the proof is to show that: under $(\cE_t)_{t\geq 1}$ (\ref{eq:concentration_events_definition}) and if the algorithm does not stop at time $t+1$, then the identification strategy satisfies the condition (\ref{eq:definition_sample_efficient_identification_strategy}) of Lemma~\ref{lem:good_strategy_for_concentration_event_implies_we_stopped}. When this holds, we say that the strategy is \textit{sample-efficient}.

Plugging together all the previous results, we obtain directly Lemma~\ref{lem:asymptotic_optmiality_when_sample_efficient_identification_startegy}.
\begin{lemma} \label{lem:asymptotic_optmiality_when_sample_efficient_identification_startegy}
Let $(\cE_t)_{t\geq 1}$ as in (\ref{eq:concentration_events_definition}). Using a sample-efficient $(\varepsilon, \delta)$-PAC strategy, i.e. (\ref{eq:definition_sample_efficient_identification_strategy}) holds true, yields an $(\varepsilon, \delta)$-PAC strategy and, for all $\mu \in \cM$ such that $|z_{F}(\mu)|=1$,
\begin{align*}
	\limsup_{\delta \rightarrow 0} \frac{\expectedvalue_{\mu} \left[ \taud \right]}{\ln \left( \frac{1}{\delta}\right)} \leq T_{\varepsilon}(\mu)
\end{align*}
\end{lemma}
\begin{proof}
Under these hypotheses, the two conditions of Lemma~\ref{lem:lemma_13_degenne_2020_GamificationPureExploration} are fulfilled. Using the definition of $(\cE_t)_{t\geq 1}$ in (\ref{eq:concentration_events_definition}), the first condition holds thanks to Lemma~\ref{lem:concentration_event_holds_with_high_proba}. By definition of a sample-efficient identification strategy we can apply Lemma~\ref{lem:good_strategy_for_concentration_event_implies_we_stopped}, hence the second condition is also true. Therefore, applying Lemma~\ref{lem:lemma_13_degenne_2020_GamificationPureExploration}, we have: for all $\delta \in (0,1)$
\begin{equation} \label{eq:finite_time_upper_bound_sample_complexity}
	\mathbb{E}_{\mu} \left[ \taud \right] \leq T_{1}(\delta) + \frac{1}{\alpha -1}
\end{equation}
with $T_{1}(\delta) = T_{\varepsilon}(\mu) \ln \left( \frac{1}{\delta}\right) + o\left(\ln \left( 1/\delta\right) \right)$. Dividing by $\ln \left( \frac{1}{\delta}\right)$ on both side and taking $\limsup_{\delta \rightarrow 0}$, we obtain the asymptotic optimality of this algorithm.

By assumption, the algorithm was $(\varepsilon, \delta)$-PAC.
By the inequality (\ref{eq:finite_time_upper_bound_sample_complexity}), we have that $\taud$ is finite almost surely, i.e. $\probability_{\mu} \left[ \taud < + \infty \right] = 1$.
\end{proof}

\paragraph{\textit{Sample-Efficient} Identification Strategy} When proving that a strategy is sample-efficient, we distinguish two independent conditions. Under $\cE_t$ if the algorithms does not stop at time $t+1$, the stopping-recommendation pair should satisfy
\begin{equation} \label{eq:desired_condition_stopping_recommendation_pair}
	2 \beta(t, \delta) \geq \max_{z \in \cZ}\inf_{\lambda \in \neg_{\varepsilon} z}  \|\mu - \lambda \|^2_{V_{N_t}} - \tilde{\cO}\left( t^{1-\beta_1} + \ln \left(1/\delta\right)^{1-\beta_2}\right)
\end{equation}
while the sampling rule has to verify
\begin{equation} \label{eq:desired_condition_sampling_rule}
	\max_{z \in \cZ}\inf_{\lambda \in \neg_{\varepsilon} z}  \|\mu - \lambda \|^2_{V_{N_t}} \geq 2t T_{\varepsilon}(\mu)^{-1} - \tilde{\cO}\left( t^{1-\beta_1} + \ln \left(1/\delta\right)^{1-\beta_2}\right)
\end{equation}
with $(\beta_1, \beta_2) \in \left(0, 1 \right]^2$. By grouping together both inequalities (\ref{eq:desired_condition_stopping_recommendation_pair}-\ref{eq:desired_condition_sampling_rule}), we obtain directly the condition of Lemma~\ref{lem:good_strategy_for_concentration_event_implies_we_stopped}. The above lemmas are assembled in a global proof in Appendix~\ref{app:subsection_summarized_proof}.

\subsection{Stopping-Recommendation Pairs} \label{app:subsection_stopping_recommendation_pairs}

Lemma~\ref{lem:reco_stop_pair_instantfurthest_eachtime_inequality_under_concentration_and_no_stop} shows that the desired condition on the stopping-recommendation pair (\ref{eq:desired_condition_stopping_recommendation_pair}) holds when considering an instantaneous furthest answer as recommendation rule and evaluating the stopping rule at each time $t$. In particular, it holds for $\left(\beta_1,\beta_2\right) = \left(1, \frac{1}{2}\right)$ (the poly-logarithmic dependence in $t$ is hidden in the notation $\tilde{\cO}$). The extension of the proof for other update schemes is detailed in Appendix~\ref{app:subsubsection_other_update_schemes}.

Note that while the condition on the recommendation rule to have a $(\varepsilon, \delta)$-PAC algorithm is very mild, i.e. $z_t \in \cZ_{\varepsilon}(\mu_{t-1})$, the choice of $z_t$ is crucial to obtain an asymptotically optimal algorithm. Intuitively it should converge asymptotically towards $z_{F}(\mu)$. For this reason, using a greedy answer $z_t \in z^{\star}(\mu_t)$ as recommendation rule is doomed to fail whenever $z_{F}(\mu) \notin z^{\star}(\mu)$ (which is often the case). It remains unclear whether using a furthest answer for the current estimator $z_t \in z_F(\mu_t)$ as recommendation rule would result in an asymptotically optimal algorithm. While converging to $z_{F}(\mu)$, it might be inefficient when associated with the stopping criterion (\ref{eq:definition_stopping_criterion}). This is an interesting open question to investigate in future work.

\begin{lemma} \label{lem:reco_stop_pair_instantfurthest_eachtime_inequality_under_concentration_and_no_stop}
	Regardless of the sampling rule, an identification strategy evaluating at each time $t$ the stopping rule (\ref{eq:definition_stopping_criterion}) with stopping threshold $\beta(t, \delta)$ (\ref{eq:definition_stopping_threshold}) for an instantaneous furthest answer $z_{t} \in z_{F}(\mu_{t-1}, N_{t-1})$ satisfies that, under $\cE_t$ as in (\ref{eq:concentration_events_definition}), if the algorithm does not stop at time $t+1$, then
	\begin{align*}
		2 \beta(t, \delta) \geq \max_{z \in \cZ} \inf_{\lambda \in \neg_{\varepsilon} z}  \|\mu - \lambda \|^2_{V_{N_t}} - h_{\delta}(t)
	\end{align*}
	where $h_{\delta}(t) =  \sqrt{8 f(t) \beta(t, \delta) } + 4 f(t) = \Theta\left( \ln(t) + \sqrt{\ln \left( \frac{1}{\delta}\right)}\right)$.
\end{lemma}
\begin{proof}
Fix any time $t \geq 1$. Suppose that $\cE_t$ holds and the algorithm does not stop at time $t+1$. From the stopping rule and the definition of $z_{t+1} \in z_{F}(\mu_{t}, N_{t})$, we obtain
\begin{equation} \label{eq:step_1_proof_recommendation_stopping_pair}
	2 \beta(t, \delta) \geq \inf_{\lambda \in \neg_{\varepsilon} z_{t+1}} \|\mu_t - \lambda\|^2_{V_{N_t}} = \max_{z \in \cZ}\inf_{\lambda \in \neg_{\varepsilon} z} \|\mu_t - \lambda\|^2_{V_{N_t}} \geq \max_{z \in z_{F}(\mu, N_{t})} \inf_{\lambda \in \neg_{\varepsilon} z} \|\mu_t - \lambda\|^2_{V_{N_t}}
\end{equation}
where the last inequality uses Lemma~\ref{lem:switch_alternative_set}. Combining the triangular inequality and concentration event $\cE_t$,
\begin{align*}
	\|\mu_t - \lambda\|^2_{V_{N_t}} &\geq (\|\mu - \lambda\|_{V_{N_t}} - \|\mu_t - \mu\|_{V_{N_t}})^2 \\
	&\geq  \|\mu - \lambda\|^2_{V_{N_t}} - 2 \|\mu - \lambda\|_{V_{N_t}} \|\mu_t - \mu\|_{V_{N_t}} \\
	&\geq \|\mu - \lambda\|^2_{V_{N_t}} - 2 \|\mu - \lambda\|_{V_{N_t}} \sqrt{f(t)}  \: .
\end{align*}
Let $\tilde z \in \argmax_{z \in z_{F}(\mu, N_{t})} \inf_{\lambda \in \neg_{\varepsilon} z} \|\mu_t - \lambda\|^2_{V_{N_t}}$ and $\tilde{\lambda} \in \argmin_{\lambda \in \neg_{\varepsilon} \tilde z} \|\mu_t - \lambda\|^2_{V_{N_t}}$. Then,
\begin{align*}
	\|\mu_t - \tilde{\lambda}\|^2_{V_{N_t}} &\geq \|\mu - \tilde{\lambda}\|^2_{V_{N_t}} - 2 \sqrt{f(t)} \sqrt{\|\mu - \tilde{\lambda}\|^2_{V_{N_t}} }
\end{align*}
Using Lemma~\ref{lem:lemma_28_degenne_2020_GamificationPureExploration} for $y= \|\mu_t - \tilde{\lambda}\|^2_{V_{N_t}} $, $\alpha = 2 \sqrt{f(t)}$ and $x=\|\mu - \tilde{\lambda}\|^2_{V_{N_t}}$, we obtain
\begin{align*}
	 \|\mu_t - \tilde{\lambda}\|^2_{V_{N_t}} &\geq \|\mu - \tilde{\lambda}\|^2_{V_{N_t}} - 2 \sqrt{f(t)} \sqrt{\|\mu_t - \tilde{\lambda}\|^2_{V_{N_t}} } - 4 f(t) \\
	 &\geq \|\mu - \tilde{\lambda}\|^2_{V_{N_t}} - 2 \sqrt{f(t)} \sqrt{2 \beta(t, \delta) } - 4 f(t) \\
	 &\geq \inf_{\lambda \in \neg_{\varepsilon} \tilde z}  \|\mu - \lambda \|^2_{V_{N_t}} - \sqrt{8 f(t) \beta(t, \delta) } - 4 f(t)\\
	 &= \max_{z \in \cZ}\inf_{\lambda \in \neg_{\varepsilon} \tilde z}  \|\mu - \lambda \|^2_{V_{N_t}} - \sqrt{8 f(t) \beta(t, \delta) } - 4 f(t) \: .
\end{align*}
The second inequality is obtained by using (\ref{eq:step_1_proof_recommendation_stopping_pair}) and the definition of $\tilde{\lambda}$. The third is obtained by taking the $\inf_{\lambda \in \neg_{\varepsilon} \tilde z}$, which is possible since $\tilde{\lambda} \in \neg_{\varepsilon} \tilde z$. The equality uses that $\tilde z \in z_{F}(\mu, N_{t})$. Therefore, we have shown that
\begin{align*}
	2 \beta(t, \delta) \geq  \max_{z \in \cZ} \inf_{\lambda \in \neg_{\varepsilon} z}  \|\mu - \lambda \|^2_{V_{N_t}} - h_{\delta}(t)
\end{align*}
where $h_{\delta}(t) =  \sqrt{8 f(t) \beta(t, \delta) } + 4 f(t) = \Theta\left( \ln(t) + \sqrt{\ln \left( \frac{1}{\delta}\right)}\right)$, hence sub-linear function of both $t$ and $\ln \left( \frac{1}{\delta}\right)$.
\end{proof}

\subsubsection{Computational Relaxations} \label{app:subsubsection_other_update_schemes}

Depending on $\cZ$, computing an instantaneous furthest answer at each time $t$ might be too costly. To reduce the computational cost, we propose two different relaxations for any stopping-recommendation pair: the \textit{lazy} version and the \textit{sticky} one. The main idea behind both of them is to avoid the computation of a new recommendation rule at each time. Defining a grid of time $\cT \subseteq \Natural$, we stick to the previous candidate answer when $t \notin \cT$, i.e. $z_{t} = z_{t-1}$, else we compute a new one. To satisfy the sole requirement on the recommendation rule (Lemma~\ref{lem:delta_PAC_recommendation_stopping_pair}), a new value has to be computed when $t \notin \cT$ and $z_{t-1} \notin \cZ_{\varepsilon}(\mu_{t-1})$.

Both computational relaxations use the stopping criterion defined in (\ref{eq:definition_stopping_criterion}). While the sticky version evaluates it at each time $t > n_{0}$, the lazy version only does it when the recommendation rule is updated ($t \in \cT$). It is important to notice that thanks to Lemma~\ref{lem:delta_PAC_recommendation_stopping_pair}, the lazy and sticky computational relaxations yield $(\varepsilon, \delta)$-PAC strategy regardless of the sampling rule.

In the following, we consider an instantaneous furthest answer as recommendation rule and a geometric-like grid of time $\cT$. To ensure asymptotic optimality of the algorithm, we need to have a strictly positive and decreasing expansion parameter $(\gamma_{i})_{i \in \Natural^{\star}}$ such that $\gamma_{i} \rightarrow 0$. More formally, we use $\cT \eqdef \left\{n_{0}+1\right\} \cup \left\{t_{i} \right\}_{i \in \Natural}$ where $t_{0} > n_{0}$ denotes the end of the first phase and for all $i \in \Natural^{\star}$, $t_{i} = \lceil(1 + \gamma_{i}) t_{i-1}\rceil$. Since $\gamma_{i}$ is strictly positive, we have $t_{i} > t_{i-1}$ and $\cT$ is an infinite set of times, hence Lemma~\ref{lem:delta_PAC_recommendation_stopping_pair} applies. Since the main components of the proof are unchanged, we only state the noteworthy modifications.

\paragraph{Lazy Update} When considering the lazy update, both the candidate answer and the stopping criterion are computed only when $t \in \cT$. Under $\cE_t$, if the algorithm doesn't stop at time $t+1$, there exists $i \in \Natural$ such that $t \in \left\llbracket t_i, t_{i+1} - 1 \right\rrbracket$ such that the stopping criterion is not met at time $t_i+1$:
\begin{align*}
	2 \beta(t_i, \delta) \geq \inf_{\lambda \in \neg_{\varepsilon} z_{t_i+1}} \|\mu_{t_i} - \lambda\|^2_{V_{N_{t_i}}} \geq \max_{z \in \cZ} \inf_{\lambda \in \neg_{\varepsilon} z}  \|\mu - \lambda \|^2_{V_{N_{t_i}}} - h_{\delta}(t_i)
\end{align*}
where $h_{\delta}$ is defined in Lemma~\ref{lem:reco_stop_pair_instantfurthest_eachtime_inequality_under_concentration_and_no_stop} and the second inequality is obtained as above since $\cE_{t_i} \subseteq\cE_t$ and $z_{t_i+1} = z_{F}(\mu_{t_i}, N_{t_i})$. Therefore, using the lazy update with an instantaneous furthest answer as recommendation rule allow to satisfy the condition (\ref{eq:desired_condition_stopping_recommendation_pair}) at time $t_i$.

Given a sampling rule satisfying the condition (\ref{eq:desired_condition_sampling_rule}) at time $t_i$, the rest of the proof is exactly the same as sketched in Appendix~\ref{app:subsection_modular_proof_scheme} and proved in Appendix~\ref{app:subsection_sampling_rules} with $t_i$ instead of $t$. Therefore, we will obtain
\begin{align*}
	H_{\mu}(t_i, \delta) \geq t_i T_{\varepsilon}(\mu)^{-1}
\end{align*}

Since $t \in \left\llbracket t_i, t_{i+1} - 1 \right\rrbracket$, we have $H_{\mu}(t_i, \delta) \leq H_{\mu}(t, \delta)$ and $t_i \geq \frac{t_{i+1} - 1}{1+\gamma_{i+1}} \geq \frac{t}{1+\gamma_{i+1}}$. This yields
\begin{align*}
	\tilde{H}_{\mu}(t, \delta) \eqdef (1+\gamma_{i+1}) H_{\mu}(t, \delta) \geq t T_{\varepsilon}(\mu)^{-1}
\end{align*}

We remark that $t \rightarrow + \infty$ implies $i \rightarrow + \infty$, hence $\gamma_{i+1} \rightarrow_{t \rightarrow + \infty} 0$. Choosing $\gamma_{i}$ independently of $\delta$ and noting that $\tilde{H}_{\mu}(t, \delta) = \ln \left( \frac{1}{\delta} \right) + \cO\left( t^{1-\beta_1} + \ln \left(1/\delta\right)^{1-\beta_2}\right)$ since $H_{\mu}(t, \delta) = \ln \left( \frac{1}{\delta} \right) + \cO\left( t^{1-\beta_1} + \ln \left(1/\delta\right)^{1-\beta_2}\right)$, we can conclude the proof. Combining the lazy update with a good sampling rule and an instantaneous furthest answer as recommendation rule leads to asymptotically optimal algorithms.

For an expansion parameter which is constant, the manipulation above shows that the upper bound won't match the lower bound by a constant multiplicative factor $1+\gamma_0$. Therefore, the shrinking expansion parameter is crucial to obtain asymptotic optimality.

\paragraph{Sticky Update} When considering the sticky update, the candidate answer is computed only when $t \in \cT$, while the stopping criterion is evaluated at each time $t$. Since the GLRTs conducted with the lazy update are strictly included in the ones conducted with the sticky update (with matching candidate answer), the sticky approach is strictly better in terms of sample complexity than considering the lazy one. Therefore, we have $\tau_{\delta}^{\text{sticky}} \leq \tau_{\delta}^{\text{lazy}}$, hence $\lim_{\delta \rightarrow 0} \frac{\expectedvalue_{\mu}[\taud^{\text{sticky}}]}{\ln(1/\delta)} \leq \lim_{\delta \rightarrow 0} \frac{\expectedvalue_{\mu}[\taud^{\text{lazy}}]}{\ln(1/\delta)}$. Using the above result showing that lazy update leads to asymptotic optimal algorithms, we conclude that sticky update also leads to asymptotic optimal algorithms.

\subsection{Sampling Rule} \label{app:subsection_sampling_rules}

In this appendix, we show that the desired condition (\ref{eq:desired_condition_sampling_rule}) on the sampling rule holds when considering the sampling rule in \hyperlink{algoLeBAI}{L$\varepsilon$BAI} (Appendix~\ref{app:subsubsection_sampling_rule_LeBAI}). Appendix~\ref{app:subsubsection_key_assumptions} details the two algorithmic requirements enforced in \hyperlink{algoLeBAI}{L$\varepsilon$BAI}.

\subsubsection{Requirements}  \label{app:subsubsection_key_assumptions}

To conduct the analysis of the sampling rule used in \hyperlink{algoLeBAI}{L$\varepsilon$BAI}, two algorithmic requirements were needed to be enforced.
However, as experiments show in Appendix~\ref{app:subsection_additional_experiments}, competitive empirical performance are obtained without them.

\paragraph{Forced Exploration} To provide an upper bound on the number of times a good event doesn't occur (Lemma~\ref{lem:upper_bound_number_of_steps_away_from_zF}), forced exploration has to be introduced by using $w_t = \frac{1}{tK} \1_K + \left( 1 - \frac{1}{t} \right) w_t^{\cL^{\cK}}$. Good events will be formally introduced below in (\ref{eq:good_event_to_lower_bound_instantaneous_gain}). Since we are using tracking and $\sum_{s = n_{0} +1}^{t} \frac{1}{s} = \Theta(\ln t)$, the forced exploration only concerns a logarithmic number of steps.

\paragraph{Exact $\cZ$-Oracle}
Likewise, to provide an upper bound on the number of times a good event doesn't occur (Lemma~\ref{lem:upper_bound_number_of_steps_away_from_zF}), the $\cZ$-oracle has to satisfy $\tilde{z}_t \in z_{F}(\mu_{t-1})$.
Since this requirement is computationally intractable, it doesn't lead to a practical algorithm.
In the experiments, we use an instantaneous furthest answer $\tilde{z}_t = z_t \in z_{F}(\mu_{t-1}, N_{t-1})$ which has good empirical behavior.
Since it was already computed for the recommendation rule, no additional computation are needed.

While the proof still eludes us, we believe that using an instantaneous furthest answer would be also a theoretically valid choice for the $\cZ$-oracle.
Intuitively an instantaneous furthest answer is a good proxy for $z_{F}(\mu)$ because the empirical proportions will eventually converge to $w_{F}(\mu)$.
Using our current proof techniques, it is difficult to show this result since the convergence properties and upper bound on the expected sample complexity are highly intertwined.
Borrowing the proofs strategy of \cite{garivier_2016_OptimalBestArm}, it might be possible to show it with purely asymptotic arguments.
We leave this interesting question to future work.

\subsubsection{Proof of Desired Property} \label{app:subsubsection_sampling_rule_LeBAI}

To highlight the structure of the proof, we prove separately two inequalities satisfied by the sampling rule used in \protect\hyperlink{algoLeBAI}{L$\varepsilon$BAI} (Lemma~\ref{lem:samp_rule_acti_inequality_under_concentration_and_no_stop_cumulative_gain} and Lemma~\ref{lem:samp_rule_acti_inequality_under_concentration_and_no_stop_samp_complexity}). This intermediate step involves the cumulative gain of the learner used by the $\cK$-player.

\begin{lemma} \label{lem:samp_rule_acti_inequality_under_concentration_and_no_stop_cumulative_gain}
	Let $\cL^{\cZ}$ be the $\cZ$-oracle such that $\tilde{z}_{s} \in z_{F}(\mu_{s-1})$.
	Let $b > 0$. For all $s \in \left\llbracket n_{0}, t-1\right\rrbracket$ and $a \in \cK$, let $(c_{s}^{a})_{s \geq n_{0}, a \in \cK}$ be defined by
\begin{equation} \label{eq:optimistic_slack_definition}
	c_{s}^{a} \eqdef \min \left\{ f\left(s^{1+b}\right) \|a\|^2_{V_{N_{s}}^{-1}}, 4M^2L_{\cK}^2\right\}
\end{equation}
The optimistic gain $g_{s}(w) \eqdef (1-\frac{1}{s}) \langle w, U_s\rangle$ are defined such that: for all $a \in \cK$ and $s\in \left\llbracket n_{0}+1, t \right\rrbracket$,
\begin{equation} \label{eq:optimistic_gain_acti_definition}
	U_s^a \eqdef \left(\|\mu_{s-1} - \lambda_s\|_{a a\transpose} + \sqrt{c_{s-1}^{a}} \right)^2
\end{equation}
with $\lambda_s \eqdef \argmin_{\lambda \in \neg_{\varepsilon} \tilde{z}_s} \|\mu_{s-1} - \lambda\|^2_{V_{w_s}}$. Then, under $\cE_t$ if \hyperlink{algoLeBAI}{L$\varepsilon$BAI} does not stop at time $t+1$,
\begin{align*}
	\max_{z \in \cZ} \inf_{\lambda \in \neg_{\varepsilon} z}  \|\mu - \lambda \|^2_{V_{N_t}} \geq \sum_{s = n_{0}+1}^{t} g_s\left(w_s^{\cL^{\cK}}\right) - r_1(t) \: ,
\end{align*}
where
\begin{align*}
	&r_1(t) = C_{1}  + \left( N_{F}(t) + t^{\frac{1}{1+b}}  \right) 4KM^2L_{\cK}^2 + h'(t) + h''(t)  = \Theta\left( \ln(t) \sqrt{t} + \ln(t)^2 t^{\frac{1}{1+b}}\right) \\
	\text{with}\quad &h''(t) =  f\left(t^{1+b}\right) \left( K \ln(K) + 2K\ln(t) \right)  + 4ML_{\cK}  \sqrt{f\left(t^{1+b}\right)}  \left( K \ln(K) + \sqrt{8Kt}  \right) \\
	\text{and}\quad &h'(t) =  4ML_{\cK}   \sqrt{\left( K \ln(K) + 2K\ln(t) \right) f\left(t^{1+b}\right) t} + 8 M^2 L_{\cK}^2 t^{\frac{1}{1+b}} \: .
\end{align*}
\end{lemma}
\begin{proof}
\noindent\textbf{Step 1. From $N_{t}$ to $W_{t}$.} Let $W_t \eqdef \sum_{s \in [t]} w_{s}$. Using Lemma~\ref{lem:switch_alternative_set},
\begin{align*}
	 \max_{z \in \cZ} \inf_{\lambda \in \neg_{\varepsilon} z}  \|\mu - \lambda \|^2_{V_{N_t}}  &\geq \max_{z \in z_{F}(\mu, W_{t})} \inf_{\lambda \in \neg_{\varepsilon} z}  \|\mu - \lambda \|^2_{V_{N_t}} \: .
\end{align*}
Recall that for all $N \in \Real_{+}^{K}$, $\|\mu - \lambda \|^2_{V_{N}} = \sum_{a \in \cK} N^{a}\|\mu - \lambda \|^2_{a a\transpose}$. Using Cauchy-Schwartz and the bounded assumption, we have: $\|\mu - \lambda \|^2_{a a\transpose} =  \langle \mu - \lambda, a\rangle^2 \leq \|\mu - \lambda\|_2^2 \|a\|^2_2 \leq 4 M^2 L_{\cK}^2$. For any $\lambda \in \cM$, using Lemma~\ref{lem:theorem_6_degenne_2020_StructureAdaptiveAlgorithms},
\begin{align*}
	\|\mu - \lambda \|^2_{V_{N_t}} &\geq \|\mu - \lambda \|^2_{V_{W_t}} - \ln(K) \sum_{a \in \cK} \|\mu - \lambda \|^2_{a a\transpose} \\
	&\geq  \|\mu - \lambda \|^2_{V_{W_t}} - C_{1} \: ,
\end{align*}
where $C_{1} \eqdef 4 \ln(K) K M^2 L_{\cK}^2 = \Theta(1)$.

Let $\tilde z \in \argmax_{z \in z_{F}(\mu, W_{t})} \inf_{\lambda \in \neg_{\varepsilon} z}  \|\mu - \lambda \|^2_{V_{N_t}}$
and $\tilde{\lambda} \in \argmin_{\lambda \in \neg_{\varepsilon} \tilde z}  \|\mu - \lambda \|^2_{V_{N_t}}$. Using the above yields
\begin{align*}
	\max_{z \in \cZ} \inf_{\lambda \in \neg_{\varepsilon} z}  \|\mu - \lambda \|^2_{V_{N_t}} &\geq \|\mu - \tilde \lambda \|^2_{V_{W_t}} - C_{1} \\
	&\geq \inf_{\lambda \in \neg_{\varepsilon} \tilde z}\|\mu -  \lambda \|^2_{V_{W_t}} - C_{1} \\
	&= \max_{z \in \cZ} \inf_{\lambda \in \neg_{\varepsilon} z}  \|\mu - \lambda \|^2_{V_{W_t}} - C_{1} \: ,
\end{align*}
where we used that $\tilde \lambda \in \neg_{\varepsilon} \tilde z$ and $\tilde z \in z_{F}(\mu, W_{t})$.

\paragraph{Step 2. From $(\mu, z_{F}(\mu, W_{t}))$ to $(\mu_{s-1},\tilde{z}_{s})$.} Using the concavity of $\inf$ and $\|\mu - \lambda \|^2_{V_{W_t}} = \sum_{s =n_{0}+1}^{t} \|\mu - \lambda \|^2_{V_{w_s}}$, we obtain
\begin{align*}
	\max_{z \in \cZ} \inf_{\lambda \in \neg_{\varepsilon} z}   \|\mu - \lambda \|^2_{V_{W_t}}  \geq \max_{z \in \cZ} \sum_{s =n_{0}+1}^{t}  \inf_{\lambda \in \neg_{\varepsilon} z} \|\mu - \lambda \|^2_{V_{w_s}} \: .
\end{align*}

Using the triangular inequality, for all $z\in \cZ$ and $s \in \left\llbracket n_{0}+1, t \right\rrbracket$
\begin{align*}
	\|\mu - \lambda\|^2_{V_{w_s}} &\geq (\|\mu_{s-1} - \lambda\|_{V_{w_s}} - \|\mu_{s-1} - \mu\|_{V_{w_s}})^2 \geq  \|\mu_{s-1} - \lambda\|^2_{V_{w_s}} - 2  \|\mu_{s-1} - \mu\|_{V_{w_s}} \|\mu_{s-1} - \lambda\|_{V_{w_s}}
\end{align*}
For $\tilde{\lambda}^{z}_s \in \argmin_{\lambda \in \neg_{\varepsilon} z}  \|\mu - \lambda \|^2_{V_{w_s}}$, we obtain
\begin{align*}
	\|\mu - \tilde{\lambda}^{z}_s\|^2_{V_{w_s}} &\geq  \|\mu_{s-1} - \tilde{\lambda}^{z}_s\|^2_{V_{w_s}} - 2  \|\mu_{s-1} - \mu\|_{V_{w_s}} \|\mu_{s-1} - \tilde{\lambda}^{z}_s\|_{V_{w_s}} \\
	& \geq \inf_{\lambda \in \neg_{\varepsilon} z}\|\mu_{s-1} -\lambda\|^2_{V_{w_s}} - 2  \|\mu_{s-1} - \mu\|_{V_{w_s}} \|\mu_{s-1} - \tilde{\lambda}^{z}_s\|_{V_{w_s}} \\
	& \geq \inf_{\lambda \in \neg_{\varepsilon} z}\|\mu_{s-1} -\lambda\|^2_{V_{w_s}} - 4 ML_{\cK}  \|\mu_{s-1} - \mu\|_{V_{w_s}}
\end{align*}
The second to last inequality is obtained by taking the $\inf_{\lambda \in \neg_{\varepsilon} z}$, which is possible since $\tilde{\lambda}^{z}_s \in \neg_{\varepsilon} z$, and the last one by upper bounding $\|\mu_{s-1} - \tilde{\lambda}_s\|_{V_{w_s}} \leq \|\mu_{s-1} - \tilde{\lambda}_s\|_{2} \leq 2 ML_{\cK}$ (since $w_s \in \simplex$).
Summing those inequalities together, using Cauchy-Schwartz when $s\geq \left\lceil t^{\frac{1}{1+b}} \right\rceil$ and the boundedness assumption when $s < \left\lceil t^{\frac{1}{1+b}} \right\rceil$, we obtain: for all $z \in \cZ$
\begin{align*}
	\sum_{s =n_{0}+1}^{t} \inf_{\lambda \in \neg_{\varepsilon} z}  \|\mu - \lambda \|^2_{V_{w_s}}
	&\geq  \sum_{s =n_{0}+1}^{t} \inf_{\lambda \in \neg_{\varepsilon} z}  \|\mu_{s-1} - \lambda \|^2_{V_{w_s}} - 4 ML_{\cK} \sum_{s =n_{0}+1}^{t} \|\mu_{s-1} - \mu\|_{V_{w_s}}
	\\
	&\geq \sum_{s =n_{0}+1}^{t} \inf_{\lambda \in \neg_{\varepsilon} z}  \|\mu_{s-1} - \lambda \|^2_{V_{w_s}} - 4 ML_{\cK} \sqrt{t-\left\lceil t^{\frac{1}{1+b}} \right\rceil+1}\sqrt{\sum_{s =\left\lceil t^{\frac{1}{1+b}} \right\rceil}^{t} \|\mu_{s-1} - \mu\|^2_{V_{w_s}} }
	\\
	&\quad \quad- 8 M^2 L_{\cK}^2 \left( \left\lceil t^{\frac{1}{1+b}} \right\rceil - n_0 -1 \right)
	\\
	& \geq  \sum_{s =n_{0}+1}^{t} \|\mu_{s-1} - \lambda_s \|^2_{V_{w_s}} - 4ML_{\cK} \sqrt{t} \sqrt{\sum_{s =\left\lceil t^{\frac{1}{1+b}} \right\rceil}^{t} \|\mu_{s-1} - \mu\|^2_{V_{w_s}} } - 8 M^2 L_{\cK}^2 t^{\frac{1}{1+b}}
\end{align*}
The second inequality is obtained by concavity of $x \mapsto \sqrt{x}$, and  the last one uses the definition of the best-response oracle in the sampling rule, $\lambda_s \in \argmin_{\lambda \in \neg_{\varepsilon} \tilde{z}_{s}}  \|\mu_{s-1} - \lambda \|^2_{V_{w_s}}$. Using Lemma~\ref{lem:small_cumulative_reweighted_deviation_mle_mu}, we know that $\sum_{s = \left\lceil t^{\frac{1}{1+b}} \right\rceil}^{t} \|\mu_{s-1} - \mu\|^2_{V_{w_s}} \leq f\left(t^{1+b}\right) \left( K \ln(K) + 2K\ln(t) \right) = \cO(\ln(t)^2)$. Therefore, we obtain: for all $z \in \cZ$
\begin{align*}
	\sum_{s =n_{0}+1}^{t} \inf_{\lambda \in \neg_{\varepsilon} z}  \|\mu - \lambda \|^2_{V_{w_s}} & \geq  \sum_{s =n_{0}+1}^{t} \inf_{\lambda \in \neg_{\varepsilon} z}  \|\mu_{s-1} - \lambda \|^2_{V_{w_s}} - h'(t)
\end{align*}
where $h'(t) \eqdef  4ML_{\cK}   \sqrt{\left( K \ln(K) + 2K\ln(t) \right) f\left(t^{1+b}\right) t} + 8 M^2 L_{\cK}^2 t^{\frac{1}{1+b}} = \Theta\left( \ln(t) \sqrt{t} + t^{\frac{1}{1+b}} \right)$. Taking the maximum on both sides for the above result and using Lemma~\ref{lem:oracle_on_Z_is_not_too_good}, we obtain by chaining the above inequalities
\begin{align*}
\max_{z \in \cZ} \inf_{\lambda \in \neg_{\varepsilon} z}  \|\mu - \lambda \|^2_{V_{N_t}} &\geq \max_{z \in \cZ} \sum_{s =n_{0}+1}^{t}\inf_{\lambda \in \neg_{\varepsilon} z}  \|\mu_{s-1} - \lambda \|^2_{V_{w_s}} - C_{1} - h'(t)
\\
&\ge \sum_{s =n_{0}+1}^{t}  \inf_{\lambda \in \neg_\varepsilon \tilde{z}_s}\|\mu_{s-1} - \lambda \|^2_{V_{w_s}} - \left( N_{F}(t) + t^{\frac{1}{1+b}}  \right) 4KM^2L_{\cK}^2  - C_{1} - h'(t)
\\
&=  \sum_{s =n_{0}+1}^{t}  \|\mu_{s-1} - \lambda_{s} \|^2_{V_{w_s}} - \left( N_{F}(t) + t^{\frac{1}{1+b}}  \right) 4KM^2L_{\cK}^2  - C_{1} - h'(t)
\: ,
\end{align*}
where $N_{F}(t)$ is defined in \eqref{eq:N_F_definition}.

\paragraph{Step 3. From $\|\mu_{s-1} - \lambda_{s} \|^2_{V_{w_s}}$ to $g_s\left(w_s^{\cL^{\cK}}\right)$.} Using Lemma~\ref{lem:optimistic_gains_upper_bound}, we obtain: for all $s \in \left\llbracket n_{0}+1, t \right\rrbracket$ and $a \in \cK$,
\begin{align*}
	 \|\mu_{s-1} - \lambda_s\|^2_{a a\transpose} \geq  U_{s}^{a}  - c_{s-1}^{a} - 4ML_{\cK} \sqrt{c_{s-1}^{a}}
\end{align*}

Summing those inequalities together, we obtain:
\begin{align*}
	 &\sum_{s =n_{0}+1}^{t} \sum_{a \in \cK} w_{s}^{a} \|\mu_{s-1} - \lambda_s\|^2_{a a\transpose}
     \\
     &\geq  \sum_{s =n_{0}+1}^{t} \sum_{a \in \cK} w_{s}^{a} U_{s}^{a}  - \sum_{s =n_{0}+1}^{t} \sum_{a \in \cK} w_{s}^{a} c_{s-1}^{a} - 4ML_{\cK} \sum_{s =n_{0}+1}^{t} \sum_{a \in \cK} w_{s}^{a} \sqrt{c_{s-1}^{a}} \\
	 &\geq \sum_{s =n_{0}+1}^{t} \sum_{a \in \cK} w_{s}^{a} U_{s}^{a} - f\left(t^{1+b}\right) \left( K \ln(K) + 2K\ln(t) \right)  - 4ML_{\cK}  \sqrt{f\left(t^{1+b}\right)}  \left( K \ln(K) + \sqrt{8Kt}  \right) \\
	 &\geq \sum_{s =n_{0}+1}^{t} \left( 1 - \frac{1}{s} \right)\langle w_{s}^{\cL^{\cK}}, U_{s} \rangle - f\left(t^{1+b}\right) \left( K \ln(K) + 2K\ln(t) \right)  \\
	 &\quad \quad- 4ML_{\cK}  \sqrt{f\left(t^{1+b}\right)}  \left( K \ln(K) + \sqrt{8Kt}  \right) \\
	 &= \sum_{s =n_{0}+1}^{t} g_s\left(w_s^{\cL^{\cK}}\right) - f\left(t^{1+b}\right) \left( K \ln(K) + 2K\ln(t) \right)  - 4ML_{\cK}  \sqrt{f\left(t^{1+b}\right)}  \left( K \ln(K) + \sqrt{8Kt}  \right)
\end{align*}
where the second inequality is obtained by using Lemma~\ref{lem:upper_bound_on_weighted_cumulative_sum_slack} and the third since the optimistic gains are positive and the weights are defined as $w_s = \frac{1}{sK} \1_K + \left( 1 - \frac{1}{s} \right) w_s^{\cL^{\cK}}$ for all $s \in \left\llbracket n_0 + 1, t \right\rrbracket$. The last equality uses the definition of the gain $g_{s}(w) = \left( 1 - \frac{1}{s} \right)\langle w, U_{s}\rangle$. Therefore, since $\|\mu_{s-1} - \lambda_{s} \|^2_{V_{w_s}}  =  \sum_{a \in \cK} w_{s}^{a} \|\mu_{s-1} - \lambda_s\|^2_{a a\transpose}$, we obtain that:
\begin{align*}
	\max_{z \in \cZ} \inf_{\lambda \in \neg_{\varepsilon} z}  \|\mu - \lambda \|^2_{V_{N_t}} \geq \sum_{s =n_{0}+1}^{t}  g_s\left(w_s^{\cL^{\cK}}\right) - C_{1}  - \left( N_{F}(t) + t^{\frac{1}{1+b}}  \right) 4KM^2L_{\cK}^2 - h'(t) - h''(t)
\end{align*}
where $h''(t) =  f\left(t^{1+b}\right) \left( K \ln(K) + 2K\ln(t) \right)  + 4ML_{\cK}  \sqrt{f\left(t^{1+b}\right)}  \left( K \ln(K) + \sqrt{8Kt}  \right) = \Theta\left( \sqrt{t \ln(t)} \right)$.
\end{proof}

\begin{lemma} \label{lem:samp_rule_acti_inequality_under_concentration_and_no_stop_samp_complexity}
	Let $\cL^{\cZ}$ be the $\cZ$-oracle such that $\tilde{z}_{s} \in z_{F}(\mu_{s-1})$. Let $\cL^{\cK}$ be a learner with regret $R_{\cL^{\cK}}(t) = \cO(t^{\alpha_1})$ with $\alpha_1 \in (0,1)$ and fed with optimistic gain $g_{s}(w) = (1-\frac{1}{s})\langle w, U_s\rangle$ where $(U_s)_{s \in \left\llbracket n_{0}+1, t \right\rrbracket}$ defined in (\ref{eq:optimistic_gain_acti_definition}). Then, under $\cE_t$ if \hyperlink{algoLeBAI}{L$\varepsilon$BAI} does not stop at time $t+1$,
	\begin{align*}
		\sum_{s = n_{0}+1}^{t} g_{s}\left(w_s^{\cL^{\cK}}\right) \geq 2t T_{\varepsilon}(\mu)^{-1} -  r_2(t)
	\end{align*}
	where
	\begin{align*}
		&r_2(t) = R_{\cL^{\cK}}(t) + h'''(t) = \cO \left(  \ln(t)^2 t^{\frac{1}{1+b}}  +  t^{\max \{\alpha_1,\frac{1}{1+b}, 1-\frac{1}{1+b}\}}  \right)\\
		\text{with}\quad &h'''(t) = 2 \left( t^{\frac{1}{1+b}} + N_{F}(t)  + t^{1-\frac{1}{1+b}} \right) T_{\varepsilon}(\mu)^{-1}
	\end{align*}
\end{lemma}
\begin{proof}
\noindent\textbf{Step 4. No-regret property.} Using the no-regret property of the chosen online learner $\cL^{\cK}$ fed with gains $\left( g_s\right)_{s \geq n_{0}+1}$ and playing $\left( w_s^{\cL^{\cK}}\right)_{s \geq n_{0}+1}$, whose regret is denoted by $R_{\cL^{\cK}}(t) = \cO(t^{\alpha_1})$, we obtain directly
\begin{align*}
	\sum_{s =n_{0}+1}^{t}  g_s\left(w_s^{\cL^{\cK}}\right) \geq  \max_{w \in \simplex } \sum_{s =n_{0}+1}^{t} g_s(w) -  R_{\cL^{\cK}}(t)
\end{align*}

\paragraph{Step 5. From the optimal gain to $T_{\varepsilon}(\mu)^{-1}$.} Dropping the first positive terms and the ones for which the good event $A_s \eqdef \left\{ \lambda_{s} \in \neg_{\varepsilon} z_{F}(\mu) \lor \mu_{s-1} \in \neg_{\varepsilon} z_{F}(\mu) \lor \tilde{z}_s = z_{F}(\mu)\right\}$ doesn't hold, i.e. $\left\llbracket n_{0}+1, \left\lceil t^{\frac{1}{1+b}} \right\rceil -1\right\rrbracket \cup \left\{s \in \left\llbracket \left\lceil t^{\frac{1}{1+b}} \right\rceil, t \right\rrbracket: \neg A_s \right\} $, and using Lemma~\ref{lem:optimistic_gain_is_lower_bounded_by_sample_complexity}, we obtain: for all $w \in \simplex$,
\begin{align*}
	\sum_{s =n_{0}+1}^{t} g_s(w) \geq \sum_{s = \left\lceil t^{\frac{1}{1+b}} \right\rceil}^{t} \1_{(A_s)} g_s(w)
	&\geq \left( 1 - t^{-\frac{1}{1+b}} \right) \sum_{s = \left\lceil t^{\frac{1}{1+b}} \right\rceil}^{t} \1_{(A_s)} \langle w, U_s\rangle \\
	&\geq \left( 1 - t^{-\frac{1}{1+b}} \right) \sum_{s = \left\lceil t^{\frac{1}{1+b}} \right\rceil}^{t} \1_{(A_s)} \inf_{\lambda \in \neg_{\varepsilon} z_{F}(\mu)} \|\mu - \lambda\|^2_{V_{w}} \\
	&\geq \left( 1 - t^{-\frac{1}{1+b}} \right) \left(t - t^{\frac{1}{1+b}} - N_{F}(t) \right) \inf_{\lambda \in \neg_{\varepsilon} z_{F}(\mu)} \|\mu - \lambda\|^2_{V_{w}}
\end{align*}
where $\left| \left\{ s \in \left\llbracket \left\lceil t^{\frac{1}{1+b}} \right\rceil, t \right\rrbracket: \neg A_s \right\}\right| \leq N_{F}(t)$ since $\{\tilde z_s = z_{F}(\mu) \} \subseteq A_s$, where $N_{F}(t)$ defined in \eqref{eq:N_F_definition} satisfies $N_{F}(t) = \cO \left(  \ln(t)^2 t^{\frac{1}{1+b}} \right)$ (Lemma~\ref{lem:upper_bound_number_of_steps_away_from_zF}). For the last inequality, we also used that $- \left\lceil t^{\frac{1}{1+b}} \right\rceil + 1  \geq - t^{\frac{1}{1+b}}$. Taking the maximum over $\simplex$, we obtain:
\begin{align*}
	 \max_{w \in \simplex } \sum_{s =n_{0}+1}^{t} g_s(w) &\geq \left( 1 - t^{-\frac{1}{1+b}} \right)\left(t - t^{\frac{1}{1+b}} - N_{F}(t) \right)\max_{w \in \simplex } \inf_{\lambda \in \neg_{\varepsilon} z_{F}(\mu)} \|\mu - \lambda\|^2_{V_{w}}  \\
	 &= 2\left( 1 - t^{-\frac{1}{1+b}} \right) \left(t - t^{\frac{1}{1+b}} - N_{F}(t) \right) T_{\varepsilon}(\mu)^{-1} \\
	 &\geq 2\left(t - t^{\frac{1}{1+b}} - N_{F}(t) - t^{1-\frac{1}{1+b}} \right) T_{\varepsilon}(\mu)^{-1}
\end{align*}
where the equality is obtained by definition of $z_{F}(\mu)$. The last inequality is obtained by dropping some positive terms. Putting everything together, we obtain:
\begin{align*}
	\sum_{s =n_{0}+1}^{t}  g_s\left(w_s^{\cL^{\cK}}\right) \geq 2 t  T_{\varepsilon}(\mu)^{-1} -  R_{\cL^{\cK}}(t) - h'''(t)
\end{align*}
where $h'''(t) = 2 \left( t^{\frac{1}{1+b}} + N_{F}(t) + t^{1-\frac{1}{1+b}}\right) T_{\varepsilon}(\mu)^{-1} =\cO \left( \ln(t)^2 t^{\frac{1}{1+b}} +  t^{\max \{\frac{1}{1+b}, 1-\frac{1}{1+b}\}}  \right)$.
\end{proof}

Chaining the two inequalities in Lemma~\ref{lem:samp_rule_acti_inequality_under_concentration_and_no_stop_cumulative_gain} and Lemma~\ref{lem:samp_rule_acti_inequality_under_concentration_and_no_stop_samp_complexity}, we obtain directly Lemma~\ref{lem:LeBAI_is_good_sampling_rule}. Therefore, the sampling rule in \hyperlink{algoLeBAI}{L$\varepsilon$BAI} satisfies the desired condition (\ref{eq:desired_condition_sampling_rule}) for $\beta_2=1$ and $\beta_1 = 1-\max \left\{ \frac{1}{2}, \frac{1}{1+b}, 1-\frac{1}{1+b}, \alpha_1\right\}$.

\begin{lemma} \label{lem:LeBAI_is_good_sampling_rule}
	Let $\cL^{\cZ}$ be the $\cZ$-oracle such that $\tilde{z}_{s} \in z_{F}(\mu_{s-1})$. Let $\cL^{\cK}$ be a learner with regret $R_{\cL^{\cK}}(t) = \cO(t^{\alpha_1})$ with $\alpha_1 \in (0,1)$ fed with optimistic gain $g_{s}(w) = (1-\frac{1}{s}) \langle w, U_s\rangle$ where $(U_s)_{s \in \left\llbracket n_{0}+1, t \right\rrbracket}$ defined in (\ref{eq:optimistic_gain_acti_definition}). Then, under $\cE_t$ if the algorithms does not stop at time $t+1$,
\begin{align*}
\max_{z \in \cZ} \inf_{\lambda \in \neg_{\varepsilon} z}  \|\mu - \lambda \|^2_{V_{N_t}} \geq  2t T_{\varepsilon}(\mu)^{-1} - r(t)
\end{align*}
where $r(t) = r_1(t) + r_2(t)= \tilde{\cO} \left(t^{\max \{\frac{1}{2},  \frac{1}{1+b}, 1-\frac{1}{1+b}, \alpha_1\}}  \right)$.
\end{lemma}

\subsection{Summarized Proof} \label{app:subsection_summarized_proof}

Combining the lemmas obtained in Appendix~\ref{app:subsection_modular_proof_scheme},~\ref{app:subsection_stopping_recommendation_pairs} and~\ref{app:subsection_sampling_rules}, we obtain directly Theorem~\ref{thm:asymptotic_optimality_algorithm}. Therefore, our algorithm \hyperlink{algoLeBAI}{L$\varepsilon$BAI} provides asymptotically optimal algorithms if instantiated properly.

\begin{theorem*}[Theorem~\ref{thm:asymptotic_optimality_algorithm}]
Let $\cL^{\cK}$ be a learner with sub-linear regret and $\cL^{\cZ}$ be the $\cZ$-oracle such that $\tilde{z}_{s} \in z_{F}(\mu_{s-1})$. Let the recommendation rule be an instantaneous furthest answer $z_t \in z_{F}(\mu_{t-1}, N_{t-1})$ and the stopping rule given by (\ref{eq:definition_stopping_criterion}), both being either evaluated at each time $t$ or with the sticky or lazy update schedule as described in Appendix~\ref{app:subsubsection_other_update_schemes}. Let the stopping threshold $\beta(t,\delta)$ as in (\ref{eq:definition_stopping_threshold}) and the exploration bonus $f(t) \eqdef 2 \beta \left(t, t^{1/\alpha}\right)$. Then, \hyperlink{algoLeBAI}{L$\varepsilon$BAI} yields an $(\varepsilon, \delta)$-PAC algorithm and, for all $\mu \in \cM$ such that $|z_{F}(\mu)|=1$,
\begin{align*}
	\limsup_{\delta \rightarrow 0} \frac{\expectedvalue_{\mu} \left[ \taud \right]}{\ln \left( \frac{1}{\delta}\right)} \leq T_{\varepsilon}(\mu) \: .
\end{align*}
\end{theorem*}
\begin{proof}
First, let's prove the result for the stopping-recommendation pair with an update at each time $t$. Combining Lemma~\ref{lem:reco_stop_pair_instantfurthest_eachtime_inequality_under_concentration_and_no_stop} and Lemma~\ref{lem:LeBAI_is_good_sampling_rule}, we obtain:
\begin{align*}
	2 \beta(t, \delta) &\geq \max_{z \in \cZ} \inf_{\lambda \in \neg_{\varepsilon} z}  \|\mu - \lambda \|^2_{V_{N_t}} - h_{\delta}(t) \\
	&\geq 2t T_{\varepsilon}(\mu)^{-1} - r(t) - h_{\delta}(t)
\end{align*}
where $h_{\delta}(t) =  \sqrt{8 f(t) \beta(t, \delta) } + 4 f(t) = \Theta\left( \ln(t) + \sqrt{\ln \left( \frac{1}{\delta}\right)}\right)$ and $r(t) = r_1(t) + r_2(t)= \tilde{\cO} \left(t^{\max \{\frac{1}{2}, \frac{1}{1+b}, 1-\frac{1}{1+b}, \alpha_1\}}  \right)$. Therefore, our identification strategy is sample-efficient, i.e. it verifies (\ref{eq:definition_sample_efficient_identification_strategy}) with
\begin{align*}
	H_{\mu}\left(t, \delta\right) \geq t T_{\varepsilon}(\mu)^{-1} \quad \quad \text{where  } H_{\mu}\left(t, \delta\right) = \beta(t, \delta) + \frac{1}{2}\left(r(t) + h_{\delta}(t)\right) = \cO \left( t^{1-\beta_1} + \ln(1/\delta)^{1-\beta_2}\right)
\end{align*}
where $\beta_1 = 1-\max \left\{ \frac{1}{2}, \frac{1}{1+b}, 1-\frac{1}{1+b}, \alpha_1\right\} \in (0,1]$ and $\beta_2=\frac{1}{2} \in (0,1]$.

Combining Lemma~\ref{lem:delta_PAC_recommendation_stopping_pair} and Lemma~\ref{lem:asymptotic_optmiality_when_sample_efficient_identification_startegy}, we obtain directly that our \hyperlink{algoLeBAI}{L$\varepsilon$BAI} yields a $(\varepsilon, \delta)$-PAC and asymptotically optimal algorithm.

When considering the computational relaxations of the update schemes called sticky or lazy, the proof is identical as explained in Appendix~\ref{app:subsubsection_other_update_schemes} and the fact that Lemma~\ref{lem:delta_PAC_recommendation_stopping_pair} can also be used.
\end{proof}

\paragraph{Values of Parameters}
In the proof above, the condition (\ref{eq:definition_sample_efficient_identification_strategy}) holds for $\beta_1 = 1-\max \left\{ \frac{1}{2}, \frac{1}{1+b}, 1-\frac{1}{1+b}, \alpha_1\right\}$ and $\beta_2=\frac{1}{2}$. For AdaHedge \citep{derooij_2013_FollowLeaderIf}, Lemma~\ref{lem:theorem_8_derooij_2013_FollowLeaderIf} shows that $\alpha_1 = \frac{1}{2}$. An optimal choice of parameter suggests that $\frac{1}{1+b}= 1-\frac{1}{1+b}$. Therefore, we choose $b=1$. Higher $b$ would imply more conservative optimistic gains. Therefore, we have shown that the condition (\ref{eq:definition_sample_efficient_identification_strategy}) holds for $\beta_1 =\beta_2=\frac{1}{2}$.
As done in \citet{degenne_2020_GamificationPureExploration}, we can choose $\alpha=3$ in the definition of the exploration bonus $f(t)= 2 \beta \left(t, t^{1/\alpha}\right)$. The higher $\alpha$, the higher the exploration bonus is.

\subsection{Technical Arguments} \label{app:subsection_technical_arguments}

In this appendix, we list and prove the technical arguments used in the core of the proof of Theorem~\ref{thm:asymptotic_optimality_algorithm}. For clarity, we distinguish between the lemmas extracted from the literature (Appendix~\ref{app:subsubsection_lemmas_from_literature}), the key new lemmas (Appendix~\ref{app:subsubsection_key_lemmas}) and the technical ones allowing to upper bound the considered cumulative sums (Appendix~\ref{app:subsubsection_upper_bounding_cumulative_sums}).

\subsubsection{Lemmas From the Literature}  \label{app:subsubsection_lemmas_from_literature}

We recall the lemmas extracted from the literature, while omitting the proofs for the sake of space. Lemma~\ref{lem:theorem_6_degenne_2020_StructureAdaptiveAlgorithms} is a powerful result allowing to upper and lower bound the difference between the empirical allocation over arms and the cumulative sum of the played proportions when using tracking. Lemma~\ref{lem:lemma_28_degenne_2020_GamificationPureExploration} and Lemma~\ref{lem:lemma_8_degenne_2019_NonAsymptoticPureExploration} are technical results on manipulation of inequalities and cumulative sums. Lemma~\ref{lem:theorem_4_degenne_2019_PureExplorationMultiple} shows that $\mu \mapsto z_{F}(\mu)$ is upper hemicontinuous on $\cM$. Upper hemicontinuity is defined in Definition~\ref{def:upper_hemicontinuity}. Lemma~\ref{lem:theorem_8_derooij_2013_FollowLeaderIf} gives an upper bound on the regret incurred by AdaHedge.

\begin{lemma}[Theorem 6 in \citet{degenne_2020_StructureAdaptiveAlgorithms}] \label{lem:theorem_6_degenne_2020_StructureAdaptiveAlgorithms}
The tracking procedure, which draws $a_t \in \argmin_{a \in \cK} N_{t-1}^{a} - W_t^{a}$ where $W_t = W_{t-1} + w_t$, ensures that for all $t \in \Natural$, for all $a \in \cK$,
\begin{align*}
	- \sum_{j=2}^{K} \frac{1}{j} \leq N_{t}^{a} - W_{t}^{a} \leq 1
\end{align*}
\end{lemma}
In practice, we will use the following slightly coarser lower bound $-\ln(K) \leq N_{t}^{a} - W_{t}^{a}$.

\begin{lemma}[Lemma 28 in \citet{degenne_2020_GamificationPureExploration}] \label{lem:lemma_28_degenne_2020_GamificationPureExploration}
For all $\alpha, y >0$, if for some $x \geq 0$, it holds that $y \geq x - \alpha \sqrt{x}$, then $x \leq y + \alpha \sqrt{y} + \alpha^2$~.
\end{lemma}

\begin{lemma}[Lemma 8 in \citet{degenne_2019_NonAsymptoticPureExploration}] \label{lem:lemma_8_degenne_2019_NonAsymptoticPureExploration}
For $t \geq t_{0} \geq 1$ and $(x_s)_{s \in [t]}$ non-negative real numbers such that $\sum_{s=1}^{t_0-1} x_s >0$,
\begin{align*}
&\sum_{s = t_{0}}^{t} \frac{x_s}{\sum_{r =1}^{s} x_r} \leq \ln\left( \sum_{s=1}^{t} x_s\right) - \ln\left( \sum_{s=1}^{t_0-1} x_s\right)  \\
&\sum_{s = t_{0}}^{t} \frac{x_s}{\sqrt{\sum_{r =1}^{s} x_r}} \leq 2 \sqrt{\sum_{s=1}^{t} x_s} - 2 \sqrt{\sum_{s=1}^{t_0-1} x_s}
\end{align*}
\end{lemma}

\begin{definition}[Upper hemicontinuity] \label{def:upper_hemicontinuity}
For a set $B$, let $\bbS(B) = 2^{B} \setminus \{\emptyset\}$ be the set of all non-empty subsets of $B$. A set-valued function $\Gamma: A \mapsto \bbS(B)$ is upper hemicontinuous at $a \in A$ if for any open neighborhood $V$ of $\Gamma(a)$ there exists a neighborhood $U$ of $a$ such that for all $x \in U$, $\Gamma(x)$ is a subset of $V$.
\end{definition}

\begin{lemma}[Theorem 4 in \citet{degenne_2019_PureExplorationMultiple}] \label{lem:theorem_4_degenne_2019_PureExplorationMultiple}
	The function $\mu \mapsto z_{F}(\mu)$ is upper hemicontinuous on $\Real^d$ with non-empty and compact values.
\end{lemma}
\begin{proof}
	Theorem 4 in \citet{degenne_2019_PureExplorationMultiple} proves that the function $\mu \mapsto z_{F}(\mu)$ is upper hemicontinuous on $\cM$ with non-empty and compact values. In their paper the function $\mu \mapsto z_{F}(\mu)$ was only defined on $\cM$. However, a careful analysis of their Appendix D reveals that their result and proof apply without any change when $\mu \mapsto z_{F}(\mu)$ is defined on $\Real^d$. Since in our setting $z_{F}(\mu)$ is defined on $\Real^d$, this concludes the proof. Note that while $\mu \mapsto z_{F}(\mu)$ is defined on $\cM$, the alternatives $(\neg_{\varepsilon}z)_{z \in \cZ}$ are all subsets of $\overline{\cM}$.
\end{proof}

\begin{lemma}\label{lem:theorem_8_derooij_2013_FollowLeaderIf}
	On the online learning problem with $K$ arms and gains $g_{s}(w) = (1-\frac{1}{s}) \langle w, U_{s}\rangle$ for $s \in \left\llbracket n_{0}+1, t \right\rrbracket$, AdaHedge predicting $\left(w_s^{\cL^{\cK}}\right)_{s \in \left\llbracket n_{0}+1, t \right\rrbracket}$ has regret
	\begin{align*}
		R_{\cL^{\cK}}(t) &= \max_{w \in \simplex} \sum_{s = n_{0}+1}^{t} g_{s}(w) - g_s\left(w_s^{\cL^{\cK}}\right) \leq  \left(\sqrt{t \ln(K)} + \frac{4}{3} \ln(K) + 2 \right)  \sigma = \cO \left(\sqrt{t}\right)
	\end{align*}
	where $\sigma = \max_{s \in \left\llbracket n_0 +, t \right\rrbracket} (1-\frac{1}{s}) \left(\max_{a \in \cK} U_{s}^a - \min_{a \in \cK} U_{s}^a\right)$.
\end{lemma}
\begin{proof}
For scaled losses in $[0,1]$, Theorem 6 in \citet{derooij_2013_FollowLeaderIf} yields that AdaHedge's cumulative regret $R_{t} $ satisfies: $R_{t} \leq 2 \sqrt{V_{t}\ln(K)} + \frac{4}{3} \ln(K) + 2$ where $V_{t} =\sum_{s \in [t]} v_s$ with $v_s = \sum_{a \in \cK} w_{s,a}^{\cL^{\cK}} (l_{s}^{a} - \langle w_{s}^{\cL^{\cK}},l_{s} \rangle)^2$.

Rewriting our gains as losses, we have $l_s^a = (1-\frac{1}{s}) \frac{\max_{b \in \cK}U_{s}^b - U_{s}^a}{b_s}$ where $b_s = (1-\frac{1}{s})\left(\max_{a \in \cK} U_{s}^a - \min_{a \in \cK} U_{s}^a\right)$ is the scale of the loss at time $s$. We have $v_s \leq \|l_{s} - \langle w_{s}^{\cL^{\cK}},l_{s} \rangle\|_{\infty}^2  = (1-\frac{1}{s})^2 \frac{\|\langle w_s^{\cL^{\cK}}, U_s\rangle - U_{s}\|_{\infty}^2}{b_s^{2}} \leq \frac{b_s^2}{\sigma^2}$ where $\sigma = \max_{s \in \left\llbracket n_0 +, t \right\rrbracket} b_s$. The upper bound on $R_{t}$ rewrites as: $R_{t}\leq \frac{1}{\sigma} \sqrt{\sum_{s \leq t} b_s^2 \ln(K)} + \frac{4}{3} \ln(K) + 2$. Theorem 16 in \citet{derooij_2013_FollowLeaderIf} yields that $R_{\cL^{\cK}}(t) = \sigma R_{t}$. Therefore, we conclude that:
\begin{align*}
    R_{\cL^{\cK}}(t) &\leq \sqrt{\sum_{s \in \left\llbracket n_0 +, t \right\rrbracket} b_s^2 \ln(K)} +\left(\frac{4}{3} \ln(K) + 2\right)  \max_{s \in \left\llbracket n_0 +, t \right\rrbracket} b_{s} \leq  \left( \sqrt{t\ln(K)} +\left(\frac{4}{3} \ln(K) + 2\right) \right) \max_{s \in \left\llbracket n_0 +, t \right\rrbracket} b_{s}
\end{align*}

Using the definition of $U_s^a$ and Lemma~\ref{lem:optimistic_gains_upper_bound}, we obtain:
\begin{align*}
	b_s \leq \max_{a \in \cK} U_s^a \leq \max_{a \in \cK} \left( \|\mu_{s-1} - \lambda_s\|^2_{a a\transpose} + c_{s-1}^{a} + 4 ML_{\cK} \sqrt{c_{s-1}^{a}}\right) \leq 16 M^2 L_{\cK}^2
\end{align*}
where the last inequality is obtained by the boundedness assumption and the definition of $c_s^a$.
\end{proof}

We also recall a trivial result obtained by comparison of sum and integrals: for all $\alpha \in (0,1)$ and all $t \in \Natural^{\star}$, $\sum_{s \in [t]} \frac{1}{s^{\alpha}} = \Theta(t^{1-\alpha})$.

\subsubsection{Key Lemmas}  \label{app:subsubsection_key_lemmas}

The following lemmas are key arguments used throughout Appendix~\ref{app:section_proof_thm_asymptotic_optimality_algorithm}. Lemma~\ref{lem:switch_alternative_set} is a simple, yet powerful, result allowing to change the considered alternatives. This subtlety arises only because we are considering multiple correct answers.

\begin{lemma} \label{lem:switch_alternative_set}
	For all $w \in \Real_{+}^{K}$, $\mu \in \cM$ and $\tilde z \in \cZ$,
	\begin{align*}
		\max_{z \in  \cZ} \inf_{\lambda \in \neg_{\varepsilon} z} \|\mu - \lambda\|^2_{V_{w}} \geq \inf_{\lambda \in \neg_{\varepsilon} \tilde z}  \|\mu - \lambda\|^2_{V_{w}} \: .
	\end{align*}
	In particular, for all $(w, \tilde{w}) \in (\Real_{+}^{K})^2$ and $(\mu, \tilde{\mu}) \in \cM^2$,
	\begin{align*}
		\max_{z \in \cZ} \inf_{\lambda \in \neg_{\varepsilon} z} \|\mu - \lambda\|^2_{V_{w}} \geq \max_{\tilde z \in z_{F}(\tilde{\mu}, \tilde{w})} \inf_{\lambda \in \neg_{\varepsilon} \tilde z}  \|\mu - \lambda\|^2_{V_{w}} \: .
	\end{align*}
\end{lemma}
\begin{proof}
\textbf{Case 1:} $\tilde z \notin \cZ_{\varepsilon}(\mu)$, hence $\mu \in \neg_{\varepsilon} \tilde z$. Therefore $\inf_{\lambda \in \neg_{\varepsilon} \tilde z}  \|\mu - \lambda\|^2_{V_{w}} = 0$. \\
\noindent\textbf{Case 2:} $\tilde z \in \cZ_{\varepsilon}(\mu)$.  The result is direct by taking the maximum.
\end{proof}

Lemma~\ref{lem:optimistic_upper_bound_on_the_coordinate_deviation_mle_true} shows that the slacks $(c_{s}^{a})_{s \in \left\llbracket \lceil t^{\frac{1}{1+b}} \rceil, t-1 \right\rrbracket,a \in \cK}$ are an upper bound of the unknown distance $\| \mu - \mu_{s}\|^2_{a a \transpose}$.

\begin{lemma} \label{lem:optimistic_upper_bound_on_the_coordinate_deviation_mle_true}
Let $(c_{s}^{a})_{s \in \left\llbracket n_{0}, t-1\right\rrbracket,a \in \cK}$ in (\ref{eq:optimistic_slack_definition}). Under $\cE_t$, for all $(s,a) \in \left\llbracket \lceil t^{\frac{1}{1+b}} \rceil, t-1 \right\rrbracket \times \cK$, $\| \mu - \mu_s\|^2_{a a \transpose} \leq c_{s}^{a}$.
\end{lemma}
\begin{proof}
Since $(\mu, \mu_s) \in \cM^2$, the boundedness yields: $\| \mu - \mu_s\|^2_{a a \transpose} = \langle  \mu - \mu_s, a \rangle^2 \leq \| \mu - \mu_s \|_{2}^2 \|a\|_{2}^2 \leq 4 M^2 L_{\cK}^2$.

By definition, under $\cE_t$, for all $s\in [t]$, we have $\|\mu_{s} - \mu \|^2_{V_{N_{s}}} \leq f(t)$. Since $x \mapsto f(x)$ is increasing, for all $s \geq \lceil t^{\frac{1}{1+b}} \rceil$, we have $f(t) \leq f\left(s^{1+b}\right)$. By initialization $V_{N_{s}}$ is positive definite, using Cauchy-Schwartz yields:
\begin{align*}
	\langle  \mu - \mu_s, a \rangle^2 \leq  \|\mu_{s} - \mu \|^2_{V_{N_{s}}}  \|a\|^2_{V_{N_{s}}^{-1}} \leq  f\left(s^{1+b}\right)  \|a\|^2_{V_{N_{s}}^{-1}}
\end{align*}

Combining both upper bound and using the definition of $(c_{s}^{a})_{s \in \left\llbracket n_{0}, t-1\right\rrbracket,a \in \cK}$ in (\ref{eq:optimistic_slack_definition}) conclude the proof.
\end{proof}

Lemma~\ref{lem:optimistic_gains_upper_bound} shows useful upper and lower bounds on the optimistic gains $(U_s^a)_{s \in \left\llbracket n_{0} + 1, t \right\rrbracket,a \in \cK}$.

\begin{lemma} \label{lem:optimistic_gains_upper_bound}
Let $b > 0$, $(c_{s-1}^{a})_{s \in \left\llbracket n_{0} + 1, t  \right\rrbracket,a \in \cK}$ and $(U_s^{a})_{s \in \left\llbracket n_{0} + 1, t \right\rrbracket,a \in \cK}$ defined in (\ref{eq:optimistic_slack_definition}-\ref{eq:optimistic_gain_acti_definition}). For all $s \in \left\llbracket n_{0}+1, t \right\rrbracket$, for all $a \in \cK$,
\begin{align*}
	 \|\mu_{s-1} - \lambda_s\|^2_{a a\transpose} \leq U_{s}^{a} \leq \|\mu_{s-1} - \lambda_s\|^2_{a a\transpose} + c_{s-1}^{a} + 4 ML_{\cK} \sqrt{c_{s-1}^{a}} \: .
\end{align*}
For all $s \in \left\llbracket \left\lceil t^{\frac{1}{1+b}} \right\rceil, t \right\rrbracket$, for all $a \in \cK$, $U_s^a \geq \|\mu_{s-1} - \mu\|^2_{a a\transpose}$ and $U_s^a \geq \|\mu - \lambda_s\|^2_{a a\transpose}$.
\end{lemma}
\begin{proof}
By definition of $U_{s}^{a}$ in (\ref{eq:optimistic_gain_acti_definition}) and using the boundedness to show $\|\mu_{s-1} - \lambda_s\|_{a a\transpose} \leq 2ML_{\cK}$, we have: for all $s \in \left\llbracket n_{0}+1, t \right\rrbracket$, for all $a \in \cK$,
\begin{align*}
	U_s^a &= \|\mu_{s-1} - \lambda_s\|^2_{a a\transpose} + c_{s-1}^{a} + 2\|\mu_{s-1} - \lambda_s\|_{a a\transpose} \sqrt{c_{s-1}^{a}} \geq \|\mu_{s-1} - \lambda_s\|^2_{a a\transpose} \\
	U_s^a & \leq \|\mu_{s-1} - \lambda_s\|^2_{a a\transpose} + c_{s-1}^{a} + 4 ML_{\cK} \sqrt{c_{s-1}^{a}}
\end{align*}

By definition of $U_{s}^{a}$ in (\ref{eq:optimistic_gain_acti_definition}), Lemma~\ref{lem:optimistic_upper_bound_on_the_coordinate_deviation_mle_true} and triangular inequality, we have: for all $s \in \left\llbracket \left\lceil t^{\frac{1}{1+b}} \right\rceil, t \right\rrbracket$, for all $a \in \cK$,
\begin{align*}
	U_s^a &= \left(\|\mu_{s-1} - \lambda_s\|_{a a\transpose} + \sqrt{c_{s-1}^{a}}\right)^2 \geq \left(\|\mu_{s-1} - \lambda_s\|_{a a\transpose} + \|\mu_{s-1} - \mu\|_{a a\transpose} \right)^2 \geq \|\mu - \lambda_s\|^2_{a a\transpose} \\
	U_s^a &\geq c_{s-1}^{a} \geq \|\mu_{s-1} - \mu\|^2_{a a\transpose}
\end{align*}
\end{proof}

Lemma~\ref{lem:optimistic_gain_is_lower_bounded_by_sample_complexity} shows that under a good event $A_s$ the gain $\langle w , U_s \rangle$ is an upper bound on the unknown $\inf_{\lambda \in \neg_{\varepsilon} z_{F}(\mu)} \|\mu - \lambda\|^2_{V_{w}}$ for all $w \in \simplex$. This good event essentially boils down to having as candidate answer the unknown furthest answer $\tilde{z}_s = z_{F}(\mu)$.

\begin{lemma} \label{lem:optimistic_gain_is_lower_bounded_by_sample_complexity}
Let $b > 0$ as in Lemma~\ref{lem:optimistic_gains_upper_bound}. Let $(A_s)_{s \in \left\llbracket \left\lceil t^{\frac{1}{1+b}} \right\rceil, t \right\rrbracket}$ defined as: for all $s \in \left\llbracket \left\lceil t^{\frac{1}{1+b}} \right\rceil, t \right\rrbracket$,
\begin{equation} \label{eq:good_event_to_lower_bound_instantaneous_gain}
	A_s \eqdef \left\{ \lambda_{s} \in \neg_{\varepsilon} z_{F}(\mu) \lor \mu_{s-1} \in \neg_{\varepsilon} z_{F}(\mu) \lor \tilde{z}_s = z_{F}(\mu)\right\}
\end{equation}
Under $\cE_t$, we have for all $w \in \simplex$ and $s \in \left\llbracket \left\lceil t^{\frac{1}{1+b}} \right\rceil, t \right\rrbracket$, under $A_s$ it holds that
\begin{align*}
	\langle w , U_s \rangle \geq \inf_{\lambda \in \neg_{\varepsilon} z_{F}(\mu)} \|\mu - \lambda\|^2_{V_{w}}
\end{align*}
\end{lemma}
\begin{proof}
Summing the lower bound of Lemma~\ref{lem:optimistic_gains_upper_bound} yields: for all $s \in \left\llbracket \left\lceil t^{\frac{1}{1+b}} \right\rceil, t  \right\rrbracket$ and all $w \in \simplex$, $\langle w , U_s \rangle \geq \|\mu - \lambda_s\|^2_{V_{w}}$ and $\langle w , U_s \rangle \geq \|\mu - \mu_{s-1}\|^2_{V_{w}}$.

\paragraph{Case 1: $\lambda_s \in \neg_{\varepsilon} z_{F}(\mu)$.} Taking the infimum on the first inequality, we have: for all $s \in \left\llbracket \left\lceil t^{\frac{1}{1+b}} \right\rceil, t  \right\rrbracket$ and all $w \in \simplex$,
$$\langle w , U_s \rangle \geq \inf_{\lambda \in \neg_{\varepsilon} z_{F}(\mu)} \|\mu - \lambda\|^2_{V_{w}}$$

\paragraph{Case 2: $\mu_{s-1} \in \neg_{\varepsilon} z_{F}(\mu)$.} Taking the infimum on the second inequality, we obtain directly: for all $w \in \simplex$ and $s \geq t^{\frac{1}{1+b}}$,
$$\langle w , U_s \rangle \geq \inf_{\lambda \in \neg_{\varepsilon} z_{F}(\mu)} \|\mu - \lambda\|^2_{V_{w}}$$

\paragraph{Case 3: $\tilde{z}_s = z_{F}(\mu)$.} By definition $\lambda_s \in \neg_{\varepsilon} \tilde{z}_s$, hence $\lambda_s \in \neg_{\varepsilon} z_{F}(\mu)$. Therefore, we can use the first case.
\end{proof}

We introduce a strictly positive geometrical quantity $\Delta_{\min}^2$ (\ref{eq:definition_geometric_lower_bound_distance_alternatives}) in Lemma~\ref{lem:strictly_positive_geometric_constant}. This result crucially uses that $\cZ$ is finite and $|z_{F}(\mu)|=1$ for the unknown parameter $\mu$. It means that we can't be arbitrarily close to $\mu$ while having a unique furthest answer which is different from $z_{F}(\mu)$.

$\Delta_{\min}^2$ represents the distance of $\mu$ to the \textit{furthest alternative} to $z_{F}(\mu)$, denoted by $\neg_{F} z_{F}(\mu)$. The furthest alternative is fundamentally different from the alternative $\neg_{\varepsilon} z_{F}(\mu)$ (which is considered in the rest of the paper). While the furthest alternative corresponds to an identification problem with a unique correct answer, the alternative corresponds to an identification problem with multiple correct answers. Therefore, we see that to achieve asymptotic optimality in this multiple correct answer setting, we need to solve the harder problem of exactly identifying the unique $z_{F}(\mu)$.

\begin{lemma} \label{lem:strictly_positive_geometric_constant}
Defining
\begin{equation} \label{eq:definition_geometric_lower_bound_distance_alternatives}
	\Delta_{\min}^2 \eqdef \inf_{\lambda \in \neg_{F} z_{F}(\mu)} \max_{a \in \cK} \|\mu - \lambda\|_{a a \transpose}^2 > 0
\end{equation}
where $\neg_{F} z_{F}(\mu) \eqdef \left\{ \lambda \in \Real^d: z_{F}(\mu) \neq z_{F}(\lambda)\right\}$. Then, we have: $\Delta_{\min}^2 > 0$.
\end{lemma}
\begin{proof}
	By Lemma~\ref{lem:theorem_4_degenne_2019_PureExplorationMultiple}, we know that $\lambda \mapsto z_{F}(\lambda)$ is upper hemicontinuous on $\Real^d$ with non-empty and compact values. Since $\cZ$ is finite and $z_{F}(\mu)$ is a singleton by assumption on $\mu$, the upper hemicontinuity implies that there exists an open a neighborhood $U$ of $\mu$ such that for all $\tilde{\mu} \in U$, $z_F(\tilde{\mu}) = z_{F}(\mu)$. This is obtained directly by using Definition~\ref{def:upper_hemicontinuity} and taking an open neighborhood $V$ of $z_{F}(\mu)$ such that $V \cap \cZ = z_{F}(\mu)$, which exists since $|z_{F}(\mu)|=1$ and $\cZ$ finite. Noticing that $U \subseteq (\neg_{F} z_{F}(\mu))^{\complement}$ and using that $U$ is an open neighborhood of $\mu$, we can conclude that $\Delta_{\min}^2 > 0$.
\end{proof}

Lemma~\ref{lem:upper_bound_number_of_steps_away_from_zF} shows that the good event $\{\tilde{z}_s = z_{F}(\mu)\} \subseteq A_s$ doesn't happen only for a sub-linear number of times. A similar argument was shown in Lemma 25 of \citet{degenne_2020_GamificationPureExploration}. This crucially uses the fact that $\Delta_{\min}^2 > 0$.

\begin{lemma} \label{lem:upper_bound_number_of_steps_away_from_zF}
	 Let $\cL^{\cZ}$ be the $\cZ$-oracle such that $\tilde{z}_{s} \in z_{F}(\mu_{s-1})$.
	 For all $s \in \left\llbracket \left\lceil t^{\frac{1}{1+b}} \right\rceil, t \right\rrbracket$, let
	\begin{equation} \label{eq:N_F_definition}
		B_s \eqdef \{\tilde{z}_s = z_{F}(\mu)\} \quad \text{and} \quad N_{F}(t) \eqdef \left| \left\{ s \in \left\llbracket \left\lceil t^{\frac{1}{1+b}} \right\rceil, t \right\rrbracket: \neg B_s \right\}\right| \: .
	\end{equation}
	 Then, under $\cE_t$,
	\begin{align*}
		N_{F}(t) \leq  \frac{K}{\Delta_{\min}^2} t^{\frac{1}{1+b}} f\left(t^{1+b}\right) \left( K \ln(K) + 2K\ln(t) \right) = \cO \left( \ln(t)^2 t^{\frac{1}{1+b}} \right)
	\end{align*}
\end{lemma}
\begin{proof}
Using Lemma~\ref{lem:small_cumulative_reweighted_deviation_mle_mu}, under $\cE_t$,
	\begin{align*}
		f\left(t^{1+b}\right) \left( K \ln(K) + 2K\ln(t) \right) \geq \sum_{s = \left\lceil t^{\frac{1}{1+b}} \right\rceil}^{t} \|\mu_{s-1} - \mu\|^2_{V_{w_s}} &\geq \sum_{s = \left\lceil t^{\frac{1}{1+b}} \right\rceil}^{t} \1_{(\neg B_s)}\|\mu_{s-1} - \mu\|^2_{V_{w_s}} \\
		&\geq \sum_{s = \left\lceil t^{\frac{1}{1+b}} \right\rceil}^{t} \1_{(\neg B_s)} \frac{1}{sK}  \sum_{a \in \cK}\|\mu_{s-1} - \mu\|^2_{a a\transpose}
	\end{align*}
where the last inequality is obtained since $w_s = \frac{1}{sK} \1_K + \left( 1 - \frac{1}{s} \right) w_s^{\cL^{\cK}}$ for all $s \in \left\llbracket n_0 + 1, t \right\rrbracket$.

Since $B_s = \{\tilde{z}_s = z_{F}(\mu)\}$ and $\tilde{z}_{s} \in z_{F}(\mu_{s-1})$, we obtain:
\begin{align*}
	\neg B_s  = \left\{\tilde{z}_s \notin z_{F}(\mu) , \tilde{z}_s \in z_{F}(\mu_{s-1})\right\} \subseteq  \left\{z_{F}(\mu_{s-1}) \neq z_{F}(\mu) \right\}
\end{align*}
where in the above $z_{F}(\mu)$ is viewed as the set containing one element. Therefore, since $\mu_{s-1} \in \Real^d$ and $z_{F}(\mu_{s-1}) \neq z_{F}(\mu)$, we have shown that $\mu_{s-1} \in\neg_{F} z_{F}(\mu)$. By the definition of $\Delta_{\min}^2$ in (\ref{eq:definition_geometric_lower_bound_distance_alternatives}), there exists $a_s \in \cK$ such that $\|\mu_{s-1} - \mu\|^2_{a_s a_s \transpose} \geq \Delta_{\min}^2$. This yields $\sum_{a \in \cK}\|\mu_{s-1} - \mu\|^2_{a a\transpose} \geq \Delta_{\min}^2$. Therefore, we obtain: $\frac{1}{s}  \sum_{a \in \cK}\|\mu_{s-1} - \mu\|^2_{a a\transpose} \geq \frac{\Delta_{\min}^2}{s}\geq \frac{\Delta_{\min}^2}{t^{\frac{1}{1+b}}}$ for all $s \in \left\llbracket \left\lceil t^{\frac{1}{1+b}} \right\rceil, t \right\rrbracket$. Putting everything together, we have shown that:
\begin{align*}
	N_{F}(t) \leq \frac{K}{\Delta_{\min}^2} t^{\frac{1}{1+b}} f\left(t^{1+b}\right) \left( K \ln(K) + 2K\ln(t) \right)
\end{align*}
\end{proof}

\begin{lemma} \label{lem:oracle_on_Z_is_not_too_good}
	 Let $\cL^{\cZ}$ be the $\cZ$-oracle such that $\tilde{z}_{s} \in z_{F}(\mu_{s-1})$.
	 Then, under $\cE_t$,
	\[
		\min_{z \in \cZ} \sum_{s = n_{0} + 1}^{t} [f_s(\tilde{z}_s)- f_s(z)] \leq \left( N_{F}(t) + t^{\frac{1}{1+b}}  \right) 4KM^2L_{\cK}^2 = \cO \left( \ln(t)^2 t^{\frac{1}{1+b}} \right)   \: ,
	\]
	where $f_s(z) = \inf_{\lambda \in \neg_{\varepsilon} z}  \|\mu_{s-1} - \lambda \|^2_{V_{w_s}} $.
\end{lemma}
\begin{proof}
Let $B_s $ as in \eqref{eq:N_F_definition} for all $s \in \left\llbracket \left\lceil t^{\frac{1}{1+b}} \right\rceil, t \right\rrbracket$.
Since $f_s(z) \geq 0$, we can drop the the first positive terms and the ones for which the good event $B_s$ doesn't hold, i.e. $\left\llbracket n_{0}+1, \left\lceil t^{\frac{1}{1+b}} \right\rceil -1\right\rrbracket \cup \left\{s \in \left\llbracket \left\lceil t^{\frac{1}{1+b}} \right\rceil, t \right\rrbracket: \neg B_s \right\} $, hence
\begin{align*}
	&\min_{z \in \cZ} \sum_{s = n_{0} + 1}^{t} [f_s(\tilde{z}_s)- f_s(z)] \\
	&\leq \min_{z \in \cZ} \sum_{s \in \left\llbracket \left\lceil t^{\frac{1}{1+b}} \right\rceil, t \right\rrbracket} \1_{(B_s)}[f_s(\tilde{z}_s)- f_s(z)] + \sum_{s \in \left\llbracket \left\lceil t^{\frac{1}{1+b}} \right\rceil, t \right\rrbracket} \1_{(\neg B_s)}f_s(\tilde{z}_s) + \sum_{s \in \left\llbracket n_{0}+1, \left\lceil t^{\frac{1}{1+b}} \right\rceil -1\right\rrbracket} f_s(\tilde{z}_s) \\
	&= \min_{z \in \cZ} \sum_{s = n_{0} + 1}^{t} \1_{(B_s)}[f_s(z_{F}(\mu))- f_s(z)] + \sum_{s \in \left\llbracket \left\lceil t^{\frac{1}{1+b}} \right\rceil, t \right\rrbracket} \1_{(\neg B_s)}f_s(\tilde{z}_s) + \sum_{s \in \left\llbracket n_{0}+1, \left\lceil t^{\frac{1}{1+b}} \right\rceil -1\right\rrbracket} f_s(\tilde{z}_s) \\
	&\leq \sum_{s \in \left\llbracket \left\lceil t^{\frac{1}{1+b}} \right\rceil, t \right\rrbracket} \1_{(\neg B_s)}f_s(\tilde{z}_s) + \sum_{s \in \left\llbracket n_{0}+1, \left\lceil t^{\frac{1}{1+b}} \right\rceil -1\right\rrbracket} f_s(\tilde{z}_s)
\end{align*}
where we used that the best constant policy for $(f_s)$ is by definition better than the constant policy playing $z_{F}(\mu)$. By boundedness assumption, we have $f_s(\tilde{z}_s) \leq  4KM^2L_{\cK}^2$ for all $s$. Therefore, under $\cE_t$,
\begin{align*}
	\min_{z \in \cZ} \sum_{s = n_{0} + 1}^{t} [f_s(\tilde{z}_s)- f_s(z)] &\leq \left( N_{F}(t) + t^{\frac{1}{1+b}}  \right) 4KM^2L_{\cK}^2 = \cO \left( \ln(t)^2 t^{\frac{1}{1+b}} \right)  \: .
\end{align*}
\end{proof}

\subsubsection{Upper Bounding Cumulative Sums}   \label{app:subsubsection_upper_bounding_cumulative_sums}

Lemma~\ref{lem:lemma_9_degenne_2019_NonAsymptoticPureExploration} is very similar to Lemma 9 in \citet{degenne_2019_NonAsymptoticPureExploration} and Lemma 15 in \citet{jourdan_2021_EfficientPureExploration}, the novelty lies in using Lemma~\ref{lem:theorem_6_degenne_2020_StructureAdaptiveAlgorithms} to obtain tighter upper bounds. It gives an upper bound on the cumulative sum of the reweighted inverse of the empirical allocations.

\begin{lemma}   \label{lem:lemma_9_degenne_2019_NonAsymptoticPureExploration}
The tracking procedure, which draws $a_t \in \argmin_{a \in \cK} N_{t-1}^{a} - W_t^{a}$ where $W_t = W_{t-1} + w_t$, ensures that for all $t \in \Natural$,
\begin{align*}
	&\sum_{a \in \cK}\sum_{s=n_{0}}^{t} \frac{w_{s}^{a}}{N_{s}^{a}} \leq K \ln(K) + K \ln \left( t \right)  \quad \quad \text{ and }\quad \quad \sum_{a \in \cK} \sum_{s=n_{0}+1}^{t} \frac{w_{s}^{a}}{N_{s-1}^{a}} \leq K \ln(K) + 2K \ln \left( t \right) \\
	&\sum_{a \in \cK}\sum_{s=n_{0}}^{t} \frac{w_{s}^{a}}{\sqrt{N_{s}^{a}}} \leq K \ln(K) + 2 \sqrt{Kt}  \quad \quad \text{ and }\quad \quad \sum_{a \in \cK} \sum_{s=n_{0}+1}^{t} \frac{w_{s}^{a}}{N_{s-1}^{a}} \leq K \ln(K) + 2 \sqrt{2 Kt}
\end{align*}
\end{lemma}

\begin{proof}
Let's prove the first and the third inequality. Let $a \in \cK$ and $t_{0}^{a}$ be the first time such that: $\sum_{s=1}^{t_{0}^{a}-1} w_{s}^{a} > \ln(K) - 1$. Since $w_{t_{0}^{a} - 1}^{a} \leq 1$, this yields $\sum_{s=1}^{t_{0}^{a}-1} w_{s}^{a} \leq \ln(K)$. Since $N_{s}^{a} \geq 1$ for $s \geq n_{0}$, we obtain:
\begin{align*}
    &\sum_{s=n_{0}}^{t} \frac{w_{s}^{a}}{N_{s}^{a}} = \sum_{s=n_{0}}^{t_{0}^{a} - 1} \frac{w_{s}^{a}}{N_{s}^{a}} + \sum_{s=t_{0}^{a}}^{t} \frac{w_{s}^{a}}{N_{s}^{a}} \leq \sum_{s=n_{0}}^{t_{0}^{a} - 1} w_{s}^{a} + \sum_{s=t_{0}^{a}}^{t} \frac{w_{s}^{a}}{N_{s}^{a}} \leq \ln(K) + \sum_{s=t_{0}^{a}}^{t} \frac{w_{s}^{a}}{N_{s}^{a}} \\
    &\sum_{s=n_{0}}^{t} \frac{w_{s}^{a}}{\sqrt{N_{s}^{a}}} = \sum_{s=n_{0}}^{t_{0}^{a} - 1} \frac{w_{s}^{a}}{\sqrt{N_{s}^{a}}} + \sum_{s=t_{0}^{a}}^{t} \frac{w_{s}^{a}}{\sqrt{N_{s}^{a}}} \leq \sum_{s=n_{0}}^{t_{0}^{a} - 1} w_{s}^{a} + \sum_{s=t_{0}^{a}}^{t} \frac{w_{s}^{a}}{\sqrt{N_{s}^{a}}} \leq \ln(K) + \sum_{s=t_{0}^{a}}^{t} \frac{w_{s}^{a}}{\sqrt{N_{s}^{a}}}
\end{align*}

Combining the Lemma~\ref{lem:theorem_6_degenne_2020_StructureAdaptiveAlgorithms} and Lemma~\ref{lem:lemma_8_degenne_2019_NonAsymptoticPureExploration} for $x_s = w_{s}^{a}$, we obtain:
\begin{align*}
    &\sum_{s=t_{0}^{a}}^{t} \frac{w_{s}^{a}}{N_{s}^{a}} \leq \sum_{s=t_{0}^{a}}^{t} \frac{w_{s}^{a}}{\sum_{r=1}^{s} w_{r}^{a} -  \ln(K)} \leq \ln \left( \sum_{s = 1}^{t} w_{s}^{a} -  \ln(K) \right) - \ln \left( \sum_{s = 1}^{t_{0}^{a} - 1} w_{s}^{a} -  \ln(K) \right) \\
    &\sum_{s=t_{0}^{a}}^{t} \frac{w_{s}^{a}}{\sqrt{N_{s}^{a}}} \leq \sum_{s=t_{0}^{a}}^{t} \frac{w_{s}^{a}}{\sqrt{\sum_{r=1}^{s} w_{r}^{a} -  \ln(K)}} \leq 2 \sqrt{ \sum_{s = 1}^{t} w_{s}^{a} -  \ln(K) } - 2 \sqrt{ \sum_{s = 1}^{t_{0}^{a} - 1} w_{s}^{a} -  \ln(K)}
\end{align*}

Since $\sum_{s=1}^{t_{0}^{a}-1} w_{s}^{a} > \ln(K) - 1$, we have $\ln \left( \sum_{s = 1}^{t_{0}^{a} - 1} w_{s}^{a} -  \ln(K) \right) \geq 0$. Therefore, we have shown:
$\sum_{s=t_{0}^{a}}^{t} \frac{w_{s}^{a}}{N_{s}^{a}} \leq \ln \left( \sum_{s = 1}^{t} w_{s}^{a}  \right)$ and $\sum_{s=t_{0}^{a}}^{t} \frac{w_{s}^{a}}{\sqrt{N_{s}^{a}}} \leq 2 \sqrt{ \sum_{s = 1}^{t} w_{s}^{a}} $. By concavity of $x \mapsto \ln(x)$ (resp. $x \mapsto \sqrt{x}$) and the fact that $\sum_{s = 1}^{t} \sum_{a \in \cK} w_{s}^{a} = t$, we obtain:
\begin{align*}
	\sum_{a \in \cK}\sum_{s=n_{0}}^{t} \frac{w_{s}^{a}}{N_{s}^{a}} \leq K\ln(K) + K \ln \left( t \right) \quad \quad \text{ and } \quad \quad \sum_{a \in \cK}\sum_{s=n_{0}}^{t} \frac{w_{s}^{a}}{\sqrt{N_{s}^{a}}} \leq K \ln(K) + 2 \sqrt{Kt}
\end{align*}

For all $s \geq n_{0}$, we have $N_{s-1}^{a} \geq 1$, hence $N_{s-1}^{a} \geq \frac{1}{2}N_{s,a}$. Plugging this inequality in the sum starting from $t_{0}^{a}$ yields: $\sum_{a \in \cK} \sum_{s=n_{0}+1}^{t} \frac{w_{s}^{a}}{N_{s-1}^{a}} \leq K\ln(K) + 2K \ln \left( t \right) $ and $\sum_{a \in \cK}\sum_{s=n_{0}}^{t} \frac{w_{s}^{a}}{\sqrt{N_{s-1}^{a}}} \leq K \ln(K) + 2 \sqrt{2Kt} $.
\end{proof}

Lemma~\ref{lem:upper_bound_on_reweighted_norm_of_arms} gives an upper bound on the norm of each arms once reweighted by the inverse of the empirical allocation.

\begin{lemma} \label{lem:upper_bound_on_reweighted_norm_of_arms}
For all $s \in \left\llbracket n_{0}+1, t \right\rrbracket$ and all $a \in \cK$, we have $\|a\|^2_{V_{N_{s-1}}^{-1}} \leq \frac{1}{N_{s-1}^{a}}$.
\end{lemma}
\begin{proof}
Recall that $\cK$ is arbitrary and by initialization $N_{s-1}^{a} \geq 1$ for all $(s,a) \in \left\llbracket n_{0}+1, t \right\rrbracket \times \cK$. Let $a \in \cK$. Let's rewrite $V_{N_{s-1}} = C_{s-1}^{a} + N_{s-1}^{a} a a\transpose$ where $C_{s-1}^{a} = \sum_{b \in \cK: b\neq a} N_{s-1}^b b b\transpose$. If $\cK = \{e_{a}\}_{a \in [d]}$, we have directly that $\|a\|^2_{V_{N}^{-1}} = \frac{1}{N^{a}}$. In the following, we consider the more general case where $\cK \neq \{e_{a}\}_{a \in [d]}$.

\paragraph{Case 1: $C_{s-1}^{a}$ is invertible.} Using the Sherman-Morrison formula, we obtain: for all $s \in \left\llbracket n_{0}+1, t \right\rrbracket$ and all $a \in \cK$
\begin{align*}
	&V_{N_{s-1}}^{-1} = (C_{s-1}^{a})^{-1} - N_{s-1}^{a} \frac{(C_{s-1}^{a})^{-1} a a\transpose (C_{s-1}^{a})^{-1} }{1+ N_{s-1}^{a}a \transpose (C_{s-1}^{a})^{-1} a } \\
	&\|a\|^2_{V_{N_{s-1}}^{-1}} = a \transpose V_{N_{s-1}}^{-1} a = \frac{a \transpose (C_{s-1}^{a})^{-1} a}{1+ N_{s-1}^{a}a \transpose (C_{s-1}^{a})^{-1} a } \leq \frac{1}{N_{s-1}^{a}}
\end{align*}
where for the last inequality, we used that $\frac{x}{1+xy}\leq \frac{1}{y}$ (since it is equivalent with $1+xy \geq xy$ which is true) for $x=a \transpose (C_{s-1}^{a})^{-1} a$ and $y=N_{s-1}^{a}$.

\paragraph{Case 2: $C_{s-1}^{a}$ is not invertible.} By initialization, we know that $V_{N_{s-1}}$ is invertible for all $s \in \left\llbracket n_{0}+1, t \right\rrbracket$. Let $u \in \text{Ker}(C_{s-1}^{a}) \setminus \{0_d\}$. We have $N_{s-1}^{a} \langle a, u \rangle a = V_{N_{s-1}} u \neq 0_d$ since $V_{N_{s-1}}$ is invertible and $u \neq 0_d$. Given that $N_{s-1}^{a} > 0$ (otherwise $V_{N_{s-1}} = C_{s-1}^{a}$, hence contradiction with invertible), we obtain that $\langle a, u \rangle > 0$, hence $u \notin  \text{Span}(a)^{\perp}$. By dimension consideration, we obtain $\text{Span}(a) = \text{Ker}(C_{s-1}^{a})$. This yields directly that $a$ is an eigenvector of $V_{N_{s-1}}$ with eigenvalue $N_{s-1}^{a} \|a\|^2_2$. Therefore, $a$ is also an eigenvector of $V_{N_{s-1}}^{-1}$ with eigenvalue $\frac{1}{N_{s-1}^{a} \|a\|^2_2}$ and we can conclude: $a \transpose V_{N_{s-1}}^{-1} a = a \transpose \left( \frac{1}{N_{s-1}^{a} \|a\|^2_2} a\right) = \frac{1}{N_{s-1}^{a} }$.
\end{proof}

Lemma~\ref{lem:upper_bound_on_weighted_cumulative_sum_slack} gives an upper bound on the cumulative sum of the reweighted slacks involved in the optimistic reward.

\begin{lemma} \label{lem:upper_bound_on_weighted_cumulative_sum_slack}
\begin{align*}
	\sum_{s =n_{0}+1}^{t} \sum_{a \in \cK} w_{s}^{a} c_{s-1}^{a} &\leq f\left(t^{1+b}\right) \left( K \ln(K) + 2K\ln(t) \right) \\
	\sum_{s =n_{0}+1}^{t} \sum_{a \in \cK} w_{s}^{a} \sqrt{c_{s-1}^{a}} &\leq \sqrt{f\left(t^{1+b}\right)}  \left( K \ln(K) + \sqrt{8Kt}  \right)
\end{align*}
\end{lemma}
\begin{proof}
Using the definition of $c_{s}^{a}$ in (\ref{eq:optimistic_slack_definition}), $f\left((s-1)^{1+b}\right) \leq f\left(t^{1+b}\right)$ for all $s \leq t$, we obtain:
\begin{align*}
	\sum_{s =n_{0}+1}^{t} \sum_{a \in \cK} w_{s}^{a} c_{s-1}^{a} \leq \sum_{s =n_{0}+1}^{t} \sum_{a \in \cK} w_{s}^{a} f\left((s-1)^{1+b}\right) \|a\|^2_{V_{N_{s-1}}^{-1}} &\leq f\left(t^{1+b}\right)  \sum_{s =n_{0}+1}^{t} \sum_{a \in \cK} w_{s}^{a} \|a\|^2_{V_{N_{s-1}}^{-1}} \\
	&\leq f\left(t^{1+b}\right)  \sum_{s =n_{0}+1}^{t} \sum_{a \in \cK} \frac{w_{s}^a}{N_{s-1}^{a}} \\
	&\leq f\left(t^{1+b}\right) \left( K \ln(K) + 2K\ln(t) \right)
\end{align*}
where the second to last inequality is obtained by Lemma~\ref{lem:upper_bound_on_reweighted_norm_of_arms} and the last one by Lemma~\ref{lem:lemma_9_degenne_2019_NonAsymptoticPureExploration}. Using the same arguments, we obtain:
\begin{align*}
	\sum_{s =n_{0}+1}^{t} \sum_{a \in \cK} w_{s}^{a} \sqrt{c_{s-1}^{a}} \leq  \sqrt{f\left(t^{1+b}\right)} \sum_{s =n_{0}+1}^{t} \sum_{a \in \cK}  w_{s}^{a} \|a\|_{V_{N_{s-1}}^{-1}} &\leq  \sqrt{f\left(t^{1+b}\right)}  \sum_{s =n_{0}+1}^{t} \sum_{a \in \cK} \frac{w_{s}^a}{\sqrt{N_{s-1}^{a}}} \\
	 &\leq  \sqrt{f\left(t^{1+b}\right)}  \left( K \ln(K) + \sqrt{8Kt}  \right)
\end{align*}
\end{proof}

Lemma~\ref{lem:small_cumulative_reweighted_deviation_mle_mu} gives an upper bound on the cumulative sum of the reweighted KL divergence between the true parameter and its MLE.

\begin{lemma} \label{lem:small_cumulative_reweighted_deviation_mle_mu}
	Let $b > 0$ as in (\ref{eq:optimistic_slack_definition}). Under $\cE_t$,
	\begin{align*}
		&\sum_{s = \left\lceil t^{\frac{1}{1+b}} \right\rceil}^{t} \|\mu_{s-1} - \mu\|^2_{V_{w_s}} \leq f\left(t^{1+b}\right) \left( K \ln(K) + 2K\ln(t) \right) \\
	\end{align*}
\end{lemma}
\begin{proof}

Using Lemma~\ref{lem:optimistic_upper_bound_on_the_coordinate_deviation_mle_true}, we obtain:
\begin{align*}
	\sum_{s = \left\lceil t^{\frac{1}{1+b}} \right\rceil}^{t} \|\mu_{s-1} - \mu\|^2_{V_{w_s}} = \sum_{s = \left\lceil t^{\frac{1}{1+b}} \right\rceil}^{t} \sum_{a \in \cK } w_{s}^{a}\|\mu_{s-1} - \mu\|^2_{a a\transpose} \leq \sum_{s = \left\lceil t^{\frac{1}{1+b}} \right\rceil}^{t} \sum_{a \in \cK } w_{s}^{a}c_{s-1}^{a} &\leq \sum_{s =n_{0}+1}^{t} \sum_{a \in \cK } w_{s}^{a}c_{s-1}^{a} \\
	&\leq f\left(t^{1+b}\right) \left( K \ln(K) + 2K\ln(t) \right)
\end{align*}
where the second inequality uses $w_{s}^{a}c_{s-1}^{a} \geq 0$ and the last inequality is obtained by Lemma~\ref{lem:upper_bound_on_weighted_cumulative_sum_slack}.
\end{proof}

\section{Implementation Details and Additional Experiments} \label{app:section_implementations_and_experiments}

After presenting the implementations details in Appendix~\ref{app:subsection_implementations_details}, we display supplementary experiments in Appendix~\ref{app:subsection_additional_experiments}.

\subsection{Implementation Details} \label{app:subsection_implementations_details}

\paragraph{Computational Cost of \hyperlink{algoLeBAI}{L$\varepsilon$BAI}}
When $\overline \cM = \mathbb R^d$, the non-convex set of alternatives can be rewritten as a union over $Z-1$ half-spaces, on which there is a closed-form formula for the closest alternative of Lemma~\ref{lem:explicit_sample_complexity_addtive_multiplicative_optimality}.
Therefore, the stopping-recommendation pair has a computational cost in $\mathcal O(Z d^2|\mathcal Z_{\varepsilon}(\mu_{t-1})|)$.
For the heuristic L$\varepsilon$BAI, i.e. $\tilde z_t = z_t$, the sampling rule has a computational cost in $\mathcal O((K + Z)d^2)$.
Computing $z_{F}(\mu_{t-1})$ requires solving $|Z_{\varepsilon}(\mu_{t-1})|$ separate optimization problems.
Each one can be rewritten as an easier one-dimensional optimization problem \cite{garivier_2016_OptimalBestArm} which can be solved numerically by using two nested binary searches.
Our experiments testify of the feasibility of that procedure.

\paragraph{Algorithm Implementations} We list below more clarifications on the exact implementation of each individual algorithm.

\begin{itemize}
	\item The discretization of the simplex $\Delta_{2}$ and $\Delta_{4}$ (with $500$ and $10000$ vectors) is obtained by drawing uniformly vectors in the simplex, in practice we used a Dirichlet distribution with parameters $\frac{1}{2}\1_2$ and $\frac{1}{4}\1_4$.
	\item The algorithms DKM, LinGame and \hyperlink{algoLeBAI}{L$\varepsilon$BAI} are implemented without the boundedness assumption. In practice, given $(\mu,z,w,x)\in \cM \times \cZ_{\varepsilon}(\mu) \times \simplex \times \cZ \setminus \{z\}$, we used the closed-form formulas for the closest alternative $\lambda_{0}(\mu,z,w,x)$ and $\lambda_{\varepsilon}(\mu,z,w,x)$ as detailed in Appendix~\ref{proof:explicit_sample_complexity_addtive_multiplicative_optimality}.
	\item We consider the greedy version of LinGapE, which does not have a theoretical guarantee in the general case. We pull an arm with $z_t \in z^{\star}(\mu_{t-1})$ and $x_{t} = \argmax_{x \neq z_t} \langle \mu_{t-1}, x - z_t \rangle + \|x - z_t\|_{V^{-1}_{N_{t-1}}} \sqrt{2 \beta(t-1, \delta)}$.
	\item Likewise, we use the greedy version of $\cX\cY$-Static in order to avoid computing the optimal allocation at each step.
	\item We set the hyper-parameter of $\cX\cY$-Adaptive to $0.1$ as done in \citet{soare_2014_BestArmIdentificationLinear,degenne_2020_GamificationPureExploration}. It controls the length of each phase.
\end{itemize}

\paragraph{Stopping Threshold} Instead of the stopping threshold (\ref{eq:definition_stopping_threshold}) supported by the theory, we use as heuristic $\beta\left(t,\delta\right) = 4 \ln \left(\frac{4 + \ln(t/2)}{\delta}\right)$, where we take the main term and plug in $d$ instead of $K$. This is similar to what we could obtain with a threshold tailored to linear bandits and featuring only $d$. With this choice, the empirical error (number of runs such that $\hat{z} \notin \cZ_{\varepsilon}(\mu)$) is lower than the confidence parameter $\delta$. Previous algorithms on BAI settings \citep{garivier_2016_OptimalBestArm,degenne_2019_NonAsymptoticPureExploration,degenne_2020_GamificationPureExploration,jourdan_2021_EfficientPureExploration} were using $\beta\left(t,\delta\right) = \ln \left(\frac{1 + \ln(t)}{\delta}\right)$ as heuristic for the stopping threshold. Interestingly, using this coarser stopping threshold for $\varepsilon$-BAI leads to an empirical error which is higher than $\delta$, violating the $(\varepsilon, \delta)$-PAC property. Therefore, we need to be closer to the theoretically validated threshold.

An additional argument in favor of considering the same stopping rule for all BAI algorithm can be seen in the choice of the stopping threshold itself. DKM \citep{degenne_2019_NonAsymptoticPureExploration} uses the concentration results in \citet{garivier_2016_OptimalBestArm}, LinGame \citep{degenne_2020_GamificationPureExploration} the ones in \citep{lattimore_2020_BanditAlgorithms} and \hyperlink{algoLeBAI}{L$\varepsilon$BAI} the ones in \citep{kaufmann_2018_MixtureMartingalesRevisited}. Therefore, a fair comparison of the sampling rule requires using the same $\beta(t,\delta)$, as defined above.

\paragraph{Reproducibility} To assess our code and reproduce the experiments presented in this paper, you need to unzip the provided code by running \texttt{unzip code.zip}. All the algorithms and experiments are implemented in \texttt{Julia 1.6.3} (but also run on \texttt{Julia 1.1.1}). Plots were generated with the \texttt{StatsPlots.jl} package. Other dependencies are listed in the \texttt{Readme.md}. The \texttt{Readme.md} file provides detailed julia instructions to reproduce our experiments, as well as a \texttt{script.sh} to run them all at once.
The general structure of the code (and some functions) is taken from the \href{https://bitbucket.org/wmkoolen/tidnabbil}{tidnabbil} library.\footnote{This library was created by \cite{degenne_2019_NonAsymptoticPureExploration}, see https://bitbucket.org/wmkoolen/tidnabbil. No license were available on the repository, but we obtained the authorization from the authors.}

\subsection{Supplementary Experiments} \label{app:subsection_additional_experiments}

Supplementary experiments for the multiplicative $\varepsilon$-optimality are shown in Appendix~\ref{app:subsection_add_exp_multiplicative}. The equivalent experiments for the additive $\varepsilon$-optimality are displayed in Appendix~\ref{app:subsubsection_add_exp_additive}.

In the following, the hard/random instances that are considered are the same as in Section~\ref{sec:section_experiments}. We consider the same choice of parameters: $(\varepsilon, \delta) = (0.05, 0.01)$, discretization of the simplex with $500$ (resp. $10000$) vectors when $K=2$ (resp. $K=4$), average (resp. standard deviation) on $5000$ runs (resp. sub-samples of size $100$). Unless specified otherwise, the recommendation rule and the $\cZ$-oracle return an instantaneous furthest answer, i.e. $z_{t} = \tilde{z}_{t} \in z_{F}(\mu_{t-1}, N_{t-1})$, and the stopping-recommendation pair is updated/evaluated at each time $t$.

\subsubsection{Multiplicative Optimality} \label{app:subsection_add_exp_multiplicative}

Below, we present experiments on the multiplicative $\varepsilon$-optimality that were conducted to highlight algorithmic choices.

\begin{figure}[ht]
	\centering
	\includegraphics[width=0.5\linewidth]{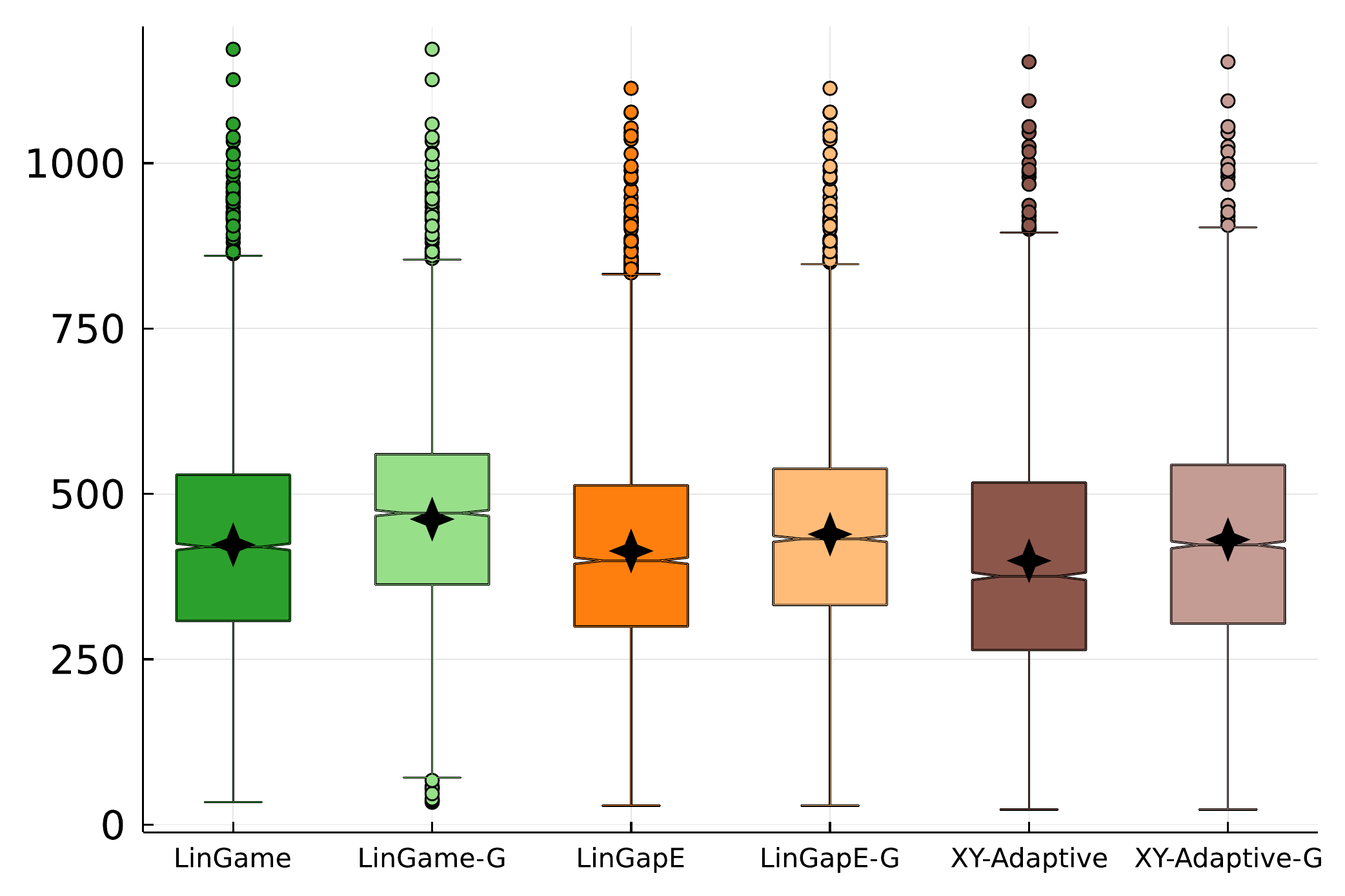}
	\caption{Empirical stopping time of the modified BAI algorithms with $z_t \in z_{F}(\mu_{t-1}, N_{t-1})$ on the hard instance (star equals mean). ``-G'' denotes when $z_t \in z^{\star}(\mu_{t-1})$.}
	\label{fig:hardinst_empirical_stop_modified_bai_mul}
\end{figure}

\paragraph{Modified BAI} Figure~\ref{fig:hardinst_empirical_stop_modified_bai_mul} conveys the same message as Figure~\ref{fig:hardinst_empirical_stop_modified_bai_add}. For all modified BAI using (\ref{eq:definition_stopping_criterion}), considering an instantaneous furthest answer instead of a greedy answer leads to lower empirical stopping time. Their ratio is $0.928$ on average (coherent with Figure~\ref{fig:theoretical_study_greedy_vs_furthest_mul}(b)). Note that the extension of LinGapE to the multiplicative $\varepsilon$-optimality is not obtained directly as was done for the additive setting, hence we didn't displayed it.

\begin{table}[ht]
\caption{Empirical stopping time ($\pm$ $\sigma$) with their original stopping rule or with ours (\ref{eq:definition_stopping_criterion}) on the hard instance ($\cK=\cZ$). The modified BAI algorithms use (\ref{eq:definition_stopping_criterion}) with $z_t \in z_{F}(\mu_{t-1}, N_{t-1})$.}
\label{tab:comparison_stopping_rule_BAI_algos_mul}
\begin{center}
\begin{tabular}{c c c c c}
  \toprule
  & LinGame & LinGapE & $\cX\cY$-Adaptive	\\
  \cmidrule(l){2-4}
  Original        & $151473$ $(\pm18082)$ & $121716$ $(\pm14771)$ & $263806$ $(\pm36896)$ \\
  Modified        & $423$ $(\pm 51)$ & $414$ $(\pm 51)$ & $399$ $(\pm 52)$ \\
  \bottomrule
\end{tabular}
\end{center}
\end{table}

\paragraph{Original Stopping Rule} In Table~\ref{tab:comparison_stopping_rule_BAI_algos_mul}, we displayed the difference between the empirical stopping time of BAI algorithms using their original stopping rule and ours (\ref{eq:definition_stopping_criterion}). Their stopping time is $438$ times higher when using their original stopping rule which was designed for an harder problem, i.e. identifying the unique best-answer which is also an $\varepsilon$-optimal answer.

\begin{table}[ht]
\caption{Average number of pulls per arm and empirical stopping time ($\pm$ $\sigma$) on the hard instance ($\cK=\cZ$). The modified BAI algorithms use (\ref{eq:definition_stopping_criterion}) with $z_t \in z_{F}(\mu_{t-1}, N_{t-1})$.}
\label{tab:average_number_pulls_per_arm_mul}
\begin{center}
\begin{tabular}{c r r r r c}
  \toprule
  & $a_1$ & $a_2$ & $a_3$ & $a_4$ & \textbf{Total} \\
  \cmidrule(l){2-6}
 \hyperlink{algoLeBAI}{L$\varepsilon$BAI} & $115$ & $247$ & $ 19$ & $  3$ & $384$ $(\pm 17)$ \\
LinGame         & $120$ & $239$ & $ 52$ & $ 12$ & $423$ $(\pm 16)$ \\
DKM             & $174$ & $223$ & $173$ & $175$ & $745$ $(\pm 29)$ \\
LinGapE         & $ 63$ & $349$ & $  1$ & $  1$ & $414$ $(\pm 16)$ \\
G-Static        & $204$ & $227$ & $  7$ & $ 19$ & $458$ $(\pm 18)$ \\
$\cX\cY$-Static & $223$ & $224$ & $  1$ & $  1$ & $449$ $(\pm 17)$ \\
$\cX\cY$-Adaptive & $ 82$ & $315$ & $  1$ & $  1$ & $399$ $(\pm 18)$ \\
Fixed           & $ 96$ & $271$ & $  1$ & $  1$ & $370$ $(\pm 16)$ \\
Uniform         & $214$ & $214$ & $214$ & $213$ & $856$ $(\pm 33)$ \\
  \bottomrule
\end{tabular}
\end{center}
\end{table}

\paragraph{Empirical Allocation} Table~\ref{tab:average_number_pulls_per_arm_mul} details the empirical allocations of pulls. Uniform and DKM sample all arms similarly. Therefore, they stop with almost twice as many samples as the other algorithms whose empirical allocations are close to the oracle allocation used in the fixed algorithm. As expected, $\cX\cY$-Adaptive outperforms $\cX\cY$-Static and G-Static. We see that \hyperlink{algoLeBAI}{L$\varepsilon$BAI} slightly outperforms LinGapE and $\cX\cY$-Adaptive, performs better than LinGame and $\cX\cY$-Static and is on par with the ``oracle'' \textit{fixed} algorithm.

\begin{figure}[ht]
	\centering
	\includegraphics[width=0.32\linewidth]{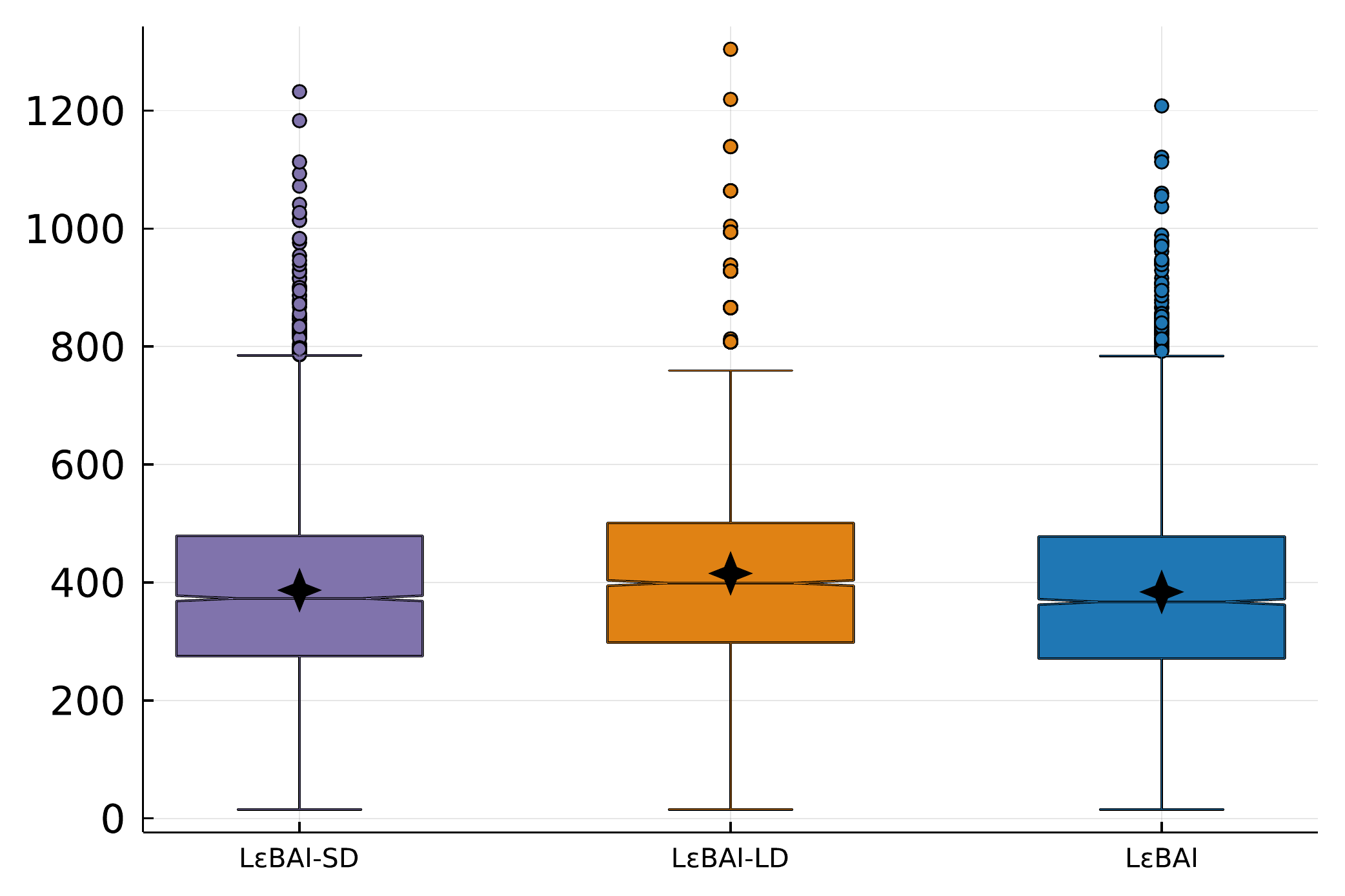}
	\includegraphics[width=0.32\linewidth]{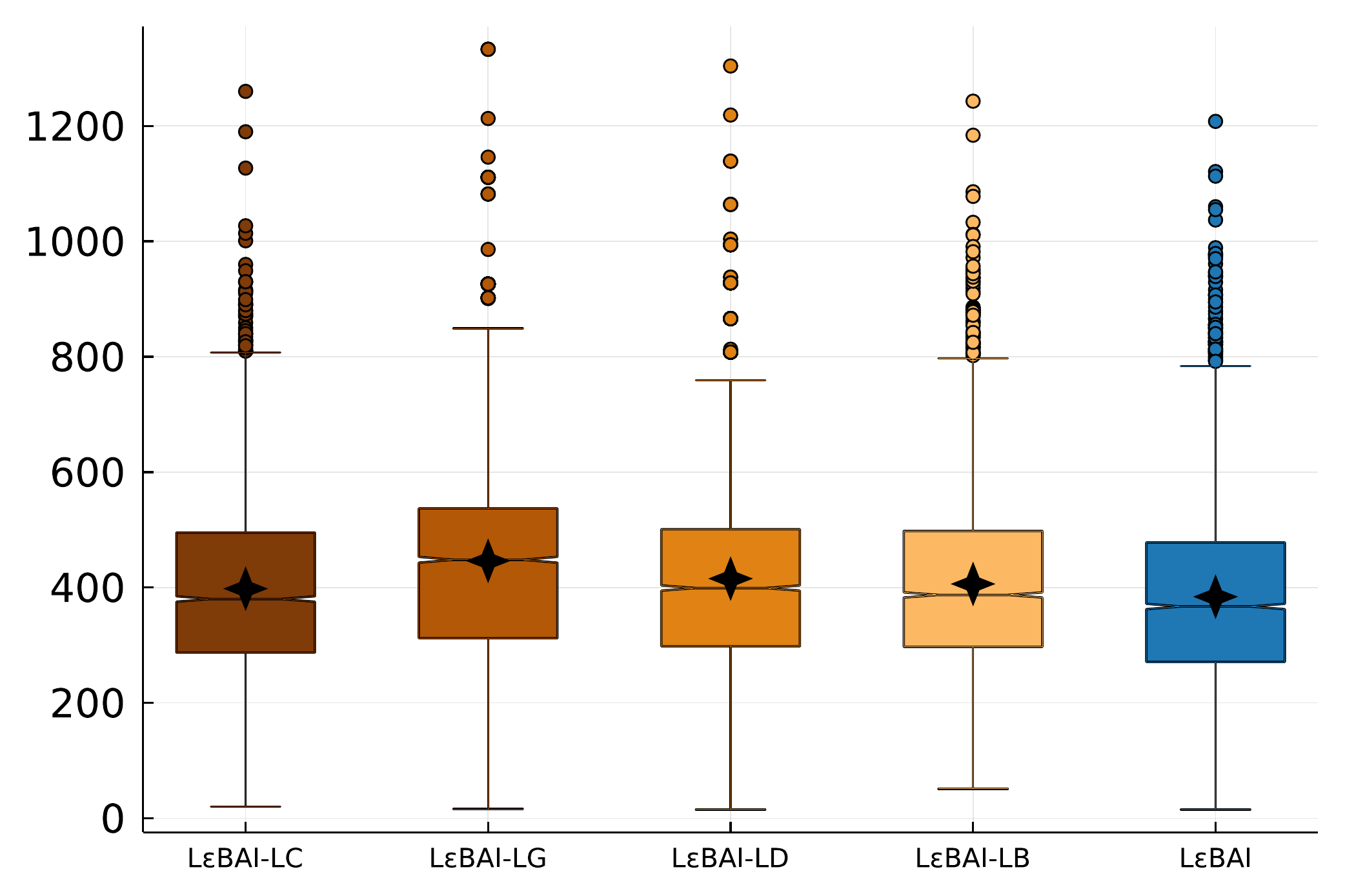}
	\includegraphics[width=0.32\linewidth]{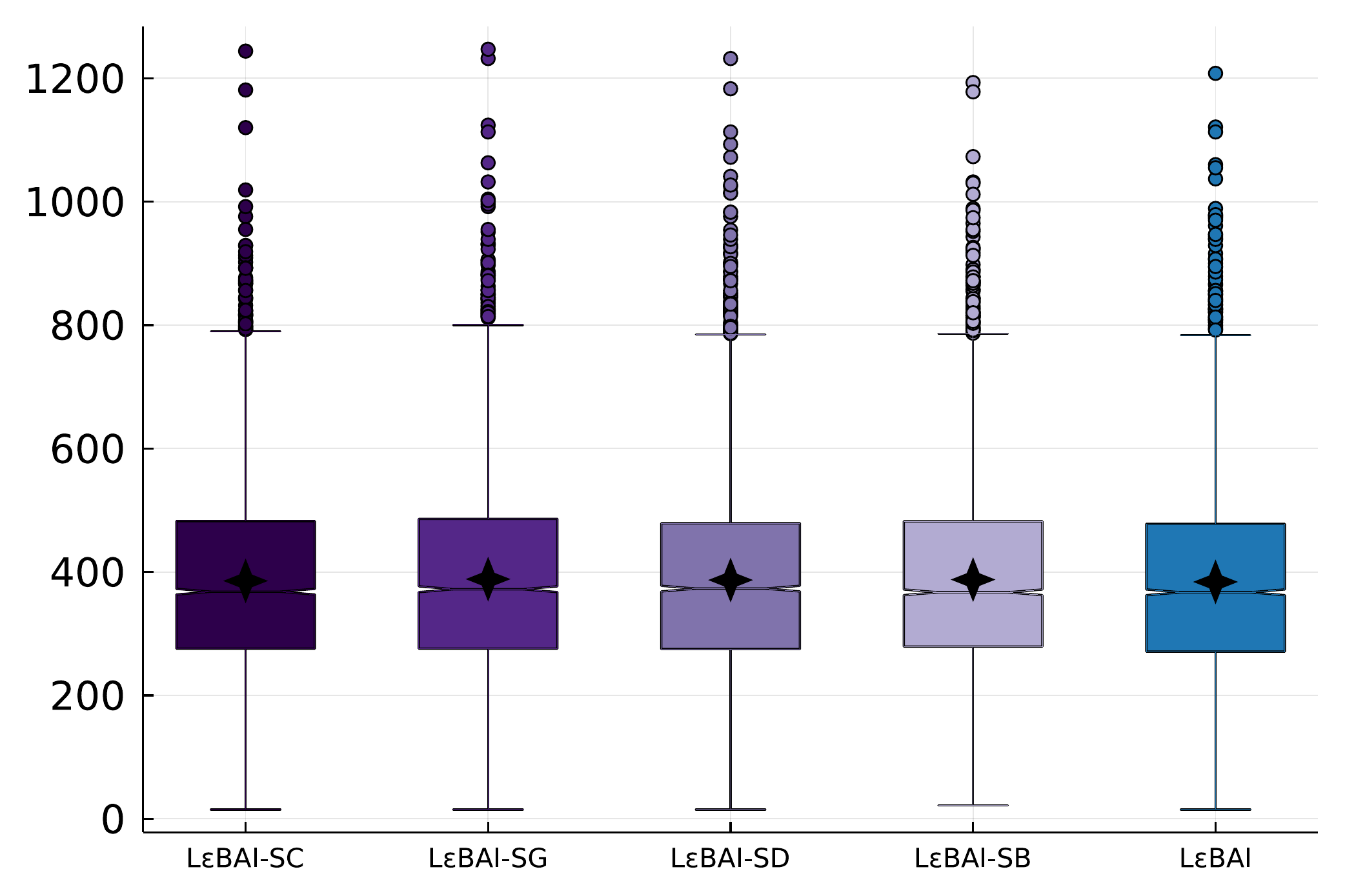}
	\caption{Empirical stopping time on the hard instance ($\cK = \cZ$) for (a) the lazy and sticky update, and different implementations of (b) the lazy scheme and (c) the sticky scheme. ``-S'' denotes the sticky scheme and ``-L'' the lazy one. The notations for implementations are: ``-C'' for the constant one with $T_{0} = 10$, ``-G'' for the geometric one with $(T_{0}, \gamma)= (10, 0.2)$, ``-D'' for geometrically decreasing one with $(T_{0}, \gamma)= (10, 0.2)$ and ``-B'' for the Bernoulli one with parameter $p=0.1$.}
	\label{fig:hardinst_stopreco_update_schedule_mul}
\end{figure}

\paragraph{Computational Relaxations} In Appendix~\ref{app:subsubsection_other_update_schemes}, we introduced the sticky and the lazy schemes. Instead of updating the candidate answer and evaluating the stopping criterion at each time $t$, the lazy scheme only does it on all time $t \in \cT \subseteq \Natural$. The sticky scheme evaluates the stopping criterion at each time $t \geq n_{0}+1$, but update the candidate answer when $t \in \cT \subseteq \Natural$. With a good choice of $\cT$, using the sticky/lazy schemes with \hyperlink{algoLeBAI}{L$\varepsilon$BAI} yield an asymptotically optimal algorithm (Appendix~\ref{app:subsubsection_other_update_schemes}). To ensure that the strategy is $(\varepsilon, \delta)$-PAC, it is sufficient to consider $|\cT| = + \infty$ (Lemma~\ref{lem:delta_PAC_recommendation_stopping_pair}), property satisfied by all the implementations below.

In Figure~\ref{fig:hardinst_stopreco_update_schedule_mul}, we test several implementations of lazy/sticky schemes to assess their impact on the empirical stopping time. They differ by how the infinite grid of time $\cT$ is defined. The constant implementation uses an arithmetic grid of time with parameter $T_{0} > n_{0}$, i.e. $\cT = \left\{n_{0}+1\right\} \cup \left\{i T_{0} \right\}_{i \in \Natural^{\star}}$. The geometric implementation uses a geometric grid of time with parameter $T_{0} > n_{0}$ and $\gamma > 0$, i.e. $\cT \eqdef \left\{n_{0}+1\right\} \cup \left\{T_{i} \right\}_{i \in \Natural}$ where $T_{i} = \lceil(1 + \gamma) T_{i-1}\rceil$ for $i \in \Natural^{\star}$. The geometrically decreasing implementation uses a grid of time with parameter $T_{0} > n_{0}$ and $\gamma > 0$, i.e. $\cT \eqdef \left\{n_{0}+1\right\} \cup \left\{T_{i} \right\}_{i \in \Natural}$ where $T_{i} = \lceil(1 + \frac{\gamma}{\sqrt{i}}) T_{i-1}\rceil$ for $i \in \Natural^{\star}$. The Bernoulli implementation with parameter $p$ is slightly different as it adds internal randomness, hence it is not a deterministic strategy anymore. The idea is simple: draw $X \sim \cB(p)$, if $X=1$ we update the candidate answer, else we stick to it. The geometrically decreasing implementation is the only implementation ensuring that \hyperlink{algoLeBAI}{L$\varepsilon$BAI} is an asymptotically optimal algorithm (Appendix~\ref{app:subsubsection_other_update_schemes}).

In Figure~\ref{fig:hardinst_stopreco_update_schedule_mul}(a), we see that the sticky relaxation allows to perform on par with the algorithm updating the recommendation rule at each time $t$. However, when considering the lazy relaxation, we pay the price of not evaluating the stopping rule at each time $t$ by incurring a slightly higher empirical stopping time. Therefore, depending on the constraints of the practitioner, the sticky/lazy schemes allow to reduce the computational cost per time step while keeping similar sample complexity.

In Figure~\ref{fig:hardinst_stopreco_update_schedule_mul}(b), we observe that the empirical stopping time might be higher depending on the implementation of the lazy scheme. Overall, algorithms using a lazy scheme suffer from slightly worse sample complexity. This is the price to pay to drastically reduce the computational cost.

In Figure~\ref{fig:hardinst_stopreco_update_schedule_mul}(c), we see that the exact implementation of the sticky scheme has few consequences as regards the empirical stopping time since they all perform on par with the algorithm updating the recommendation rule at each time $t$. Therefore, the sticky scheme is the computational relaxation to adopt when one wishes to reduce the computational cost in a significant manner without damaging the empirical performance.

\begin{figure}[ht]
	\centering
	\includegraphics[width=0.5\linewidth]{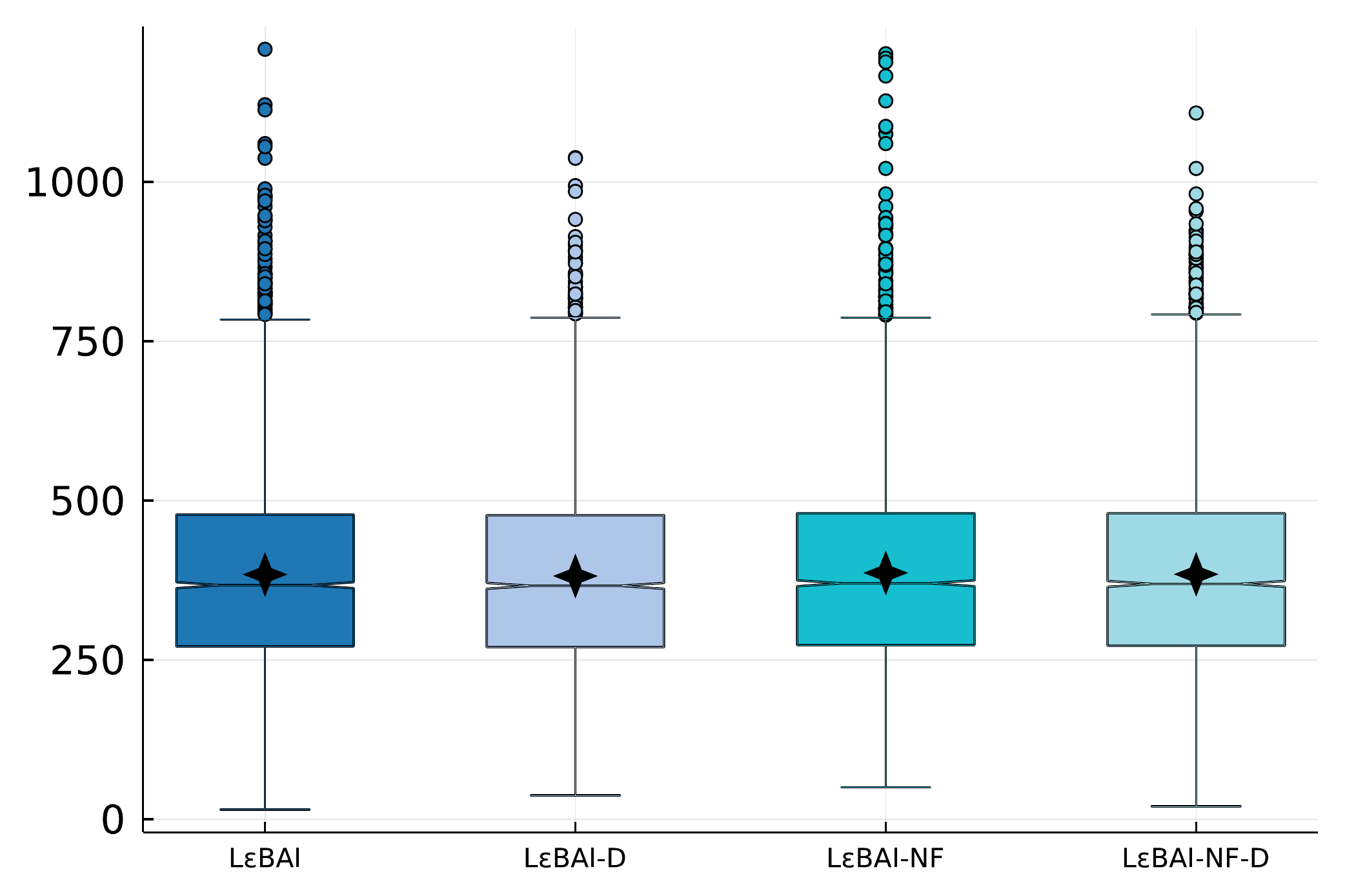}
	\caption{Empirical stopping time on the hard instance ($\cK = \cZ$). ``-D'' denotes when the D-Tracking is used instead of C-Tracking and ``-NF'' denotes the removal of forced exploration.}
	\label{fig:BAI_hardinst_tracking_fe_mul}
\end{figure}

\paragraph{Tracking And Forced Exploration} The tracking procedure that we use throughout the paper is referred as C-Tracking in the literature, since it tracks the cumulative sum of pulling proportions over arms played by the agent, i.e. $a_t \in \argmin_{a \in \cK} N_{t-1}^{a} - W_t^{a}$ where $W_t = \sum_{s=n_{0}+1}^{t} w_s$. Another commonly considered tracking procedure is D-Tracking \citep{garivier_2016_OptimalBestArm}, in which the pulling proportion over arms played by the agent at time $t$ is directly tracked, i.e. $a_t \in \argmin_{a \in \cK} N_{t-1}^{a} - (t-n_{0}) w_t^a$. In \citet{degenne_2019_PureExplorationMultiple}, it is shown that D-Tracking might fail when several correct answer exist. This explains why we only considered C-Tracking in this paper. In Figure~\ref{fig:BAI_hardinst_tracking_fe_mul}, we observe that considering D-Tracking instead of C-Tracking leads to similar empirical stopping time.

As discussed in Appendix~\ref{app:subsubsection_key_assumptions}, we need a logarithmic forced exploration to conclude the proof, which was not needed by previous game-based approaches \citet{degenne_2020_GamificationPureExploration}. In Figure~\ref{fig:BAI_hardinst_tracking_fe_mul}, we see that removing the forced exploration has almost no impact on the empirical stopping time.

\subsubsection{Additive Optimality} \label{app:subsubsection_add_exp_additive}

When comparing the following plots and their equivalent in the main content of the paper, we observe that the characteristic time and the empirical stopping time in the additive $\varepsilon$-optimality is always smaller than for the multiplicative one (all other parameters being identical). If there isn't a natural choice of $\varepsilon$-optimality, this fact might influence the choice of the practitioner. Overall the messages conveyed by the following plots are the same as the ones highlighted in the multiplicative $\varepsilon$-optimality case, we report them for the sake of completeness.

\begin{figure}[ht]
	\centering
	\includegraphics[width=0.485\linewidth]{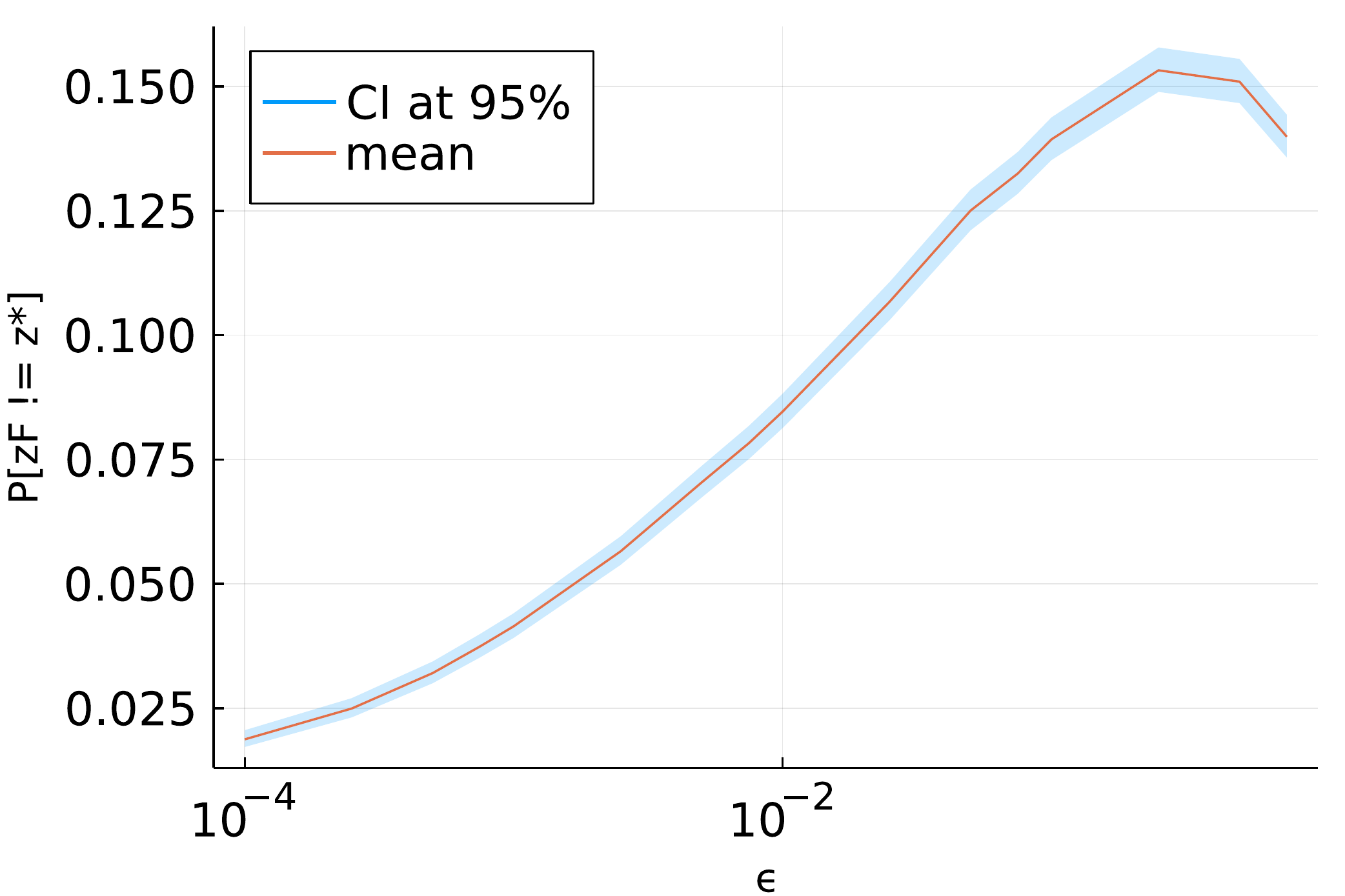}
	\includegraphics[width=0.485\linewidth]{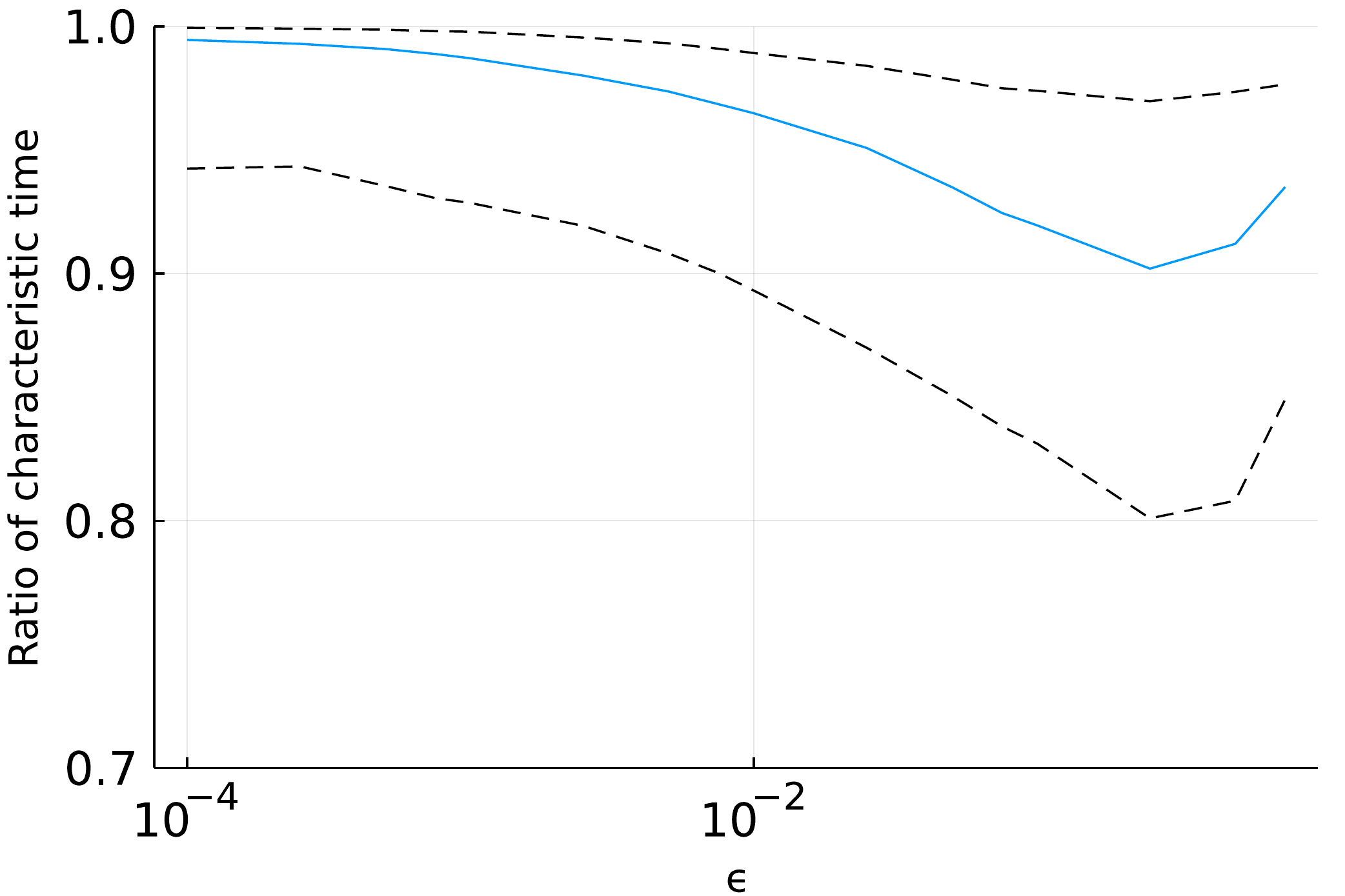}
	\caption{Influence of $\varepsilon$ on (a) the proportion of draws where $z_{F}(\mu) \notin z^{\star}(\mu)$, (b) the median (and first/third quartile) of $\frac{T_{\varepsilon}^{\add}(\mu)}{T_{g,\varepsilon}^{\add}(\mu)}$, when $z_{F}(\mu) \notin z^{\star}(\mu)$.}
	\label{fig:theoretical_study_greedy_vs_furthest_add}
\end{figure}

\paragraph{Furthest Answer} The results observed in Figure~\ref{fig:theoretical_study_greedy_vs_furthest_add} are similar to the ones in Figure~\ref{fig:theoretical_study_greedy_vs_furthest_mul}, while the discrepancy is less noteworthy. The proportion of draws where $ z_{F}(\mu) \notin z^{\star}(\mu)$ is not negligible, reaching on average $10 \%$. The impact on the characteristic time is visible with a ratio on average at $0.95$, but it is very slim when considering $\varepsilon \leq 0.01$. Up to $\varepsilon \approx 0.25$ the proportion of draws where $z_{F}(\mu) \notin z^{\star}(\mu)$ has a logarithmic increase with $\varepsilon$. Therefore, when they are different, the furthest answer still outperforms a greedy answer achieving $T_{g,\varepsilon}^{\add}(\mu)$. We note that this difference is smaller in the additive $\varepsilon$-optimality than for the multiplicative $\varepsilon$-optimality.

\begin{table}[ht]
\caption{Empirical stopping time ($\pm$ $\sigma$) with their original stopping rule or with ours (\ref{eq:definition_stopping_criterion}) on the hard instance ($\cK=\cZ$). The modified BAI algorithms use (\ref{eq:definition_stopping_criterion}) with $z_t \in z_{F}(\mu_{t-1}, N_{t-1})$.}
\label{tab:comparison_stopping_rule_BAI_algos_add}
\begin{center}
\begin{tabular}{c  c  c  c }
  \toprule
  & LinGame & LinGapE & $\cX\cY$-Adaptive\\
\cmidrule(l){2-4}
Original        & $151473$ $(\pm18082)$ & $121716$ $(\pm14771)$ & $263806$ $(\pm36896)$ \\
Modified        & $393$ $(\pm 49)$ & $297$ $(\pm 39)$ & $314$ $(\pm 38)$ \\
\bottomrule
\end{tabular}
\end{center}
\end{table}

\paragraph{Original Stopping Rule} In Table~\ref{tab:comparison_stopping_rule_BAI_algos_add}, the \textit{Original} line is unchanged compared to Table~\ref{tab:comparison_stopping_rule_BAI_algos_mul} since BAI algorithms consider $\varepsilon = 0$, hence no multiplicative/additive notions of $\varepsilon$-optimality. Similarly, their stopping time is $545$ times higher when using their original stopping rule which was designed for an harder problem, i.e. identifying the unique best-answer which is also an $\varepsilon$-optimal answer.

\begin{table}[ht]
\caption{Empirical stopping time ($\pm$ $\sigma$) for different combinations of sampling rule and recommendation rule on the hard instance with $\cK=\{e_1,e_2\}$.}
\label{tab:average_empirical_stopping_time_add}
\begin{center}
\begin{tabular}{c  c  c  c }
  \toprule
  & $z^{\star}(\mu_{t-1})$ & $z_F(\mu_{t-1})$ & $z_F(\mu_{t-1}, N_{t-1})$\\
\cmidrule(l){2-4}
\hyperlink{algoLeBAI}{L$\varepsilon$BAI} & $350$ $(\pm 14)$ & $318$ $(\pm 12)$ & $318$ $(\pm 13)$ \\
$\varepsilon$-TaS & $323$ $(\pm 12)$ & $298$ $(\pm 13)$ & $298$ $(\pm 13)$ \\
Fixed           & $324$ $(\pm 10)$ & $300$ $(\pm 11)$ & $300$ $(\pm 11)$ \\
Uniform         & $477$ $(\pm 15)$ & $434$ $(\pm 16)$ & $434$ $(\pm 16)$ \\
\bottomrule
\end{tabular}
\end{center}
\end{table}

\paragraph{Choosing Answers} Table~\ref{tab:average_empirical_stopping_time_add} conveys the same messages as Table~\ref{tab:average_empirical_stopping_time_mul}. It shows that using greedy is consistently worse than an (instantaneous) furthest answer, the ratio of their stopping time is $0.917$ on average (coherent with Figure~\ref{fig:theoretical_study_greedy_vs_furthest_add}(b)) and that instantaneous furthest answer and furthest answer have the same empirical performance. While \hyperlink{algoLeBAI}{L$\varepsilon$BAI} consistently outperform uniform sampling ($73 \%$), it performs slightly worse than $\varepsilon$-TaS and the ``oracle'' \textit{fixed} algorithm (tracking the optimal allocation $w_{F}(\mu)$).

\begin{table}[ht]
\caption{Average number of pulls per arm and empirical stopping time ($\pm$ $\sigma$) on the hard instance ($\cK=\cZ$). The modified BAI algorithms use (\ref{eq:definition_stopping_criterion}) with $z_t \in z_{F}(\mu_{t-1}, N_{t-1})$.}
\label{tab:average_number_pulls_per_arm_add}
\begin{center}
\begin{tabular}{c r r r r c}
  \toprule
  & $a_1$ & $a_2$ & $a_3$ & $a_4$ & \textbf{Total} \\
\cmidrule(l){2-6}
 \hyperlink{algoLeBAI}{L$\varepsilon$BAI} & $ 77$ & $228$ & $ 13$ & $  3$ & $321$ $(\pm 13)$ \\
LinGame         & $112$ & $221$ & $ 49$ & $ 11$ & $393$ $(\pm 15)$ \\
DKM             & $169$ & $217$ & $168$ & $169$ & $723$ $(\pm 27)$ \\
LinGapE         & $ 47$ & $248$ & $  1$ & $  1$ & $297$ $(\pm 13)$ \\
G-Static        & $196$ & $217$ & $  7$ & $ 18$ & $438$ $(\pm 18)$ \\
$\cX\cY$-Static & $215$ & $217$ & $  1$ & $  1$ & $434$ $(\pm 17)$ \\
$\cX\cY$-Adaptive & $ 76$ & $236$ & $  1$ & $  1$ & $314$ $(\pm 12)$ \\
Fixed           & $ 48$ & $251$ & $  1$ & $  1$ & $300$ $(\pm 11)$ \\
Uniform         & $211$ & $211$ & $211$ & $210$ & $844$ $(\pm 33)$ \\
\bottomrule
\end{tabular}
\end{center}
\end{table}

\begin{figure}[ht]
	\centering
	\includegraphics[width=0.5\linewidth]{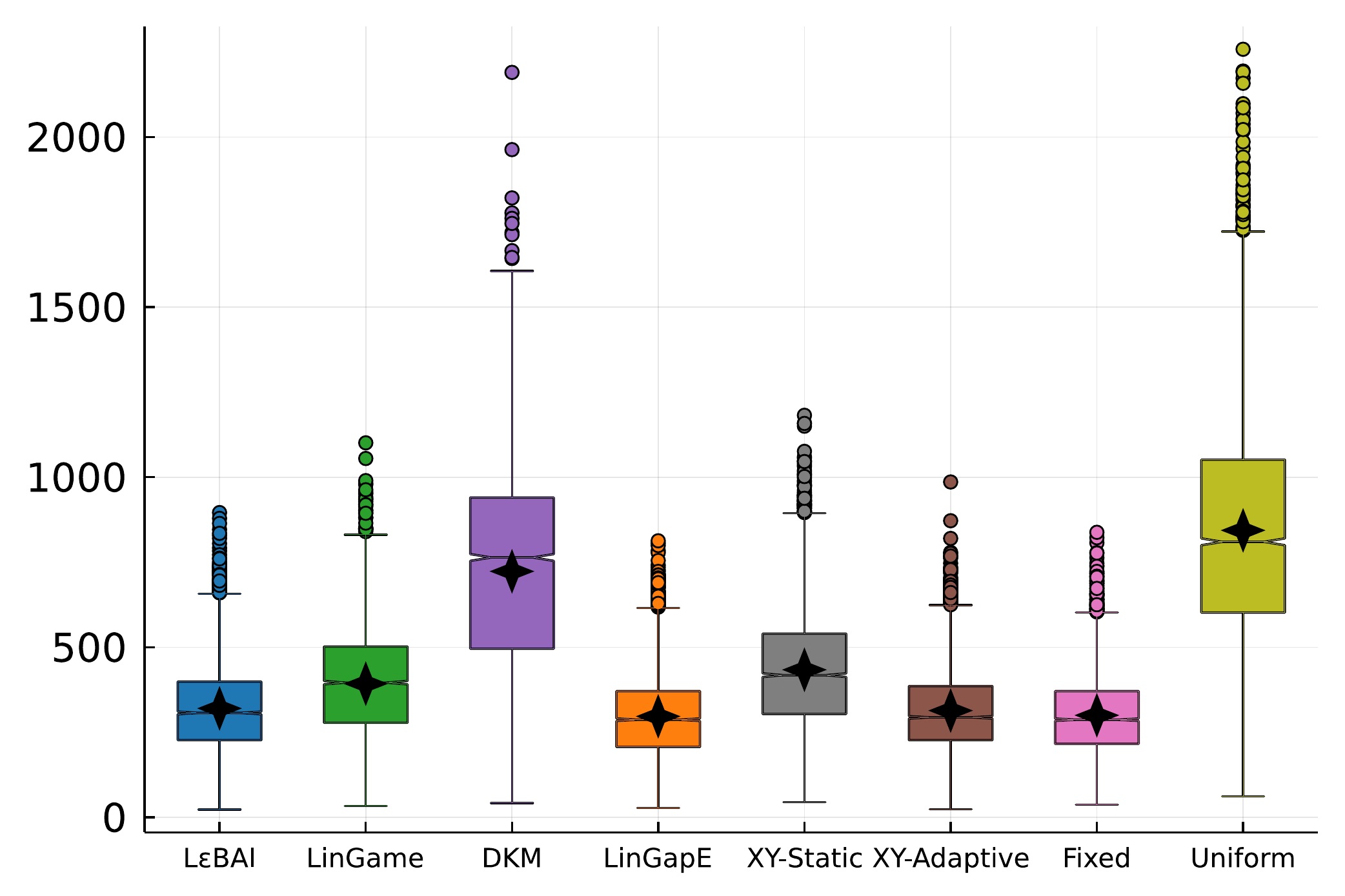}
	\caption{Empirical stopping time on the hard instance ($\cK=\cZ$), the star represents the mean. The modified BAI algorithms use (\ref{eq:definition_stopping_criterion}) with $z_t \in z_{F}(\mu_{t-1}, N_{t-1})$.}
	\label{fig:hardinst_empirical_stop_deltas_add}
\end{figure}

\paragraph{Modified BAI Algorithms and Empirical Allocation} The detailed empirical allocations presented in Table~\ref{tab:average_number_pulls_per_arm_add} bare similarity with the results shown in Table~\ref{tab:average_number_pulls_per_arm_mul}. Likewise, Figure~\ref{fig:hardinst_empirical_stop_deltas_add} shows similar results as Figure~\ref{fig:hardinst_empirical_stop_lebai_vs_modified_bai_mul}. Uniform and DKM sample all arms equally. Their stopping time is twice as high as the other algorithms whose empirical allocations are close to the oracle allocation (fixed algorithm). $\cX\cY$-Adaptive outperforms $\cX\cY$-Static and G-Static. We see that \hyperlink{algoLeBAI}{L$\varepsilon$BAI} slightly outperforms LinGapE and $\cX\cY$-Adaptive, performs better than LinGame and $\cX\cY$-Static and is on par with the ``oracle'' \textit{fixed} algorithm.

\begin{figure*}[ht]
	\centering
	\includegraphics[width=0.485\linewidth]{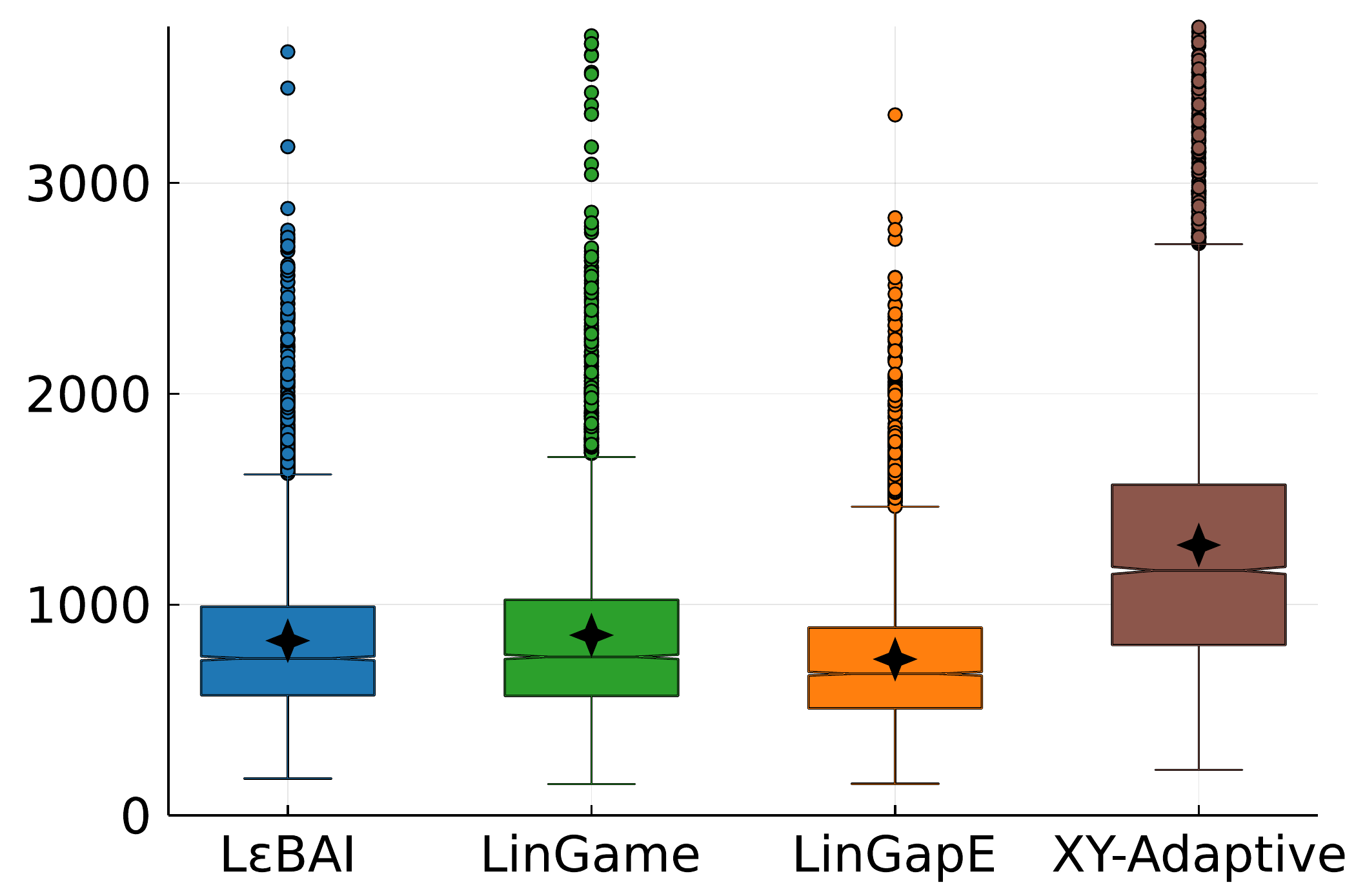}
	\includegraphics[width=0.485\linewidth]{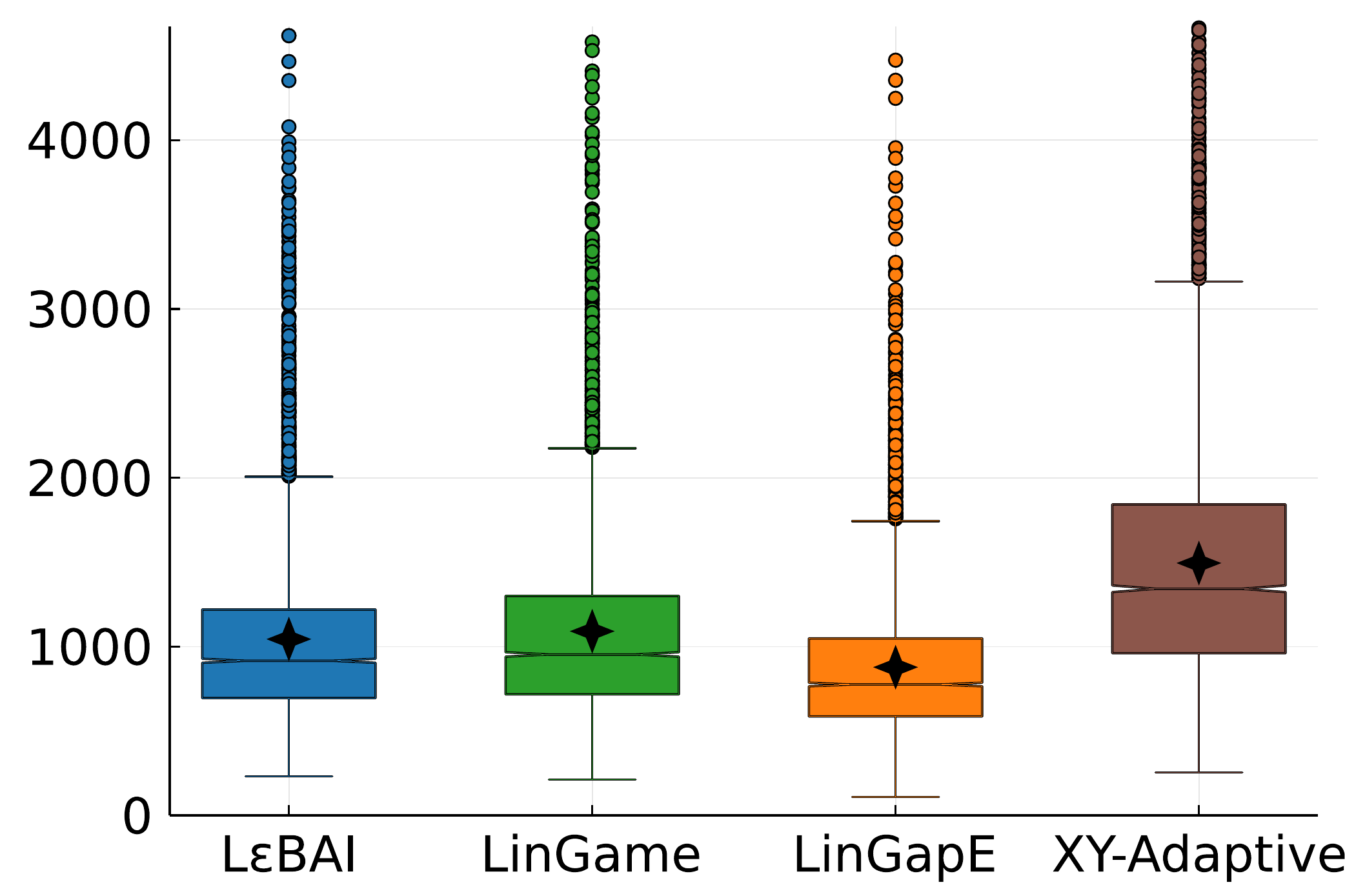} \\
	\includegraphics[width=0.485\linewidth]{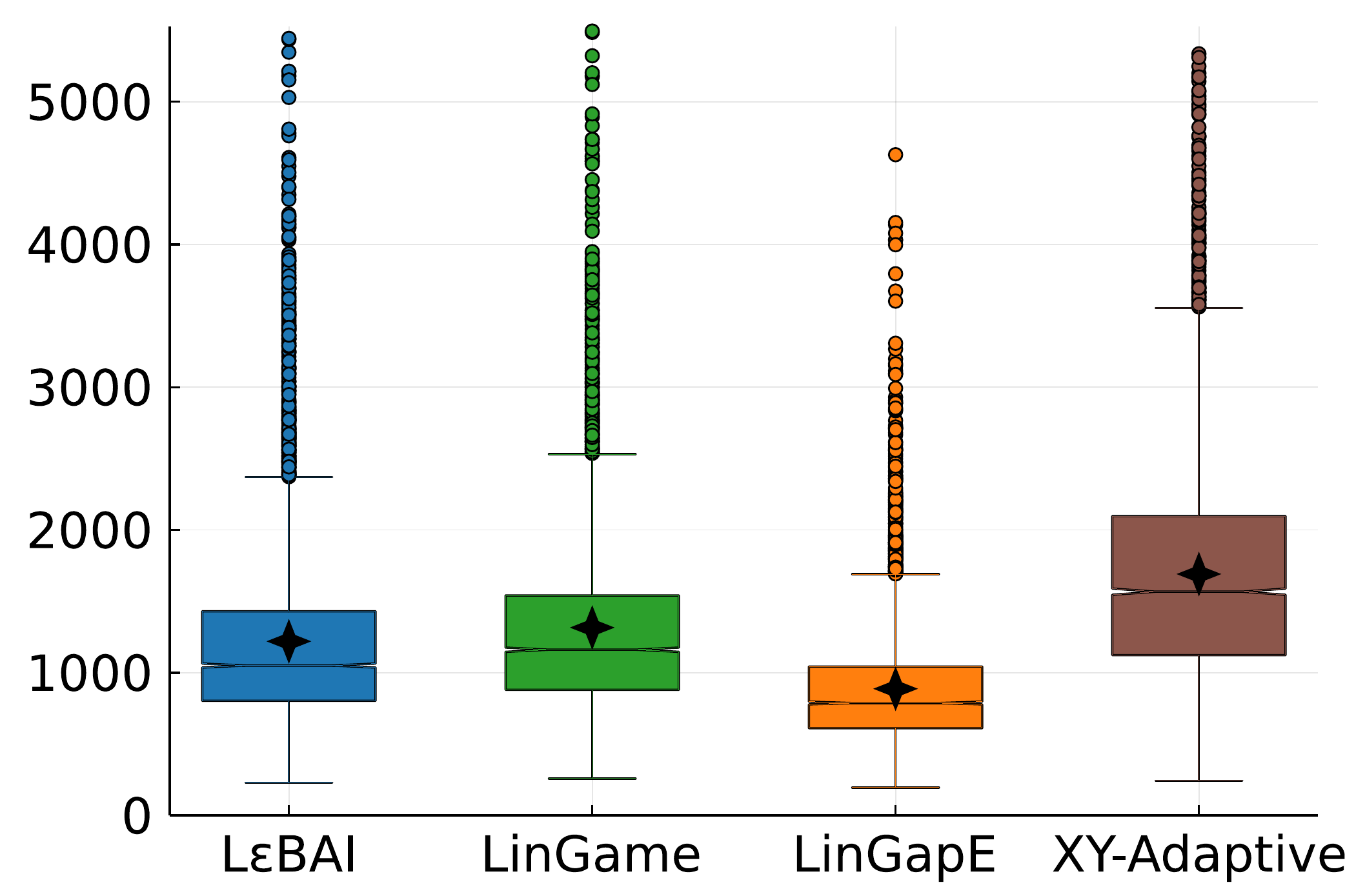}
	\includegraphics[width=0.485\linewidth]{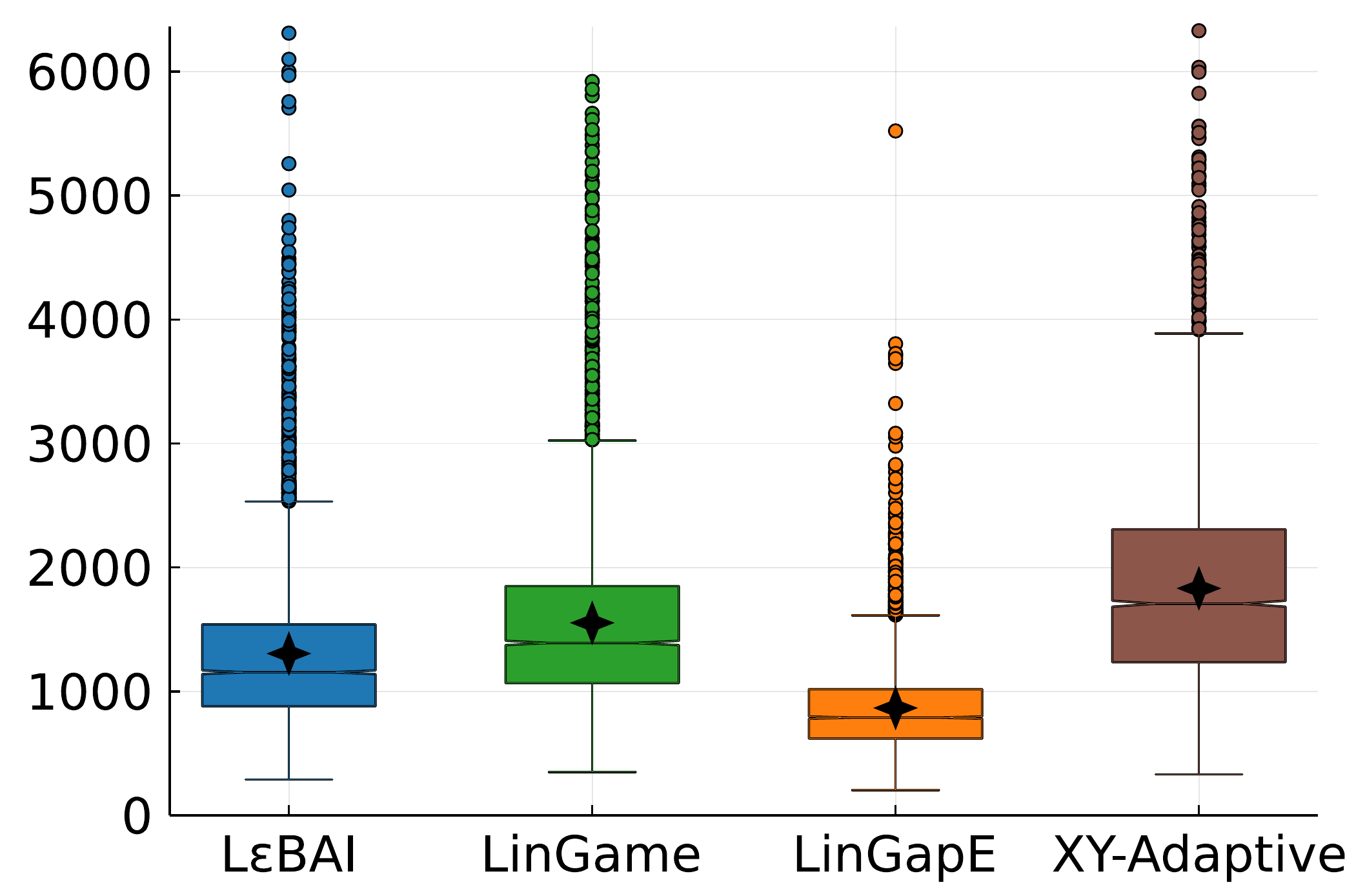}
	\caption{Empirical stopping time on random instances ($\cK=\cZ$) for $d \in \{6,8,10,12\}$ (from top left to bottom right). The modified BAI algorithms use (\ref{eq:definition_stopping_criterion}) with $z_t \in z_{F}(\mu_{t-1}, N_{t-1})$.}
	\label{fig:randinst_empirical_stop_dimensions_add}
\end{figure*}

\paragraph{Random Instances In Higher Dimensions} For random instances and increasing dimension, Figure~\ref{fig:randinst_empirical_stop_dimensions_add} reveals the same trends as Figure~\ref{fig:randinst_empirical_stop_dimensions_mul}. \hyperlink{algoLeBAI}{L$\varepsilon$BAI} shows competitive empirical performance with modified BAI algorithms. Even though it is outperformed by LinGapE, \hyperlink{algoLeBAI}{L$\varepsilon$BAI} is almost twice as fast as $\cX\cY$-Adaptive and appears to be slightly more robust than LinGame to increasing dimension.

\vfill


\end{document}